\documentclass{article}
\PassOptionsToPackage{numbers, compress}{natbib}
\usepackage[final,main]{neurips_2026}

\usepackage{times}

\usepackage{nicefrac}       % compact symbols for 1/2, etc.
\usepackage{microtype}      % microtypography

\usepackage[utf8]{inputenc}
\usepackage[T1]{fontenc}
\usepackage{microtype}
\usepackage{graphicx}
\usepackage{subfigure}
\usepackage{booktabs}

\usepackage[colorlinks=true,citecolor=blue,linkcolor=blue]{hyperref}
\usepackage{url}
\usepackage{nicefrac}
\usepackage{listings}
\usepackage{mathtools}
\usepackage{amsfonts}
\usepackage{amssymb}
\usepackage{amsmath}
\usepackage{amsthm}
\usepackage{docmute}
\usepackage[acronym, nomain]{glossaries}
\usepackage{indentfirst}
\usepackage{bm}
\usepackage{here}
\usepackage{booktabs}       % professional-quality tables
\usepackage{cancel}
\usepackage{multirow}
\usepackage{wrapfig}
\usepackage{lipsum}
\usepackage{thm-restate}

\newtheorem{lemma}{Lemma}
\newtheorem{theorem}{Theorem}
\newtheorem{proposition}{Proposition}
\newtheorem{corollary}{Corollary}

\newtheorem{remark}{Remark}
\newtheorem{gprelemma}{Lemma}[section]

\DeclarePairedDelimiterX{\KL}[2]{\mathrm{KL}[}{]}{#1\;\delimsize\|\;#2}

\DeclarePairedDelimiterX\braket[2]{\langle}{\rangle}{#1 \delimsize\vert #2}

\newcommand{\calx}{\mathcal{X}}

\newcommand{\Law}{\mathrm{Law}}

\newcommand{\calw}{\mathcal{W}}

\newcommand{\calz}{\mathcal{Z}}

\newcommand{\caly}{\mathcal{Y}}

\usepackage{multirow}

\usepackage{mathtools}
\usepackage{xcolor}

\mathtoolsset{showonlyrefs=true}
\let\KL\relax
\newcommand{\KL}{\mathrm{KL}}

\providecommand{\E}{\mathbb E}
\providecommand{\Law}{\operatorname{Law}}
\providecommand{\Unif}{\mathrm{Unif}}

\providecommand{\cA}{\mathcal A}
\title{Information-Theoretic Generalization Bounds \\for Sequential Decision Making}
\author{%
  {Futoshi Futami\textsuperscript{1,2,3}\qquad Masahiro Fujisawa\textsuperscript{1,3,4}}\\
  {\textsuperscript{1}The University of Osaka\enspace
       \textsuperscript{2}The University of Tokyo\enspace
       \textsuperscript{3}RIKEN AIP\enspace
       \textsuperscript{4}Toyota Motor Corporation}\\
  {\texttt{futami.futoshi.es@osaka-u.ac.jp}\qquad\texttt{fujisawa@ist.osaka-u.ac.jp}}%
}

\begin{document}
\maketitle

\begin{abstract}
Information-theoretic generalization bounds based on the supersample construction
are a central tool for algorithm-dependent generalization analysis in the batch i.i.d.~setting.
However, existing supersample conditional mutual information (CMI) bounds do not directly apply to sequential decision-making
problems such as online learning, streaming active learning, and bandits, where data are revealed
adaptively and the learner evolves along a causal trajectory.
To address this limitation, we develop a sequential supersample framework that separates the
{learner's filtration from auxiliary random variables used only in the analysis.}
Under a row-wise exchangeability assumption, the sequential generalization gap is controlled by
sequential CMI, a sum of {roundwise CMI terms between losses and selectors.}
We also establish a Bernstein-type refinement that yields faster rates under suitable variance
conditions.
The {SCMI} proof strategy applies to online learning, streaming active learning with importance
weighting, and stochastic multi-armed bandits.
\end{abstract}

\section{Introduction}
\label{sec_intro}

Information-theoretic (IT) generalization bounds~\citep{russo16,Xu2017} are now a standard tool for algorithm-dependent generalization analysis. Rather than controlling generalization only through the size of a hypothesis class, these bounds measure how strongly a particular learning procedure couples the training data to the output it returns. Mutual information (MI)~\citep{Xu2017} and conditional mutual information (CMI)~\citep{steinke20a} are especially useful for this purpose because they quantify the information about the sample that is retained in the learned predictor, parameters, or evaluated losses.

Among IT approaches, supersample-based CMI bounds provide particularly sharp and often non-vacuous bounds in the batch data setting, where the training data are independent and identically distributed (i.i.d.)~\citep{steinke20a,hellstrom2022a,harutyunyan2021}. A paired supersample separates the selected training coordinate from an unselected ghost coordinate{, which serves as a virtual validation observation used only in the analysis. The} resulting CMI term measures how much the algorithm reveals about the selector. This viewpoint captures stability~\citep{steinke20a,futami2023timeindependent,wang2022on}, compression~\citep{steinke20a,sefidgaran2023minimum}, and implicit regularization~\citep{wang2022on,wang2023} effects that are often invisible to worst-case complexity measures. Recent scalar refinements further improve the numerical behavior of these bounds~\citep{wang23}. The same framework also clarifies how deterministic algorithms can generalize and how algorithm-dependent information terms connect to classical combinatorial quantities such as the VC dimension~\citep{Mohri}.

Existing CMI arguments are essentially batch arguments.
They use the symmetry of paired i.i.d. supersamples to compare, in each row, the selected training coordinate with its unselected ghost coordinate.
This symmetry is not directly available in adaptive settings.
{For example, consider a binary feedback model in which the learner's current prediction determines the distribution of future labels. In this model, changing an earlier selector while holding the entire supersample fixed does not preserve the joint law, so the usual batch CMI argument does not apply directly (see Appendix~\ref{app:scmi-toy-models} for details).}
In online learning~\citep{haddouche2022online,haddouche2023pacbayes}, active learning~\citep{beygelzimer09}, stochastic bandits~\citep{seldin2012pacbayes}, and other sequential problems, each observation is generated through a history-dependent interaction between the learner and the environment.
A terminal CMI term can measure the dependence between the final output and the whole selector vector, but it does not reveal where the information needed for each {comparison of training and ghost losses} enters the trajectory.
Thus, extending CMI methods to such settings requires a construction that preserves the ghost-sample comparison while respecting causality.

To handle adaptivity, we introduce a causal version of the supersample construction used in batch CMI.
{The construction compares training and ghost losses while separating the learner's filtration from auxiliary random variables introduced only for the theoretical analysis.}
The learner history contains only selected past observations; for the analysis, we additionally reveal the current paired row so that the selected coordinate can be compared with its ghost coordinate at that round.
This separation is essential because a terminal CMI term can measure dependence on the whole selector vector, but it does not show {how much information each loss reveals about its selector}.
We therefore measure, at each round, the CMI between the learner-induced loss and the current selector after fixing the selected past and the current paired row.
We call these roundwise quantities sequential CMI (SCMI).
{For a learner using only selected observations, row-wise exchangeability implies that} the train--holdout gap decomposes into {correlations between the losses and selectors at each round}, which are controlled by SCMI in Theorem~\ref{thm:general-rademacher}.
We also give a Bernstein refinement in Theorem~\ref{thm:general-bernstein}, which yields fast-rate bounds when the corresponding {second-moment condition} holds.

We apply the {SCMI analysis} in three sequential examples. Section~\ref{sec:online-application} studies online learning and sequential complexity. Here, selector prefixes generate a binary tree of counterfactual histories, connecting SCMI with sequential Rademacher complexity and, under pathwise realizability, with the Littlestone dimension~\citep{rakhlin2015online}. Section~\ref{sec:active-learning-application} considers streaming active learning with importance weighting.
The sequential supersample includes query coins and label masks, so the learner observes only queried labels while the proof compares selected and ghost coordinates.
This yields an evaluated SCMI bound for the terminal predictor in Theorem~\ref{thm:active-iw-slow}.

Section~\ref{sec:bandit-application} treats stochastic multi-armed bandits with smoothing.
Using the observed feedback, we view the empirical bandit gap estimate and its ghost counterpart as a train--holdout pair for an importance-weighted gap process. {Under a fixed positive gap from the optimal arm and tuning that uses this gap,} a Bernstein-inspired variance estimate yields {an expected cumulative} regret bound with square-root horizon dependence, up to logarithmic and gap-dependent factors, improving the horizon order relative to existing PAC-Bayes bandit bounds~\citep{seldin2012pacbayes,haddouche2023pacbayes}.

\section{Preliminaries}
\label{sec:preliminaries}
We use uppercase letters for random variables and lowercase letters for their realizations.
For $a\in\mathbb N$, write $[a]\coloneqq \{1,\dots,a\}$.
We denote MI and CMI by $I(\cdot;\cdot)$ and $I(\cdot;\cdot\mid\cdot)$, respectively, and Kullback--Leibler divergence by $\KL(\cdot\|\cdot)$.

We first recall the batch supersample CMI setting in a form that highlights the
{correlation between the loss and the selector} controlled by information-theoretic bounds~\citep{steinke20a}.
Let $\widetilde Z=((Z_{i,0},Z_{i,1}))_{i=1}^n$ be an $n\times 2$ supersample whose entries are drawn i.i.d.~from $\mathcal{D}$.
Let $U=(U_1,\dots,U_n)\sim\Unif(\{0,1\}^n)$ be {a selector vector} independent of $\widetilde Z$, and define
$S=\widetilde Z_U \coloneqq (Z_{i,U_i})_{i=1}^n$ as a training dataset.
Let $\calw\subset\mathbb{R}^d$ be a parameter space, and let $\cA:\calz^n\to\calw$
be a possibly randomized learning algorithm, {with output $W$ and} conditional distribution $P(W|S)$. For a bounded loss {function} $\ell:\mathcal{W}\times\mathcal{Z}\to[0,1]$, define
{the population risk} $L_{\mathcal D}(W)\coloneqq \E_{Z\sim\mathcal D}[\ell(W,Z)\mid W]$, {the empirical risk} $\widehat L_S(W)\coloneqq \frac1n\sum_{i=1}^n \ell(W,Z_{i,U_i})$, {the coordinate-$0$ loss value} $L_i^+ \coloneqq  \ell(W,Z_{i,0})$, {and} $\varepsilon_i\coloneqq 2U_i - 1$.
We refer to this as the \emph{batch supersample setting}. In this setting, the standard one-coordinate supersample identity~\citep{wang23} is
\begin{align}
\E[L_{\mathcal D}(W)-\widehat L_S(W)]
&=
\frac1n\sum_{i=1}^n\E[\ell(W,Z_{i,1-U_i})-\ell(W,Z_{i,U_i})]
=
\frac2n\sum_{i=1}^n\E[\varepsilon_iL_i^+],
\label{eq:batch-one-coordinate-id}
\end{align}
where the generalization gap is written as a signed correlation between the selector noise and the evaluated loss.
Rademacher-complexity arguments~\citet{Mohri} control analogous signed correlations in the worst case; the CMI approach controls them through the information revealed about the selector.

\begin{theorem}[Batch single-loss {evaluated CMI} bound~\citep{wang23}]
\label{thm:batch-single-loss-ecmi}
Under the batch supersample setting,
\begin{align}
\bigl|\E[L_{\mathcal D}(W)-\widehat L_S(W)]\bigr|
&\le
\frac2n\sum_{i=1}^n\sqrt{2\,I(L_i^+;U_i\mid \widetilde Z)}.
\label{eq:batch-one-coordinate-final}
\end{align}
\end{theorem}
The CMI term in Theorem~\ref{thm:batch-single-loss-ecmi} is often called \emph{loss CMI}.
In supervised learning, it can be related to function-level and output-level CMI terms.
Let $\calx$ and $\caly$ be the input and label spaces, let $\calz\coloneqq \calx\times\caly$, and write
$Z_{i,u}=(X_{i,u},Y_{i,u})$.
For a predictor $f_W:\calx\to\caly$, define $F_{\widetilde Z}\coloneqq (f_W(X_{i,u}))_{i\le n,u\in\{0,1\}}$ as the predictions on the $2n$ supersample inputs.
For this comparison, assume that the evaluated loss factors through these predictions, so conditional on $\widetilde Z$ each $L_i^+$ is a measurable function of $F_{\widetilde Z}$.
Jensen's inequality and the data-processing inequality~\citep{cover2012element} give
\begin{align}
\frac1n\sum_{i=1}^n\sqrt{I(L_i^+;U_i\mid \widetilde Z)}
\leq
\sqrt{\frac1n I(F_{\widetilde Z};U\mid\widetilde Z)}
\leq
\sqrt{\frac1n I(W;U\mid \widetilde Z)}.
\label{eq:batch-cmi-comparison}
\end{align}
Here $I(F_{\widetilde Z};U\mid\widetilde Z)$ is the {functional CMI}~\cite{harutyunyan2021}, while $I(W;U\mid\widetilde Z)$ is the output CMI, or parameter CMI when $W$ is a parameter.
Thus the generalization gap is controlled by the information about the selector encoded in the learned object.

For {binary-valued} function classes, the {functional CMI} is further controlled by a pattern count.
In particular, if the VC dimension of the class of $f_W$ is $d<\infty$, then
\citet{harutyunyan2021} show that
\begin{align}
I(F_{\widetilde Z};U\mid\widetilde Z)
\leq
\max\{(d + 1)\log2,d \log (2en/d) \}.
\label{eq:batch-ecmi-vc}
\end{align}
Thus the supersample argument connects algorithm-dependent information bounds to classical
algorithm-independent complexity measures.
Batch CMI therefore has two complementary interpretations: it is an algorithm-dependent
correlation bound, and in VC classes it is also controlled by a deterministic pattern count.

\section{General Sequential Conditional Mutual Information Bounds}
\label{sec:general-framework}
The batch supersample setting in Section~\ref{sec:preliminaries} and Eq.~\eqref{eq:batch-one-coordinate-id} expresses the
generalization gap as a {correlation between the loss and the selector} after conditioning on a paired
supersample. We now develop its sequential analogue.
The main difficulty is that, in sequential learning, the learner only observes the
selected path, whereas the proof must temporarily compare the selected coordinate
with its unselected ghost counterpart.
Our construction therefore separates the learner filtration from {its local enlargement by auxiliary random variables introduced for the proof}
and then controls the resulting {correlations between losses and selectors at each round}
by sequential CMI (SCMI).

\subsection{Proposed SCMI setting}
\label{subsec:seq-setting}
We describe the sequential supersample setting used throughout the paper.
Informally, at each round \(t\), the environment generates a paired latent row
\(\widetilde Z_t=(Z_{t,0},Z_{t,1})\), a fair selector \(U_t\) chooses one coordinate,
and the learner updates only from the selected {observation}.
The unselected coordinate is never observed by the learner, but is used in the proof
as a ghost coordinate for the train--holdout comparison. Thus, the construction involves {two random elements}.
The first is the learner history $H_t$, with sigma-field $\mathcal{H}_t \coloneqq  \sigma(H_t)$, which contains only the information retained along the selected path.
The second is {the conditioning variable} $G_{t-1}\coloneqq (H_{t-1},\widetilde Z_t)$, where $\mathcal{G}_{t-1}\coloneqq \mathcal H_{t-1}\vee\sigma(\widetilde Z_t)$, which exposes the current latent row only for the purpose of comparing selected and unselected coordinates. The learner does not observe $G_{t-1}$.

We now give the formal construction. All random elements below take values in standard Borel spaces, so regular conditional
laws are available.
The construction proceeds round by round.
An application specifies an initial history $H_0$ and sets $\mathcal{H}_0\coloneqq \sigma(H_0)$.
Assume that $H_{t-1}$ and $\mathcal H_{t-1}$ have already been defined.
At round $t$, the latent supersample row is sampled from a history-dependent kernel: $\Law(\widetilde Z_t\mid \mathcal H_{t-1})=P_t(\cdot\mid H_{t-1})$.
This kernel may depend on the selected past encoded in $H_{t-1}$; it is not an
independence assumption across time.
It covers, for example, history-dependent online data, active queries, and bandit feedback.

{For the proof, we then {define the conditioning variable} $G_{t-1}\coloneqq (H_{t-1},\widetilde Z_t)$, with sigma-field $\mathcal G_{t-1}\coloneqq \mathcal{H}_{t-1}\vee\sigma(\widetilde Z_t)$.
{For each $t\in[n]$,} draw $U_t\in\{0,1\}$ so that}
\begin{align}
    \Pr(U_t=0\mid \mathcal G_{t-1})
    =\Pr(U_t=1\mid \mathcal G_{t-1})
    =\frac{1}{2} \qquad \text{a.s.}
    \label{eq:selector-sym}
\end{align}
Set $Z_t \coloneqq  Z_{t,U_t}$ and $Z'_t \coloneqq  Z_{t,1-U_t}$, and use the shorthand
\begin{align*}
    \widetilde Z^t \coloneqq  (\widetilde Z_1,\ldots,\widetilde Z_t),
    \quad
    U^t \coloneqq  (U_1,\ldots,U_t),
    \quad
    Z^t \coloneqq  (Z_1,\ldots,Z_t),
    \quad
    W^{t-1}\coloneqq (W_1,\ldots,W_{t-1}).
\end{align*}
Although $Z_t$ is defined as $Z_{t,U_t}$ and therefore depends on $U_t$,
conditioning on the learner history means conditioning on the observed value retained
by the learner, not conditioning on the selector itself.
Unless an application explicitly records the selector, $U_t$ is not part of $H_t$.

{We assume that, conditional on the learner history \(H_{t-1}\), the conditional law of \(W_t\) depends on the current supersample row \(\widetilde Z_t\) and the selector \(U_t\) only through the selected observation \(Z_t=Z_{t,U_t}\). Specifically, for every \(t\),}
\begin{align}
    \Law(W_t\mid H_{t-1},\widetilde Z_t,U_t)
    =
    K_t(\cdot\mid H_{t-1},Z_t)
    \qquad \text{a.s.}
    \label{eq:selected-feedback-kernel}
\end{align}
Here $K_t$ is a probability kernel on the learner-state space $\mathcal W$.
Deterministic updates correspond to the special case where $K_t$ is a Dirac kernel.
After $W_t$ and any other retained round-$t$ variables have been formed, the next learner history $H_t$ is defined recursively by adjoining only those retained selected-path variables to $H_{t-1}$.
Set $\mathcal H_t\coloneqq \sigma(H_t)$, so that $\mathcal H_{t-1}\subseteq \mathcal H_t$.
By construction, $H_t$ never contains the unselected coordinate $Z'_t$ unless an application explicitly changes the observation model. Representative choices of \(H_t\) and \(G_{t-1}\) are summarized in Table~\ref{tab:framework-correspondence}.

The distinction between $H_t$ and $G_{t-1}$ is essential.
The process $(\mathcal H_t)_{t\ge0}$ is the learner filtration.
By contrast, {conditioning on} $G_{t-1}$ {fixes} the current
latent row before the current selector is averaged over.
Consequently, the information terms below do not condition on the current selector.

We also allow predictable learner-side weights. Let \(Q_t\) satisfy, for every \(t\in[n]\),
\begin{align}
0\le Q_t\le 1,
\qquad
Q_t\text{ is }\mathcal H_{t-1}\text{-measurable}.
\label{eq:predictable-weight}
\end{align}
For a loss {function} \(\ell:\mathcal W\times\mathcal Z\to[0,1]\), define
{the weighted loss values}
$L_t^+ \coloneqq Q_t\,\ell(W_t,Z_{t,0})$ and
$L_t^- \coloneqq Q_t\,\ell(W_t,Z_{t,1})$.
Predictable weights are useful for learner-side stopping rules and pre-round importance weights.
The {expected} weighted training and holdout risks are
\begin{align}
L_n^{\mathrm{tr}}\coloneqq \frac1n\sum_{t=1}^n\E[\E[Q_t \, \ell(W_t,Z_t)\mid G_{t-1}]],\qquad L_n^{\mathrm{ho}}\coloneqq \frac1n\sum_{t=1}^n\E[\E[Q_t\,\ell(W_t,Z_t')\mid G_{t-1}]].
\label{eq:def-weighted-risks}
\end{align}
{Define the sequential generalization gap}
$\mathrm{Err}_n^{\mathrm{seq}}
\coloneqq 
L_n^{\mathrm{ho}}-L_n^{\mathrm{tr}}$ and
$\varepsilon_t\coloneqq 2U_t-1\in\{-1,+1\}$. 
We refer to the construction above as the \emph{SCMI setting}.

\paragraph{Meaning of the holdout {risk}.}
If the current row is conditionally i.i.d.,
\begin{align}
    P_t(\cdot\mid H_{t-1})
    =
    P_t^0(\cdot\mid H_{t-1})
    \otimes
    P_t^0(\cdot\mid H_{t-1}),
    \label{eq:conditional-product-row}
\end{align}
then the unselected coordinate is a fresh conditional draw from $P_t^0(\cdot\mid H_{t-1})$ and is not used to construct $W_t$.
In this case, $L_n^{\rm ho}=\frac1n\sum_{t=1}^n\E\!\left[
Q_t\int\ell(W_t,z)\,P_t^0(\mathrm{d}z\mid H_{t-1})
\right]$, {where the integral is the conditional population risk of $W_t$.}
This differs from the prequential online {risk} used in PAC-Bayes supermartingale analyses~\citep{haddouche2022online}, where the hypothesis at round $t$ is chosen before the round-$t$ observation is seen.
That prequential target can be represented in the present notation by evaluating a pre-update state, for example by replacing $W_t$ with $W_{t-1}$ in the loss.

The main bounds do not require the conditional product structure in Eq.~\eqref{eq:conditional-product-row}.
Instead, we assume row-wise exchangeability. Let $\tau(z_0,z_1)\coloneqq (z_1,z_0)$. 
We assume, for all $t\in[n]$,
\begin{align}
    P_t(\cdot\mid H_{t-1})
    =
    \tau_{\#}P_t(\cdot\mid H_{t-1})
    \qquad \text{a.s.}
    \label{eq:row-exchange-general}
\end{align}
A conditionally i.i.d.~row is the canonical sufficient case.  Even without row-wise exchangeability, one can still obtain a two-coordinate SCMI bound by keeping the signed difference \(L_t^+-L_t^-\) instead of reducing to the one-coordinate loss \(L_t^+\).  Appendix~\ref{app:without-exchangeability-scmi} records this variant.

\begin{table}[t]
\centering
\small
\caption{Application-specific instances of the SCMI setting.}
\label{tab:framework-correspondence}
\resizebox{\textwidth}{!}{
\begin{tabular}{lll}
\toprule
Setting & Learner history & Proof-side enlargement \\
\midrule
Online learning & selected observations and states $(Z_s,W_s)$ & current paired row $(Z_{t,0},Z_{t,1})$ \\
Active learning & selected features, query coins, observed labels, and states $(X_s,V_s,\bar Y_s,W_s)$ & current row $(X_{t,0},Y_{t,0},V_{t,0}),(X_{t,1},Y_{t,1},V_{t,1})$ \\
Bandits & selected policies, actions, and rewards $(\pi_s,A_s,R_s)$ & policy and paired feedback $\pi_t,(A_{t,0},R_{t,0}),(A_{t,1},R_{t,1})$ \\
\bottomrule
\end{tabular}}
\end{table}

Throughout the paper, $I(X;Y\mid G)$ denotes the usual scalar CMI, namely the expectation of the regular CMI given the sigma-field generated by the random element $G$.
In all main bounds, the current selector appears as the second argument of MI, not as a conditioning variable.

\subsection{Information-theoretic generalization bounds under SCMI}
\label{subsec:general-slow}
We prove the general SCMI bounds.
The first step is a sequential analogue of the one-coordinate supersample identity of Eq.~\eqref{eq:batch-one-coordinate-id}. 
Under {Eq.~\eqref{eq:selected-feedback-kernel}} and row-wise exchangeability~\eqref{eq:row-exchange-general}, the train--holdout gap can be written as a sum of {correlations between $L_t^+$ and $\varepsilon_t$}.

\begin{lemma}
\label{lem:basic-seq-sym}
Suppose that \eqref{eq:selected-feedback-kernel} and \eqref{eq:row-exchange-general} hold. 
Then, for every predictable weight process $Q$,
\begin{align}
\mathrm{Err}_n^{\rm seq}
    =
    \frac{2}{n}
    \sum_{t=1}^n
    \E\!\left[
        \E\!\left[
            \varepsilon_t L_t^+
            \mid G_{t-1}
        \right]
    \right].
\label{eq:basic-seq-sym}
\end{align}
\end{lemma}
Eq.~\eqref{eq:basic-seq-sym} is the sequential counterpart of the batch correlation identity in Eq.~\eqref{eq:batch-one-coordinate-id}.
The static correlation $\E[\varepsilon_i L_i^+]$ is replaced by the roundwise conditional correlation $\E\!\left[\E[\varepsilon_t L_t^+\mid G_{t-1}]\right]$, which keeps the current latent row fixed and conditions on the information available before the current selector is averaged over. 
{Here,} $G_{t-1}$ exposes $(Z_{t,0},Z_{t,1})$ for the comparison, while the learner state $W_t$ is still generated only from the selected {observation}.

{Applying the Donsker--Varadhan lemma to Lemma~\ref{lem:basic-seq-sym} gives the next theorem.}

\begin{theorem}[General SCMI bound]
\label{thm:general-rademacher}
Suppose that \eqref{eq:selected-feedback-kernel} and \eqref{eq:row-exchange-general} hold. Let $\ell:\mathcal W\times\mathcal Z\to[0,1]$. Then, 
\begin{align}
    \bigl|\mathrm{Err}_n^{\rm seq}\bigr|
    \le
    \frac{2}{n}
    \sum_{t=1}^n
    \sqrt{
        2\,I(L_t^+;U_t\mid G_{t-1})
    }
    \le
    2
    \sqrt{
        \frac{2}{n}
        \sum_{t=1}^n
        I(W_t;U_t\mid G_{t-1})
    }.
\label{eq:general-rademacher-bound}
\end{align}
\end{theorem}
The first inequality is an evaluated or loss-SCMI bound: it only charges information about the selector that is present in the evaluated loss $L_t^+$.
For the second inequality, data processing first gives $I(L_t^+;U_t\mid G_{t-1})\le I(W_t;U_t\mid G_{t-1})$, because, after conditioning on $G_{t-1}$, the variables $Q_t,Z_{t,0},Z_{t,1}$ are fixed and $L_t^+=Q_t\ell(W_t,Z_{t,0})$ is a measurable function of $W_t$. Jensen's inequality then aggregates the resulting state-SCMI terms.
Appendix~\ref{app:without-exchangeability-scmi} records a two-coordinate SCMI bound without row-wise exchangeability, replacing the one-coordinate loss \(L_t^+\) by the difference \(L_t^+-L_t^-\).

{The bound depends on $\sum_{t=1}^n I(W_t;U_t\mid G_{t-1})$.} 
It has a directed-information interpretation{~\citep{massey1990causality,kramer1998directed}}, discussed in Appendix~\ref{app:directed-information}. 
The round-$t$ increment is large only when the selected {observation} at round $t$ reveals information about the current selector through the updated learner state, after the selected past and the current latent pair have already been fixed. 
Thus SCMI charges information at the time it can affect the train--holdout comparison, rather than through a single terminal CMI term.

We next {provide} a Bernstein-type refinement.
The slow-rate bound above uses only boundedness and therefore gives a square-root
dependence on the information term.
When the signed excess process satisfies a Bernstein condition, the information term
enters linearly.
For the ordinary supervised excess-risk form, fix a comparator \(w^\star\in\mathcal W\) and define $e_{t,0}(w^\star)\coloneqq Q_t\bigl(\ell(W_t,Z_{t,0})-\ell(w^\star,Z_{t,0})\bigr)$ and $e_{t,1}(w^\star)\coloneqq Q_t\bigl(\ell(W_t,Z_{t,1})-\ell(w^\star,Z_{t,1})\bigr)$.
With this convention, $e_{t,1-U_t}$ is the {holdout excess loss} and $e_{t,U_t}$ is the {training excess loss}.
Assume \(|e_{t,u}(w^\star)|\le b\) almost surely for \(u\in\{0,1\}\), with \(b\ge0\). Define the {expected} holdout and selected excess risks by
\begin{align*}
    R_n^{\mathrm{ho}}\coloneqq \frac{1}{n}\sum_{t=1}^n\E\!\left[\E\!\left[e_{t,1-U_t}(w^\star)\mid G_{t-1}\right]\right],
    \qquad
    R_n^{\mathrm{tr}}\coloneqq \frac{1}{n}\sum_{t=1}^n\E\!\left[\E\!\left[e_{t,U_t}(w^\star)\mid G_{t-1}\right]\right].
\end{align*}
The same row-swap argument as in Lemma~\ref{lem:basic-seq-sym} gives $R_n^{\mathrm{ho}}-R_n^{\mathrm{tr}}=\frac{2}{n}\sum_{t=1}^n\E\!\left[\E[\varepsilon_t e_{t,0}(w^\star)\mid G_{t-1}]\right]$, which is the key for deriving the following theorem.

\begin{theorem}[Bernstein-type SCMI bound]
\label{thm:general-bernstein}
Suppose that Eqs.~\eqref{eq:selected-feedback-kernel} and \eqref{eq:row-exchange-general} hold. Assume the nonnegative-excess condition $R_n^{\mathrm{ho}}\ge0$ and that the Bernstein condition
\begin{align}
\frac{1}{n}\sum_{t=1}^n
\E\!\left[
\E\!\left[(e_{t,0}(w^\star))^2\mid G_{t-1}\right]
\right]
\le
B R_n^{\mathrm{ho}}
\label{eq:general-bernstein-condition}
\end{align}
holds for some $B>0$.
If $b=0$, the claim is trivial. Otherwise, for every $0<\lambda<3/b$ such that $1-B\lambda/(1-\lambda b/3)>0$, we have
\begin{align}
\left(1-\frac{B\lambda}{1-\lambda b/3}\right)R_n^{\mathrm{ho}}
\le
R_n^{\mathrm{tr}}
+
\frac{2}{\lambda n}\sum_{t=1}^n I(e_{t,0}(w^\star);U_t\mid G_{t-1}).
\label{eq:bernstein-fast-main}
\end{align}
\end{theorem}
Theorem~\ref{thm:general-bernstein} is the fast-rate component of the framework. It is stated in a supervised excess-risk notation, but the proof only uses the row-swap identity above, boundedness, and the Bernstein second-moment condition. Thus an application may use the theorem once the relevant {second moment} is controlled by the holdout excess risk. In partial-feedback problems, this verification can be difficult. The bandit application shows that the Bernstein approach remains useful, but the terminal exponential-weights posterior requires {a separate second-moment bound} rather than a naive substitution into the theorem. Explicit choices of $\lambda$ and the bounded nonnegative special case are given in Appendix~\ref{subsec:bernstein-additional-results}, while the shifted-Rademacher fast-rate refinement is given in Appendix~\ref{app:wang-refinement}.
{}{Appendix~\ref{app:scmi-toy-models} gives examples showing that the square-root and linear dependences on SCMI in Theorems~\ref{thm:general-rademacher} and~\ref{thm:general-bernstein}, respectively, cannot in general be improved.}

\subsection{Application: Online learning and sequential complexity}
\label{sec:online-application}
Online learning is the canonical specialization of the SCMI framework.
Set $Q_t\equiv 1$, let $H_t^{\rm on}=(H_{t-1}^{\rm on},Z_t,W_t)$ and $G_{t-1}^{\rm on}=(H_{t-1}^{\rm on},\widetilde Z_t)$, where $Z_t=Z_{t,U_t}$.
{Eq.~\eqref{eq:selected-feedback-kernel}} becomes $\Law(W_t\mid H_{t-1}^{\rm on},\widetilde Z_t,U_t)=K_t^{\rm on}(\cdot\mid H_{t-1}^{\rm on},Z_t)$.
Thus the learner may update adaptively from the selected past and the current selected observation, but it cannot use the unselected coordinate.
Theorem~\ref{thm:general-rademacher} then gives the online train--holdout bound with $L_t^+=\ell(W_t,Z_{t,0})$ and $L_t^-=\ell(W_t,Z_{t,1})$.

{Appendix~\ref{app:threshold-motivation-en} extends the deterministic threshold learner of \citet{steinke20a} to an adaptive observation process. At each round, candidate features are sampled uniformly from the interval between zero and the current threshold estimate, so their distribution depends on past updates. For this process, $I(W_n;Z^n)=\infty$, but evaluating the SCMI of $L_t^+$ directly from its definition and the update rule gives a cumulative $O(\log n)$ bound. Theorem~\ref{thm:general-bernstein} then gives an $O((\log n)/n)$ bound, matching an independent calculation of the sequential generalization gap.}

A standard example is a randomized Gibbs online learner: the update kernel may be chosen as $K_t^{\rm on}(\mathrm{d}w\mid H_{t-1}^{\rm on},Z_t)
\propto
    \exp\{-\eta\ell(w,Z_t)\}\Pi_{t-1}(\mathrm{d}w\mid H_{t-1}^{\rm on})$, where \(\Pi_{t-1}\) is a posterior measure determined by the selected past.
This example illustrates that the SCMI bound applies to {randomized updates}. % and not only to finite-dimensional deterministic iterates.
The online specialization also explains why SCMI is the sequential analogue of batch CMI.
{Appendix~\ref{app:online-pre-post} relates this bound to standard pre-update regret through the change in training risk caused by one update.}
{The predictable-stopping case \(Q_t=\mathbf 1\{\tau\ge t\}\) is given in Appendix~\ref{sec:stopping-application}.}
{}{Appendix~\ref{app:scmi-numerical} numerically evaluates the bounds in Theorem~\ref{thm:general-rademacher} for an adaptive Gibbs learner, with the reported bound estimates below the trivial upper bound.}

\paragraph{{Tree representation} and sequential Rademacher complexity.}
The online supersample construction can be viewed as a tree-indexed analogue
of the batch supersample construction.  In the batch setting, after conditioning
on the supersample \(\widetilde Z\), the \(t\)-th pair is attached to the fixed
index \(t\).  In the online setting, the current pair is generated after the
counterfactual history determined by the previously selected coordinates.
Thus, before round \(t\), a possible counterfactual history is indexed by a
deterministic selector prefix \(u_{<t}\in\{0,1\}^{t-1}\), and the actual online
process follows the single random path \(U^n\) through this {tree}.

This {tree representation} is the same viewpoint that underlies
sequential Rademacher complexity and Littlestone dimension~\citep{rakhlin2015online}.
Sequential Rademacher complexity studies worst-case signed correlations over
such trees, whereas SCMI keeps the realized data-dependent online experiment
and measures, through \(I(L_t^+;U_t\mid G_{t-1}^{\rm on})\), how much
information about the current random branch \(U_t\) is carried by the
{loss}.  This is analogous to the relationship between batch CMI
and Rademacher complexity discussed in Section~\ref{sec:preliminaries}.  {If the loss is binary and the pathwise realizability condition holds,} the SCMI sum is controlled by the sequential
pattern growth of the fixed loss class, and hence by its Littlestone
dimension.
The formal tree and Littlestone definitions, as well as the precise
assumption and proof, are given in Appendix~\ref{app:seq-covers-littlestone}.

\section{Application: Streaming active learning with importance weighting}
\label{sec:active-learning-application}
Our next example is streaming-based active learning {with importance weighting (IW),} following \citet{beygelzimer09}. In this setting, given the input data $X_t$, the learner queries the corresponding label $Y_t$ with some probability. Formally, let $\mathcal Z=\mathcal X\times\mathcal Y$ and $P_{XY}=P_XP_{Y\mid X}$. At round $t$, the learner observes $X_t$, then calculates a query probability $p_t=\pi_t(W_{t-1},X_t)\in[p_{\min},1]$ with \(p_{\min}>0\), and requests the corresponding label by drawing $V_t\sim{\rm Unif}[0,1]$ and $Q_t=\mathbf1\{V_t\le p_t\}$. The label is revealed to the learner only if $Q_t=1$. 
Here $Q_t$ denotes the {active learning} query indicator; unlike the predictable weights in Section~\ref{subsec:seq-setting}, it is generated after the current feature and query coin have been exposed.
Let $\bot\notin\mathcal Y$ denote a missing-label symbol and write $\bar Y_t=Y_t$ if queried and $\bar Y_t=\bot$ otherwise. To make the label-dependent update explicit, let $S_t$ be the importance-weighted labeled sample passed to the supervised learner. Set $S_0=\emptyset$ and update
\begin{align}
S_t\coloneqq 
\begin{cases}
S_{t-1}\cup\{(X_t,Y_t,1/p_t)\}, & Q_t=1,\\
S_{t-1}, & Q_t=0.
\end{cases}
\label{eq:active-labeled-set-update}
\end{align}
Thus, when $Q_t=0$, no pseudo-label is assigned and no supervised loss involving $Y_t$ is evaluated. The learner state is updated from the weighted labeled sample and the previous state as
\begin{align}
W_t\coloneqq \Psi_t(S_t,W_{t-1}),\qquad 
H_t^{\mathrm{act}}\coloneqq (H_{t-1}^{\mathrm{act}},X_t,V_t,\bar Y_t,W_t).
\label{eq:active-state-update}
\end{align}
Here, $\bot$ is used only for bookkeeping and is never passed to the loss function.

Consequently, $S_t$ is not an i.i.d. sample as assumed in standard CMI bounds, because its elements are selected by an adaptive query rule; this is naturally handled by the sequential setting in Section~\ref{subsec:seq-setting}. For the supersample analysis, let $Z_{t,u}:=(X_{t,u},Y_{t,u})$ with $u\in\{0,1\}$, where{, given \(H_{t-1}^{\rm act}\),} the two pairs $(Z_{t,0},V_{t,0})$ and $(Z_{t,1},V_{t,1})$ are {drawn independently from \(P_{XY}\otimes{\rm Unif}[0,1]\).} The selector chooses the triple used by the learner: $(X_t,Y_t,V_t)=(X_{t,U_t},Y_{t,U_t},V_{t,U_t})$. {For the analysis, define}
\[
G_{t-1}^{\mathrm{act}}
:=
(H_{t-1}^{\mathrm{act}},X_{t,0},Y_{t,0},V_{t,0},X_{t,1},Y_{t,1},V_{t,1}),
\qquad
\mathcal G_{t-1}^{\mathrm{act}}
:=\sigma(G_{t-1}^{\mathrm{act}}),
\]
so we assume that
$\mathbb P(U_t=0\mid\mathcal G_{t-1}^{\mathrm{act}})=\mathbb P(U_t=1\mid\mathcal G_{t-1}^{\mathrm{act}})=1/2$.
For $u\in\{0,1\}$, set $p_{t,u}\coloneqq \pi_t(W_{t-1},X_{t,u})$ and $Q_{t,u}\coloneqq \mathbf 1\{V_{t,u}\le p_{t,u}\}$. {Eq.~\eqref{eq:selected-feedback-kernel} holds for these triples} because $S_t$, and hence $W_t=\Psi_t(S_t,W_{t-1})$, is a function of the previous learner history and the selected feedback $(X_t,p_t,Q_t,\bar Y_t)$ only. For the final predictor $W_n$ and loss {function} $\ell:\mathcal W\times\mathcal Z\to[0,1]$, define {the population risk} $R(w)\coloneqq \int\ell(w,z)P_{XY}(dz)$. {The IW train and holdout risks are}
\begin{align}
L_n^{\mathrm{IW}}&\coloneqq \frac1n\sum_{t=1}^n\E\!\left[\E\!\left[\frac{Q_t}{p_t}\ell(W_n,Z_t)\mid G_{t-1}^{\mathrm{act}}\right]\right],\\
L_n^{\mathrm{IW,ho}}&\coloneqq \frac1n\sum_{t=1}^n\E\!\left[\E\!\left[\frac{Q_{t,1-U_t}}{p_{t,1-U_t}}\ell(W_n,Z_{t,1-U_t})\mid G_{t-1}^{\mathrm{act}}\right]\right].
\label{eq:active-iw-train}
\end{align}
Here, the loss is never evaluated at $\bot$: when $Q_t=0$ the contribution is zero, while when $Q_t=1$ the true label is observed. Thus $Q_t/p_t$ is the usual importance-weighted contribution. The holdout identity is an expectation-level importance-weighting identity: after conditioning on the selected {active learning} transcript, the unselected query coin is still averaged out. Consequently, $L_n^{\mathrm{IW,ho}}=\mathbb E[R(W_n)]$; the proof is given in Appendix~\ref{app:proof-active-iw-population}.

Theorem~\ref{thm:general-rademacher} is not applied here by literal substitution: the query indicator depends on the current feature and query coin, and the final predictor $W_n$ depends on the selected future trajectory. The bridge from this final output to roundwise selector correlations is the terminal row-swap identity in Lemma~\ref{lem:active-row-swap-identity}. Define $L_{n,t}^+\coloneqq Q_{t,0}\ell(W_n,Z_{t,0})/p_{t,0}$ and $L_{n,t}^-\coloneqq Q_{t,1}\ell(W_n,Z_{t,1})/p_{t,1}$. Here $Q_{t,u}$ is the coordinate-wise query indicator and $1/p_{t,u}$ is the importance correction.
\begin{theorem}
\label{thm:active-iw-slow}
Assume {Eq.~\eqref{eq:selected-feedback-kernel}} and, for every \(t\), the conditionally i.i.d.~construction in Eq.~\eqref{eq:conditional-product-row} for the two triples \((X_{t,u},Y_{t,u},V_{t,u})\), \(u\in\{0,1\}\). Under the {active learning} setting above,
\begin{align}
\left|\mathbb E[R(W_n)]-L_n^{\mathrm{IW}}\right|
\le
\frac{2}{p_{\min}n}\sum_{t=1}^n
\sqrt{2\,I(L_{n,t}^+;U_t\mid G_{t-1}^{\mathrm{act}})}.
\label{eq:active-iw-slow-sum}
\end{align}
\end{theorem}
The same importance-weighted supersample construction also supports a Shifted-Rademacher fast-rate refinement. See Appendix~\ref{app:active-fast-rate-refinements} for the details.

\section{Application: Stochastic bandits with smoothing}
\label{sec:bandit-application}
We apply the {SCMI framework} to stochastic multi-armed bandits using a Bernstein-type argument tailored to partial feedback. Let \(\mathcal A=[K]=\{1,\ldots,K\}\), \(K\ge2\), and write \(\mathcal P_K\) for the set of distributions on \(\mathcal A\). Each arm \(a\in\mathcal A\) has an unknown reward distribution \(P_a\) on \([0,1]\), with mean \(r(a)\). At round \(s\), the learner chooses a predictable policy \(\pi_s\in\mathcal P_K\){. Given the selected history and \(\pi_s\), the learner draws \(A_s\) from \(\pi_s\). Conditional on this history, \(\pi_s\), and \(A_s=a\), the observed reward \(R_s\) has law \(P_a\).} Let \(a^\star\) be the unique optimal arm, define \(\Delta(a)=r(a^\star)-r(a)\), \(\Delta_{\min}=\min_{a\ne a^\star}\Delta(a)>0\), and \(\Delta(\rho)=\sum_a\rho(a)\Delta(a)\) for \(\rho\in\mathcal P_K\). The assumption \(\Delta_{\min}>0\) is the standard positive-gap regime in gap-dependent analyses of stochastic bandits~\citep{bubeck2012regret}. 

The algorithm uses the reward estimate \(\widehat r_t\), computed from
\begin{align}
R_s^a:=\frac{\mathbf 1\{A_s=a\}R_s}{\pi_s(a)},\qquad
\widehat r_t(a):=\frac1t\sum_{s=1}^t R_s^a.
\label{eq:def-bandit-rhat}
\end{align}
For analysis only, write
\(\widehat\Delta_t(\rho):=\sum_a\rho(a)\{\widehat r_t(a^\star)-\widehat r_t(a)\}\).
The variable \(R_s^a\) is an importance-weighted version of the observed scalar reward, not a counterfactual reward for an unplayed arm. From \(\widehat r_t\), define the unsmoothed exponential-weights distribution \(\rho_t^{\rm exp}(a)\propto\exp\{\gamma_t\widehat r_t(a)\}\), {where \(\gamma_t>0\) is deterministic}. After \(t\) observations, the reported or next-play policy is \(\widetilde\rho_t^{\rm exp}(a)=(1-K\epsilon_{t+1})\rho_t^{\rm exp}(a)+\epsilon_{t+1}\), {where \((\epsilon_t)_{t\ge1}\) is a deterministic nonincreasing sequence in \((0,1/K]\)} and the played policy satisfies \(\pi_s(a)\ge\epsilon_s\) at each data-collection round \(s\). Thus \(\rho_t^{\rm exp}\) is used for the exponential-weights update, while \(\widetilde\rho_t^{\rm exp}\) is the policy whose regret is evaluated. A natural implementation can set the next behavior policy to this smoothed distribution, but the analysis only requires the lower bound \(\pi_s(a)\ge\epsilon_s\) on the behavior probabilities used in the importance weights. {Starting from an initial history \(H_0^{\rm bd}\), the learner history is defined by} \(H_s^{\rm bd}:=(H_{s-1}^{\rm bd},\pi_s,A_s,R_s)\), with \(\mathcal H_s^{\rm bd}:=\sigma(H_s^{\rm bd})\), and the pre-round field is \(\mathcal F_{s-1}^{\rm bd}:=\mathcal H_{s-1}^{\rm bd}\vee\sigma(\pi_s)\). The smoothed distribution \(\widetilde\rho_s^{\rm exp}\) is an output computed from this transcript and is not added as a separate history component.

We now introduce {paired feedback for the analysis}. After \(\pi_s\) has been chosen, draw two conditionally independent observed-feedback coordinates $(A_{s,0},R_{s,0})$ and $(A_{s,1},R_{s,1})$ by sampling \(A_{s,u}\sim\pi_s\) and then \(R_{s,u}\sim P_{A_{s,u}}\), independently for \(u\in\{0,1\}\) conditional on \(\mathcal F_{s-1}^{\rm bd}\). {Before drawing the selector, define the conditioning variable by
\begin{align}
G_{s-1}^{\rm bd}
&:=(H_{s-1}^{\rm bd},\pi_s,A_{s,0},R_{s,0},A_{s,1},R_{s,1}),
\qquad s\ge1.
\label{eq:bandit-proof-field-recursive}
\end{align}}
\noeqref{eq:bandit-proof-field-recursive}
Let \(\mathcal G_{s-1}^{\rm bd}:=\sigma(G_{s-1}^{\rm bd})\). Then draw the current selector so that \(\mathbb P(U_s=0\mid\mathcal G_{s-1}^{\rm bd})=\mathbb P(U_s=1\mid\mathcal G_{s-1}^{\rm bd})=1/2\), set \((A_s,R_s):=(A_{s,U_s},R_{s,U_s})\), and reveal only \((A_s,R_s)\) to the learner. {Subsequent rounds are generated from the selected history \(H_s^{\rm bd}\) using fresh randomness.  This selected path has the same law as the ordinary bandit process described above.} Thus the random element \(G_{s-1}^{\rm bd}\) {contains the current feedback pair} but not \(U_s\).

For \(u\in\{0,1\}\) and \(a\in[K]\), define
\begin{align}
R_{s,u}^a:=\frac{\mathbf1\{A_{s,u}=a\}R_{s,u}}{\pi_s(a)},\qquad
G_{s,u}(a):=R_{s,u}^{a^\star}-R_{s,u}^a,
\qquad
G_{s,u}(\rho):=\sum_{a=1}^K\rho(a)G_{s,u}(a).
\label{eq:def-bandit-G}
\end{align}
Then \(\widehat\Delta_t(\rho)=t^{-1}\sum_{s=1}^tG_{s,U_s}(\rho)\).

{As in existing PAC-Bayes bandit analyses~\citep{seldin2012pacbayes,haddouche2023pacbayes},} our final target is the {expected} regret
\(\mathbb E[\Delta(\widetilde\rho_t^{\rm exp})]\) of the learned smoothed policy.  We decompose this target as
\begin{align}
\mathbb E[\Delta(\widetilde\rho_t^{\rm exp})]
&=
\underbrace{\mathbb E[\Delta(\widetilde\rho_t^{\rm exp})-\Delta(\rho_t^{\rm exp})]}_{\text{smoothing}}
+\!
\underbrace{\left(\mathbb E[\Delta(\rho_t^{\rm exp})]-2\mathbb E[\widehat\Delta_t(\rho_t^{\rm exp})]\right)}_{\text{{controlled by SCMI}}}
+
\underbrace{2\mathbb E[\widehat\Delta_t(\rho_t^{\rm exp})]}_{\text{empirical optimization}},
\label{eq:bandit-main-decomposition}
\end{align}
where the factor $2$ in the second and third term anticipates the Bernstein-type generalization bound below, which controls $\mathbb E[\Delta(\rho_t^{\exp})]$ {by twice the expected empirical gap plus} {an SCMI term}.
The first term is {bounded by} the deterministic smoothing cost {$K\epsilon_{t+1}$}, and the third term is controlled by the exponential-weights variational inequality. Thus it remains to bound the middle term by {SCMI}.
To do so, {given \(H_t^{\rm bd}\), draw an auxiliary arm \(\bar A_t\) from \(\rho_t^{\rm exp}\) using an independent random seed}.  Then, after averaging over the {auxiliary arm} and the bandit randomness, \(\mathbb E[\widehat\Delta_t(\rho_t^{\rm exp})]=t^{-1}\sum_{s=1}^t\mathbb E[G_{s,U_s}(\bar A_t)]\) and \(\mathbb E[\Delta(\rho_t^{\rm exp})]=t^{-1}\sum_{s=1}^t\mathbb E[G_{s,1-U_s}(\bar A_t)]\).  {The first identity follows by averaging over \(\bar A_t\) given the selected history \(H_t^{\rm bd}\). For the second, given \(H_t^{\rm bd}\) and \(\bar A_t\), the unselected feedback at round \(s\) is a fresh draw under \(\pi_s\), so its expected gap estimate is \(\Delta(\bar A_t)\).}  Hence the middle term in \eqref{eq:bandit-main-decomposition} {can be bounded using the selected and unselected gap estimates} \(G_{s,u}(\bar A_t)\).

Under this setting, the following theorem controls the regret of the learned smoothed policy.
\begin{theorem}[{Regret bound} for smoothed exponential weights]
\label{thm:bandit-main}
Assume that rewards lie in \([0,1]\), that the optimal arm is unique, and that \(\pi_s(a)\ge\epsilon_s\) for all \(s\ge1\) and \(a\in[K]\).  Then, for every \(t\ge1\),
\begin{align}
\mathbb E[\Delta(\widetilde\rho_t^{\rm exp})]
\le
K\epsilon_{t+1}
+
\frac{2\log K}{\gamma_t}
+
\frac{52}{t\epsilon_t\Delta_{\min}}
\sum_{s=1}^t I\!\left(G_{s,0}(\bar A_t);U_s\mid G_{s-1}^{\rm bd}\right).
\label{eq:bandit-one-step-cmi}
\end{align}
Moreover, the SCMI term can be bounded to give
\begin{align}
\mathbb E[\Delta(\widetilde\rho_t^{\rm exp})]
\le
K\epsilon_{t+1}
+
\frac{2\log K}{\gamma_t}
+
\frac{52\log K}{t\epsilon_t\Delta_{\min}}.
\label{eq:bandit-one-step}
\end{align}
Consequently, with
$\epsilon_t=
\min\left\{
\frac1K,
\sqrt{\frac{52\log K}{Kt\Delta_{\min}}}
\right\}$ and
$\gamma_t=\frac{2\log K}{K\epsilon_t}$,
there is a universal constant \(C>0\) such that
\begin{align}
\sum_{t=1}^T\mathbb E[\Delta(\widetilde\rho_t^{\rm exp})]
\le
C\sqrt{\frac{KT\log K}{\Delta_{\min}}}.
\label{eq:bandit-cumulative}
\end{align}
\end{theorem}
{If $\pi_{t+1}=\widetilde\rho_t^{\rm exp}$, this sum is the expected cumulative regret over rounds $2,\ldots,T+1$; the initial round contributes at most one.}
The three terms on the right-hand side of \eqref{eq:bandit-one-step-cmi} correspond, respectively, to smoothing, empirical optimization, and {the SCMI bound on the middle term of} Eq.~\eqref{eq:bandit-main-decomposition}.
{The displayed tuning uses $\Delta_{\min}$; its $\sqrt{T}$ horizon dependence holds for a fixed positive $\Delta_{\min}$.}

Existing PAC-Bayes bandit analyses give \(\mathcal O(K^{1/3}T^{2/3})\)-type cumulative bounds~\citep{seldin2012pacbayes,haddouche2023pacbayes}. Following the observation of \citet{seldin2012pacbayes}, we use a sharper variance evaluation. Under the standard positive-gap assumption relative to the optimal arm, the resulting bound improves the PAC-Bayes horizon order to \(\mathcal O(\sqrt{KT})\), up to logarithmic and gap-dependent factors. 
%In the proof, the general Bernstein bound in Theorem~\ref{thm:general-bernstein} cannot be applied directly.
Because the terminal posterior $\rho_t^{\exp}$ is computed from the selected trajectory, the {second-moment calculation needed for this bound} is not a direct instance of Theorem~\ref{thm:general-bernstein} in our proof; {the cases in which coordinate $0$ is selected or unselected must be treated separately}.
%because the terminal posterior \(\rho_t^{\rm exp}\) depends on the selected trajectory;
{} The proof is given in Appendix~\ref{app:proofs-bandit-application}, and Appendix~\ref{app:bandit-proof-filtration-choice} discusses {the variables included in $G_{s-1}^{\rm bd}$}.

%The proof-side context \(G_{s-1}^{\rm bd}\) is enlarged beyond the local fields of Sections~\ref{sec:general-framework} and~\ref{sec:active-learning-application} because \(\rho_t^{\rm exp}\), and hence the virtual arm \(\bar A_t\), is tied to the whole selected path.  Retaining past proof-side feedback makes the entropy-reduction step a direct chain-rule argument; Appendix~\ref{app:bandit-proof-filtration-choice} gives the details.
{}
{Appendix~\ref{app:bandit-ordinary-mi-bernstein} gives a shorter ordinary-MI proof with the same regret order; the SCMI proof retains the roundwise CMI terms involving the gap estimates.}

\section{Related Work}
\label{sec:related-work}
Existing IT analysis controls generalization through the dependence between the training sample and the learned output~\citep{russo16,Xu2017}. The supersample CMI framework~\citep{steinke20a} sharpens this idea by conditioning on a paired sample and measuring how much the algorithm reveals about the selector. This view connects algorithm-dependent information bounds with deterministic pattern-counting arguments such as VC growth, as recalled in Section~\ref{sec:preliminaries}. Subsequent work developed sharper CMI variants, evaluated-CMI bounds, and fast-rate refinements~\citep{haghifam2020,hellstrom2022a,wang23}. As discussed in Section~\ref{sec:general-framework}, our contribution is to retain the supersample selector idea while replacing the sample selector by a causal sequence of selectors and histories. This extension also preserves, in the online setting, a counterpart of the relationship between CMI, Rademacher complexity, and VC-type combinatorial dimensions in the i.i.d. setting. In particular, as discussed in Section~\ref{sec:online-application}, our sequential supersample construction is naturally connected to sequential Rademacher complexity and can be controlled {under the pathwise realizability condition} by the Littlestone dimension~\citet{rakhlin2015online}, which plays the role of an online analogue of the VC dimension.

PAC-Bayes bounds are another major class of algorithm-dependent generalization bounds, and their applications to sequential decision-making have been studied for a long time~\citep{seldin2012pacbayes,haddouche2022online,haddouche2023pacbayes}. More recently, filtration-based PAC-Bayes analyses have treated posteriors or randomized predictors adapted to a growing history, where supermartingale techniques make it possible to handle martingale, online, and heavy-tailed settings~\citep{haddouche2023pacbayes}. Our framework also uses a learner filtration, but the proof additionally introduces a latent supersample row and conditions on it to compare selected and unselected coordinates. For bandits, PAC-Bayes ideas were developed by \citet{seldin2012pacbayes}; our bandit application uses the same smoothed exponential-weights structure, but couples it with the SCMI term and a gap-dependent variance calculation. 
Moreover, in the i.i.d.~setting, CMI bounds admit PAC-Bayes interpretations with data-dependent or optimally chosen priors.
%Moreover, CMI bounds can be interpreted as corresponding to PAC-Bayes bounds with an optimally chosen prior. 
Since PAC-Bayes analyses based on CMI have already been developed in the i.i.d.~setting~\citep{grunwald2021pac}, our SCMI bound suggests a possible route toward analogous PAC-Bayes bounds with data-dependent priors for online learning.

\section{Conclusion and Limitations}
\label{sec_conclusion}
%We developed an SCMI framework for IT analysis. 
We developed a sequential supersample framework for IT generalization analysis.
By making the learner filtration explicit, the framework {extends} batch supersample CMI {to} a causal sum of roundwise {CMI terms in online learning. The corresponding bounds use the final predictor in active learning and the learned posterior in bandits.} The resulting template applies to online learning, streaming active learning, and stochastic bandits with smoothing. The fast-rate results further show how sharper {second-moment estimates} can improve slow-rate SCMI bounds; in {our bandit SCMI proof}, this requires {an additional CMI term in the second-moment bound} under partial feedback. The main limitations are that our bounds are expectation bounds and assume bounded losses. High-probability guarantees, heavy-tailed losses, dynamic regret, contextual bandits, reinforcement learning, and broader partial-monitoring problems remain open directions.
%By making the filtration explicit, the framework transforms batch supersample CMI into a causal sum of roundwise selector-information influx terms. This is the main conceptual difference from terminal CMI bounds: SCMI measures the information revealed about the current selector precisely at the round where the selected coordinate is compared with its ghost counterpart. The resulting template applies to online learning, streaming active learning, and stochastic bandits with smoothing. The fast-rate results show how sharper proof-side variance estimates can improve slow-rate SCMI bounds; in bandits, this requires an additional selector-comparison correction under partial feedback. The major limitations are that our bounds are expectation bounds and assume bounded losses; extending them to high-probability guarantees, heavy-tailed losses, dynamic regret, contextual bandits, reinforcement learning, and broader partial-monitoring problems remains an important direction.
\section*{{Acknowledgements}}
{FF was supported by JST BOOST Program Grant Number JPMJBY24G8. FF and MF were supported by JSPS KAKENHI Grant Number 26K02974, Japan. MF was supported by JSPS KAKENHI Grant Number JP25K21286, Japan. This work was supported by Japan Science and Technology Agency (JST) as part of Adopting Sustainable Partnerships for Innovative Research Ecosystem (ASPIRE), Grant Number JPMJAP25B1.}

\bibliographystyle{plainnat}
\bibliography{main}
\newpage

\appendix

\paragraph{Appendix organization.}
Appendix~\ref{app:setting-details} proves the core SCMI results, including the one-coordinate identity, the slow-rate bound, and the exchangeability-free two-coordinate variant. Appendix~\ref{app:directed-information} gives the directed-information interpretation, Appendix~\ref{app:seq-covers-littlestone} gives the sequential covering and Littlestone-dimension argument, and Appendix~\ref{app:bernstein-consequences} collects the Bernstein and Shifted-Rademacher fast-rate refinements. Appendices~\ref{app:online-specialization}, \ref{sec:stopping-application}, and \ref{app:active-learning} specialize the framework to online learning, {stopping with predictable weights}, and active learning. Appendix~\ref{app:proofs-bandit-application} {proves the bandit theorem and then explains} why direct substitution into Theorem~\ref{thm:general-bernstein} fails. Appendix~\ref{app:bandit-proof-filtration-choice} {compares the local conditioning variables across the three applications}. Appendix~\ref{app:bandit-ordinary-mi-bernstein} gives an ordinary-MI baseline. {Appendix~\ref{app:scmi-numerical} gives the numerical illustration, and Appendix~\ref{app:scmi-toy-models} gives the feedback and matching-rate examples.} {Appendix~\ref{app:threshold-motivation-en} compares MI, CMI, and SCMI through threshold classification.}

\section{Proofs for the general SCMI framework}
\label{app:setting-details}

This appendix proves the general SCMI statements from Section~\ref{sec:general-framework}. Appendix~\ref{app:proofs-general-framework} proves the one-coordinate symmetrization and the slow-rate bound; Appendix~\ref{app:without-exchangeability-scmi} gives the exchangeability-free two-coordinate variant.

\subsection{One-coordinate symmetrization and slow-rate proof}
\label{app:proofs-general-framework}

\subsubsection{Proof of Lemma~\ref{lem:basic-seq-sym}}
\label{app:proof-basic-seq-sym}

\begin{proof}
Fix a round \(t\). The proof is a one-round comparison after {}{the learner history \(H_{t-1}\)} has been fixed. Since \(Q_t\) is \(\mathcal H_{t-1}\)-measurable and \(H_{t-1}\) generates \(\mathcal H_{t-1}\), there exists a measurable function \(q_t\) such that \(Q_t=q_t(H_{t-1})\) a.s. Here \(q_t\) is only a measurable version of the predictable weight as a function of the learner history; it is not an additional algorithmic object. {By Eq.~\eqref{eq:selected-feedback-kernel},}
\[
\Law(W_t\mid H_{t-1},\widetilde Z_t,U_t)
=
K_t(\,\cdot\,\mid H_{t-1},Z_{t,U_t}).
\]
For \(h\), a coordinate used for the update \(z\), and a coordinate used for evaluation \(\bar z\), define
\begin{align}
\varphi_t(h,z,\bar z)
\coloneqq 
q_t(h)\int \ell(w,\bar z)\,K_t(dw\mid h,z).
\label{eq:proof-basic-phi}
\end{align}
Thus \(\varphi_t(h,z,\bar z)\) is the {weighted risk} when the learner is updated with \(z\) and then evaluated at \(\bar z\).

Averaging over the selector $U_t$ gives the {one-round training and holdout risks}:
\begin{align}
\E\bigl[Q_t\ell(W_t,Z_{t,U_t})\bigr]
&=
\frac12\,\E\!\left[\varphi_t(H_{t-1},Z_{t,0},Z_{t,0})+\varphi_t(H_{t-1},Z_{t,1},Z_{t,1})\right],
\label{eq:proof-basic-tr}
\\
\E\bigl[Q_t\ell(W_t,Z_{t,1-U_t})\bigr]
&=
\frac12\,\E\!\left[\varphi_t(H_{t-1},Z_{t,0},Z_{t,1})+\varphi_t(H_{t-1},Z_{t,1},Z_{t,0})\right].
\label{eq:proof-basic-ho}
\end{align}
Row-wise exchangeability says that the joint law is unchanged when \(Z_{t,0}\) and \(Z_{t,1}\) are swapped after conditioning on \(H_{t-1}\). Hence
\begin{align}
\E\!\left[\varphi_t(H_{t-1},Z_{t,0},Z_{t,1})\right]
&=
\E\!\left[\varphi_t(H_{t-1},Z_{t,1},Z_{t,0})\right],
\label{eq:proof-basic-cross}
\\
\E\!\left[\varphi_t(H_{t-1},Z_{t,0},Z_{t,0})\right]
&=
\E\!\left[\varphi_t(H_{t-1},Z_{t,1},Z_{t,1})\right].
\label{eq:proof-basic-diag}
\end{align}
Subtracting the {training risk} from the {holdout risk} therefore {}{gives the following difference:}
\begin{align}
&\E\bigl[Q_t\ell(W_t,Z_{t,1-U_t})\bigr]
-
\E\bigl[Q_t\ell(W_t,Z_{t,U_t})\bigr]\\
&=
\E\!\left[\varphi_t(H_{t-1},Z_{t,1},Z_{t,0})-\varphi_t(H_{t-1},Z_{t,0},Z_{t,0})\right].
\label{eq:proof-basic-gap}
\end{align}

{}{The same difference} is obtained from the signed correlation with \(L_t^+=Q_t\ell(W_t,Z_{t,0})\). Indeed, if \(U_t=1\), then \(\varepsilon_t=1\) and the conditional contribution is \(\varphi_t(H_{t-1},Z_{t,1},Z_{t,0})\); if \(U_t=0\), then \(\varepsilon_t=-1\) and the conditional contribution is \(-\varphi_t(H_{t-1},Z_{t,0},Z_{t,0})\). Therefore
\begin{align}
2\,\E[\varepsilon_tL_t^+]
&=
\E\!\left[\varphi_t(H_{t-1},Z_{t,1},Z_{t,0})-\varphi_t(H_{t-1},Z_{t,0},Z_{t,0})\right].
\label{eq:proof-basic-onebranch}
\end{align}
Combining \eqref{eq:proof-basic-gap} and \eqref{eq:proof-basic-onebranch} gives, for each \(t\),
\[
\E\bigl[Q_t\ell(W_t,Z_{t,1-U_t})\bigr]
-
\E\bigl[Q_t\ell(W_t,Z_{t,U_t})\bigr]
=
2\E[\varepsilon_tL_t^+].
\]
Let \(\mathcal G_{t-1}:=\sigma(G_{t-1})=\mathcal H_{t-1}\vee\sigma(\widetilde Z_t)\). The conditional expectation in the SCMI identity is taken with respect to {\(\mathcal G_{t-1}\)}. Hence the tower property gives
\[
\E[\varepsilon_tL_t^+]
=
\E\!\left[
\E\!\left[
\varepsilon_tL_t^+
\middle|\mathcal G_{t-1}\right]
\right]
=
\E\!\left[
\E\!\left[
\varepsilon_tL_t^+
\middle|G_{t-1}\right]
\right].
\]
Summing over \(t\) therefore yields
\[
\mathrm{Err}_n^{\mathrm{seq}}
=
\frac{2}{n}\sum_{t=1}^n \E[\varepsilon_tL_t^+]
=
\frac{2}{n}\sum_{t=1}^n
\E\!\left[
\E\!\left[
\varepsilon_t L_t^+
\middle| G_{t-1} \right] \right],
\]
which is \eqref{eq:basic-seq-sym}.
\end{proof}

\subsubsection{Proof of Theorem~\ref{thm:general-rademacher}}
\label{app:proof-general-rademacher}

\begin{proof}
Fix \(t\in[n]\) and set
\[
    m_t
    :=
    \mathbb E\!\left[
        \mathbb E\!\left[\varepsilon_tL_t^+\mid G_{t-1}\right]
    \right]
    =
    \mathbb E[\varepsilon_tL_t^+].
\]
Let \(s_t\in\{-1,+1\}\) be the sign of \(m_t\), with either sign chosen when \(m_t=0\). Then \(s_tm_t=|m_t|\).  We apply the conditional Donsker--Varadhan variational formula to the conditional law of \((L_t^+,U_t)\) given \(G_{t-1}\), using as reference the conditional product law under which \(L_t^+\) keeps its conditional marginal law and the selector is replaced by an independent uniform copy \(U_t'\).  Write \(\varepsilon_t'=2U_t'-1\).  For the test function
\[
    f(\ell,u):=s_t\lambda(2u-1)\ell,
    \qquad \lambda>0,
\]
the variational formula gives
\begin{align}
I(L_t^+;U_t \mid G_{t-1})
&\ge
\mathbb E[f(L_t^+,U_t)]
-
\mathbb E\!\left[
\log \mathbb E\!\left[
\exp\{f(L_t^+,U_t')\}
\middle| G_{t-1}
\right]\right]
\notag\\
&=
\lambda |m_t|
-
\mathbb E\!\left[
\log \mathbb E\!\left[
\exp\{s_t\lambda\varepsilon_t' L_t^+\}
\middle| G_{t-1}
\right]\right].
\label{eq:proof-rad-dv-expanded}
\end{align}
Here the sign \(s_t\) is deterministic once the distribution is fixed; it is used only to turn the scalar correlation \(m_t\) into its absolute value.

It remains to bound the logarithmic term.  Conditionally on \((L_t^+,G_{t-1})\), the variable \(\varepsilon_t'\) is an independent Rademacher sign. Therefore
\begin{align}
\mathbb E\!\left[
\exp\{s_t\lambda\varepsilon_t' L_t^+\}
\middle| L_t^+,G_{t-1}
\right]
&=
\frac12\exp\{s_t\lambda L_t^+\}
+
\frac12\exp\{-s_t\lambda L_t^+\}
\notag\\
&=
\cosh(s_t\lambda L_t^+)
=
\cosh(\lambda L_t^+),
\end{align}
where the last equality uses that \(\cosh\) is even and \(s_t\in\{-1,+1\}\).  Taking conditional expectation over \(L_t^+\) gives
\[
\mathbb E\!\left[
\exp\{s_t\lambda\varepsilon_t' L_t^+\}
\middle| G_{t-1}
\right]
=
\mathbb E\!\left[\cosh(\lambda L_t^+)\middle| G_{t-1}\right].
\]
Since \(0\le L_t^+\le1\) and \(\cosh(\lambda x)\le\exp(\lambda^2x^2/2)\le\exp(\lambda^2/2)\) for \(x\in[0,1]\), we obtain
\[
\mathbb E\!\left[
\exp\{s_t\lambda\varepsilon_t' L_t^+\}
\middle| G_{t-1}
\right]
\le
\exp(\lambda^2/2).
\]
Substituting this bound into \eqref{eq:proof-rad-dv-expanded} yields
\[
\lambda
\left|
\mathbb E\!\left[\mathbb E[\varepsilon_tL_t^+\mid G_{t-1}]\right]
\right|
\le
I(L_t^+;U_t \mid G_{t-1})+\frac{\lambda^2}{2}.
\]
{Dividing by \(\lambda\) and minimizing the resulting right-hand side} over \(\lambda>0\) gives
\begin{align}
\bigl|\mathbb E\!\left[\mathbb E\!\left[\varepsilon_t L_t^+\middle| G_{t-1}\right]\right]\bigr|
\le
\sqrt{2\,I(L_t^+;U_t \mid G_{t-1})}.
\label{eq:proof-rad-one-step}
\end{align}
Finally, Lemma~\ref{lem:basic-seq-sym} implies
\begin{align}
\bigl|\mathrm{Err}_n^{\mathrm{seq}}\bigr|
&=
\frac{2}{n}\left|\sum_{t=1}^n \mathbb E\!\left[\mathbb E\!\left[\varepsilon_t L_t^+\middle| G_{t-1}\right]\right]\right|
\le
\frac{2}{n}\sum_{t=1}^n \bigl|\mathbb E\!\left[\mathbb E\!\left[\varepsilon_t L_t^+\middle| G_{t-1}\right]\right]\bigr|
\\
&\le
\frac{2}{n}\sum_{t=1}^n \sqrt{2\,I(L_t^+;U_t \mid G_{t-1})}.
\end{align}
This proves the evaluated SCMI inequality in \eqref{eq:general-rademacher-bound}.
\end{proof}

\subsubsection{Proof of the state-based {inequality} in Theorem~\ref{thm:general-rademacher}}
\label{app:proof-general-directed-slow}

\begin{proof}
Because $L_t^+$ is a measurable function of $(W_t,Q_t,\widetilde Z_t)$ and $(Q_t,\widetilde Z_t)$ is measurable with respect to $G_{t-1}$, the data-processing inequality yields
\begin{align}
I(L_t^+;U_t \mid G_{t-1})
\le
I(W_t;U_t \mid G_{t-1}).
\end{align}
{To obtain the second inequality in \eqref{eq:general-rademacher-bound}, combine this data-processing bound with Jensen's inequality for the concave map $x \mapsto \sqrt{x}$:}
\begin{align}
\frac{1}{n}\sum_{t=1}^n \sqrt{2\,I(W_t;U_t \mid G_{t-1})}
\le
\sqrt{\frac{2}{n}\sum_{t=1}^n I(W_t;U_t \mid G_{t-1})}.
\end{align}
Multiplying by $2$ proves the claim.
\end{proof}

\subsection{Variants without row-wise exchangeability}
\label{app:without-exchangeability-scmi}

The main proof uses exact row-wise exchangeability to reduce the train--holdout difference to {a correlation between a single-coordinate loss and \(\varepsilon_t\)}. We record a variant that makes no row-wise exchangeability assumption, {}{and uses the loss difference \(L_t^+-L_t^-\).}

\subsubsection{Two-coordinate SCMI bound without row-wise exchangeability}
\label{app:nonexchangeable-scmi}

The main text uses row-wise exchangeability to obtain the one-coordinate identity in Lemma~\ref{lem:basic-seq-sym}.
This assumption is not needed for a slow-rate information bound if we keep both coordinates in the loss difference.
The price is that the information term involves the two-coordinate statistic \(L_t^+-L_t^-\), rather than the single evaluated loss \(L_t^+\).

\begin{lemma}[Two-coordinate symmetrization]
\label{lem:two-coordinate-sym-no-exchange}
Assume the selector symmetry~\eqref{eq:selector-sym} and the predictable-weight condition~\eqref{eq:predictable-weight}.
No row-wise exchangeability assumption is required. Then
\begin{align}
    \mathrm{Err}_n^{\rm seq}
    =
    \frac1n
    \sum_{t=1}^n
    \E\!\left[
        \E\!\left[
            \varepsilon_t\bigl(L_t^+-L_t^-\bigr)
            \mid G_{t-1}
        \right]
    \right].
    \label{eq:two-coordinate-sym-no-exchange}
\end{align}
\end{lemma}

\begin{proof}
For each round \(t\), the identity
\[
    Q_t\ell(W_t,Z_t')-Q_t\ell(W_t,Z_t)
    =
    \varepsilon_t\bigl(L_t^+-L_t^-\bigr)
\]
holds pointwise. Indeed, when \(U_t=1\), the right-hand side is \(L_t^+-L_t^-\), and when \(U_t=0\), it is \(L_t^- - L_t^+\).
Taking conditional expectation given \(G_{t-1}\), then expectation, and summing over \(t\) proves the claim.
\end{proof}

\begin{proposition}[Two-coordinate SCMI bound without row-wise exchangeability]
\label{prop:two-coordinate-nonexchangeable-scmi}
Assume the selector symmetry~\eqref{eq:selector-sym} and the predictable-weight condition~\eqref{eq:predictable-weight}.
No row-wise exchangeability assumption is required. Then, for any loss \(\ell:\mathcal W\times\mathcal Z\to[0,1]\),
\begin{align}
    \bigl|\mathrm{Err}_n^{\rm seq}\bigr|
    &\le
    \frac1n
    \sum_{t=1}^n
    \sqrt{2\,I\bigl(L_t^+-L_t^-;U_t\mid G_{t-1}\bigr)}
    \label{eq:two-coordinate-scmi-bound-a}
    \\
    &\le
    \frac1n
    \sum_{t=1}^n
    \sqrt{2\,I\bigl((L_t^+,L_t^-);U_t\mid G_{t-1}\bigr)}
    \label{eq:two-coordinate-scmi-bound-b}
    \\
    &\le
    \frac1n
    \sum_{t=1}^n
    \sqrt{2\,I\bigl(W_t;U_t\mid G_{t-1}\bigr)}
    \label{eq:two-coordinate-scmi-bound-c}
    \\
    &\le
    \sqrt{
        \frac{2}{n}
        \sum_{t=1}^n I\bigl(W_t;U_t\mid G_{t-1}\bigr)
    }.
    \label{eq:two-coordinate-scmi-bound-d}
\end{align}
\end{proposition}

\begin{proof}
Set \(X_t\coloneqq L_t^+-L_t^-\). Since \(0\le Q_t\le1\) and \(0\le\ell\le1\), we have \(|X_t|\le1\).
Fix a round \(t\). For each value \(g\) of {the random element} \(G_{t-1}\), define the conditional divergence
\[
    d_t(g)
    \coloneqq
    \KL\!\left(
        P_{X_t,U_t\mid G_{t-1}=g}
        \,\middle\|\,
        P_{X_t\mid G_{t-1}=g}P_{U_t\mid G_{t-1}=g}
    \right).
\]
By selector symmetry, \(U_t\) is conditionally uniform given \(G_{t-1}=g\).
Applying the Donsker--Varadhan variational formula under this conditional law gives, for every \(\lambda>0\),
\[
    d_t(g)
    \ge
    \lambda\E[\varepsilon_tX_t\mid G_{t-1}=g]
    -
    \log
    \E\!\left[
        \exp\{\lambda\varepsilon_t'X_t\}
        \mid G_{t-1}=g
    \right],
\]
where \(\varepsilon_t'\) is an independent Rademacher variable under the product conditional law.
Under this product law, \(\varepsilon_t'\) is independent of \(X_t\). Hence, since \(|X_t|\le1\),
\[
    \E\!\left[
        \exp\{\lambda\varepsilon_t'X_t\}
        \mid G_{t-1}=g
    \right]
    =
    \E\!\left[
        \cosh(\lambda X_t)
        \mid G_{t-1}=g
    \right]
    \le
    \exp(\lambda^2/2).
\]
Applying the same argument to \(-X_t\) and optimizing over \(\lambda>0\) yields
\[
    \left|
        \E[\varepsilon_tX_t\mid G_{t-1}=g]
    \right|
    \le
    \sqrt{2d_t(g)}.
\]
Therefore, by Lemma~\ref{lem:two-coordinate-sym-no-exchange},
\begin{align*}
    \bigl|\mathrm{Err}_n^{\rm seq}\bigr|
    &\le
    \frac1n\sum_{t=1}^n
    \left|
        \E\!\left[
            \E[\varepsilon_tX_t\mid G_{t-1}]
        \right]
    \right| \\
    &\le
    \frac1n\sum_{t=1}^n
    \E\!\left[
        \left|
            \E[\varepsilon_tX_t\mid G_{t-1}]
        \right|
    \right] \\
    &\le
    \frac1n\sum_{t=1}^n
    \E\!\left[\sqrt{2d_t(G_{t-1})}\right] \\
    &\le
    \frac1n\sum_{t=1}^n
    \sqrt{2\E[d_t(G_{t-1})]} \\
    &=
    \frac1n\sum_{t=1}^n
    \sqrt{2I(X_t;U_t\mid G_{t-1})}.
\end{align*}
This proves~\eqref{eq:two-coordinate-scmi-bound-a}.

The inequality~\eqref{eq:two-coordinate-scmi-bound-b} follows from conditional data processing, since \(X_t=L_t^+-L_t^-\) is a measurable function of \((L_t^+,L_t^-)\).
For~\eqref{eq:two-coordinate-scmi-bound-c}, note that \(G_{t-1}\) contains the current row \(\widetilde Z_t=(Z_{t,0},Z_{t,1})\), and \(Q_t\) is \(\mathcal H_{t-1}\)-measurable, hence \(G_{t-1}\)-measurable.
Thus, after conditioning on \(G_{t-1}\),
\[
    (L_t^+,L_t^-)
    =
    \bigl(
        Q_t\ell(W_t,Z_{t,0}),
        Q_t\ell(W_t,Z_{t,1})
    \bigr)
\]
is a measurable function of \(W_t\).
Conditional data processing therefore gives
\[
    I\bigl((L_t^+,L_t^-);U_t\mid G_{t-1}\bigr)
    \le
    I\bigl(W_t;U_t\mid G_{t-1}\bigr).
\]
Finally, \eqref{eq:two-coordinate-scmi-bound-d} follows from Jensen's inequality.
\end{proof}

\begin{remark}[Role of the exchangeability assumption]
\label{rem:role-exchangeability-two-coordinate}
The bound above is the fully exchangeability-free version of the SCMI argument.
It is useful when row-wise exchangeability is unavailable{. Its information term involves the loss difference \(L_t^+-L_t^-\), whereas Theorem~\ref{thm:general-rademacher} uses the single-coordinate loss \(L_t^+\).}
Row-wise exchangeability and {Eq.~\eqref{eq:selected-feedback-kernel}} are what reduce this two-coordinate quantity to the cleaner one-coordinate loss \(L_t^+\), which is the form used for the main SCMI, shifted-Rademacher, and Bernstein developments.
\end{remark}

\section{Directed-information interpretation}
\label{app:directed-information}

{We compare the information terms in Theorem~\ref{thm:general-rademacher} with directed information and explain why, given the learner history and the current pair, the loss \(L_t^+\) can reveal less about the selector \(U_t\) than the learner state \(W_t\) does.}

Classical directed information measures causal information flow in a time series. For two processes \(X^n=(X_1,\ldots,X_n)\) and \(Y^n=(Y_1,\ldots,Y_n)\), the directed information is
\[
I(X^n\to Y^n)
\coloneqq 
\sum_{t=1}^n I(X^t;Y_t\mid Y^{t-1}),
\]
and its causally conditioned version is commonly written as
\[
I(X^n\to Y^n\Vert Z^n)
\coloneqq 
\sum_{t=1}^n I(X^t;Y_t\mid Y^{t-1},Z^t).
\]
These quantities were introduced to distinguish causal information flow from ordinary symmetric MI in systems with feedback~\citep{massey1990causality,kramer1998directed}; related interpretations appear in feedback capacity, control, compression, and hypothesis testing~\citep{tatikonda2009control}.

The state-based line in Theorem~\ref{thm:general-rademacher} has the same causal form. Define {the sum of CMI terms}
\[
\mathcal I_n^{\mathrm{SCMI}}
\coloneqq 
\sum_{t=1}^n I(W_t;U_t\mid G_{t-1}).
\]
Because MI is symmetric in its first two arguments, the increment may also be read as \(I(U_t;W_t\mid G_{t-1})\). It measures how much information about the current selector is injected into the updated state \(W_t\), after {}{the learner history \(H_{t-1}\) and the current pair \(\widetilde Z_t\)} have already been fixed. The conditioning is essential: \(G_{t-1}=(H_{t-1},\widetilde Z_t)\) removes information already present in the learner history and also removes the intrinsic randomness of the current pair \((Z_{t,0},Z_{t,1})\). What remains is precisely the information created by selecting one coordinate and updating the learner with it.

This is why a directed quantity appears in sequential learning. At round \(t\), the learner state is adapted to the past and is evaluated immediately after the {update in Eq.~\eqref{eq:selected-feedback-kernel}}. The relevant overfitting event is therefore local in time: did the update reveal which coordinate of the current row was selected? A single terminal parameter-CMI term such as \(I(W_n;U^n\mid \widetilde Z^n)\) is not aligned with this target. It may charge information about selectors that no longer affects the evaluated loss, and it may miss {}{information about earlier selectors} if the terminal state later forgets it. The SCMI sum instead charges each round at the moment where that information can affect the train--holdout comparison.

{The first upper bound in Eq.~\eqref{eq:general-rademacher-bound} uses the evaluated loss \(L_t^+\). The data-processing inequality gives}
\[
I(L_t^+;U_t\mid G_{t-1})
\le
I(W_t;U_t\mid G_{t-1}).
\]
Indeed, after conditioning on \(G_{t-1}\), the variables \(Q_t\), \(Z_{t,0}\), and \(Z_{t,1}\) are fixed, and
\[
L_t^+=Q_t\ell(W_t,Z_{t,0})
\]
is a measurable function of \(W_t\). Data processing therefore gives the inequality. Combining it with Theorem~\ref{thm:general-rademacher} yields
\[
|\mathrm{Err}_n^{\mathrm{seq}}|
\le
2\sqrt{\frac{2}{n}\,\mathcal I_n^{\mathrm{SCMI}}}.
\]

Thus SCMI is the sequential counterpart of batch CMI. Batch CMI measures how much a terminal output reveals about the whole selector vector after conditioning on the paired supersample. SCMI decomposes the same idea into causal increments and conditions away the already available past at each round. In online learning this yields the direct roundwise bound of Theorem~\ref{thm:general-rademacher}; in active learning and bandits, the same {CMI argument} is combined with application-specific row-swap identities.

\section{Sequential covers, Littlestone dimension, and SCMI pattern bounds}
\label{app:seq-covers-littlestone}

This appendix gives the sequential pattern argument used in
Section~\ref{sec:online-application}. We follow the terminology of
\citet{rakhlin2015online}, but adapt the notation to the online supersample setting. The sequential-complexity results recalled here, including sequential covers, sequential fat-shattering and Littlestone dimensions, and the sequential Sauer--Shelah bound, are all standard results from \citet{rakhlin2015online}. Our paper-specific step is to translate these results into the online supersample notation and connect them to SCMI pattern counting. The pattern bound below should be read as a complexity interpretation under the stated pathwise realizability assumption, rather than as an automatic consequence of Theorem~\ref{thm:general-rademacher}.

\paragraph{Trees and sequential Rademacher complexity.}
A \(\mathcal Z\)-valued tree \(\mathbf z\) of depth \(n\) is a complete rooted
binary tree whose nodes are labeled by elements of \(\mathcal Z\).  We identify
\(\mathbf z\) with labeling functions
\[
\mathbf z=(\mathbf z_1,\ldots,\mathbf z_n),
\qquad
\mathbf z_t:\{\pm1\}^{t-1}\to\mathcal Z .
\]
Thus \(\mathbf z_1\in\mathcal Z\) is the label of the root, and
\(\mathbf z_t(\epsilon_{1:t-1})\) is the label of the node reached by the
prefix \(\epsilon_{1:t-1}\).  For a full path
\(\epsilon=(\epsilon_1,\ldots,\epsilon_n)\in\{\pm1\}^n\), we write
\[
\mathbf z_t(\epsilon)
\coloneqq 
\mathbf z_t(\epsilon_{1:t-1})
\]
as shorthand.

For a real-valued class \(\mathcal L\subseteq\mathbb R^{\mathcal Z}\), the
sequential Rademacher complexity of \(\mathcal L\) on \(\mathbf z\) is
\[
\mathfrak R_n(\mathcal L,\mathbf z)
\coloneqq 
\E_{\epsilon}
\left[
\sup_{\ell\in\mathcal L}
\frac1n\sum_{t=1}^n
\epsilon_t\,\ell\bigl(\mathbf z_t(\epsilon)\bigr)
\right],
\]
where \(\epsilon_1,\ldots,\epsilon_n\) are independent Rademacher variables.
The worst-case sequential Rademacher complexity is
\[
\mathfrak R_n(\mathcal L)
\coloneqq 
\sup_{\mathbf z}
\mathfrak R_n(\mathcal L,\mathbf z),
\]
where the supremum is over all \(\mathcal Z\)-valued depth-\(n\) trees.

If \(\mathbf z_t(\epsilon)\) does not depend on \(\epsilon_{1:t-1}\), then the
tree degenerates into a fixed tuple \((z_1,\ldots,z_n)\), and the display
above becomes the usual fixed-sample Rademacher average.  Thus sequential
Rademacher complexity is the standard replacement of a fixed sample by a
binary tree in online learning.

In the online supersample construction, the corresponding prefix is the
selector history.  With
\[
\epsilon_t=2U_t-1,
\]
the prefix \(\epsilon_{1:t-1}\) is equivalent to \(U_{<t}\).  This prefix
determines the counterfactual pre-round history under which the current paired
row is drawn.  The conditioning object
\[
G_{t-1}^{\mathrm{on}}
=
(H_{t-1}^{\mathrm{on}},\widetilde Z_t)
\]
fixes the realized history and the current paired row before the current
branch \(U_t\) is averaged over.  Hence the SCMI term in the main text is a
data-dependent analogue of the tree-indexed signed correlation above: it keeps
the realized online experiment and measures the information carried by the
realized loss, instead of taking a worst-case supremum over all deterministic
trees.

\paragraph{Sequential covers.}
Let \(\mathbf z\) be a fixed \(\mathcal Z\)-valued tree of depth \(n\), and let
\(\mathcal L\subseteq\mathbb R^{\mathcal Z}\).  A finite collection
\[
\mathcal V=\{\mathbf v^1,\ldots,\mathbf v^M\}
\]
of real-valued trees of depth \(n\) is an \(\alpha\)-cover of \(\mathcal L\)
on \(\mathbf z\) with respect to the sequential \(\ell_p\) metric if, for
every \(\ell\in\mathcal L\) and every path
\(\epsilon\in\{\pm1\}^n\), there exists \(m\in[M]\) such that
\[
\left(
\frac1n\sum_{t=1}^n
\left|
\mathbf v^m_t(\epsilon)
-
\ell\bigl(\mathbf z_t(\epsilon)\bigr)
\right|^p
\right)^{1/p}
\le
\alpha .
\]
For \(p=\infty\), this is interpreted as the maximum over \(t\le n\).  The
sequential covering number is
\[
\mathcal N_p(\alpha,\mathcal L,\mathbf z)
\coloneqq 
\min\{|\mathcal V|:\mathcal V
\text{ is an \(\alpha\)-cover of \(\mathcal L\) on \(\mathbf z\)}\}.
\]
The worst-case sequential covering number is
\[
\mathcal N_p(\alpha,\mathcal L,n)
\coloneqq 
\sup_{\mathbf z}
\mathcal N_p(\alpha,\mathcal L,\mathbf z).
\]

These covers are pathwise.  They do not require one code tree to match a
function on all nodes simultaneously.  Rather, for each \(\ell\) and each
path, there must exist a code tree that matches \(\ell\) along that path.
This is the key difference from counting all labelings of all nodes.

The basic sequential covering bound is the analogue of the finite-class
Rademacher bound.  If \(\mathcal L\subseteq[-1,1]^{\mathcal Z}\), then for any
\(\alpha>0\),
\[
\mathfrak R_n(\mathcal L,\mathbf z)
\le
\alpha
+
\sqrt{
\frac{2\log \mathcal N_1(\alpha,\mathcal L,\mathbf z)}{n}
}.
\]
A sharper form is obtained by integrating the sequential entropy.  For
\(p\ge1\), define
\[
\mathfrak D^p_n(\mathcal L,\mathbf z)
\coloneqq 
\inf_{\alpha\ge0}
\left\{
4\alpha
+
\frac{12}{\sqrt n}
\int_{\alpha}^{1}
\sqrt{\log \mathcal N_p(\delta,\mathcal L,\mathbf z)}\,d\delta
\right\}.
\]
Then, in the bounded case,
\[
\mathfrak R_n(\mathcal L,\mathbf z)
\le
\mathfrak D^2_n(\mathcal L,\mathbf z).
\]
Thus the sequential covering numbers play the same role for sequential
Rademacher complexity as classical covering numbers play in the i.i.d. theory.

\paragraph{Sequential fat-shattering and Littlestone dimension.}
For real-valued classes, the relevant sequential combinatorial parameter is
the sequential fat-shattering dimension.  A depth-\(d\) \(\mathcal Z\)-valued
tree \(\mathbf z\) is \(\alpha\)-shattered by
\(\mathcal L\subseteq\mathbb R^{\mathcal Z}\) if there exists a real-valued
witness tree \(\mathbf s\) of depth \(d\) such that, for every path
\(\epsilon\in\{\pm1\}^d\), there exists \(\ell_\epsilon\in\mathcal L\) with
\[
\epsilon_t
\left(
\ell_\epsilon\bigl(\mathbf z_t(\epsilon)\bigr)
-
\mathbf s_t(\epsilon)
\right)
\ge
\frac{\alpha}{2},
\qquad
t=1,\ldots,d .
\]
The sequential fat-shattering dimension
\(\operatorname{fat}_{\alpha}(\mathcal L)\) is the largest depth \(d\) of an
\(\alpha\)-shattered tree.

This dimension controls sequential covering numbers.  For
\(\mathcal L\subseteq[-1,1]^{\mathcal Z}\),
\[
\mathcal N_\infty(\alpha,\mathcal L,n)
\le
\left(\frac{2en}{\alpha}\right)^{\operatorname{fat}_{\alpha}(\mathcal L)}.
\]
This is the sequential analogue of the classical fat-shattering entropy
bound.

For binary losses, we specialize to
\[
\mathcal L\subseteq\{0,1\}^{\mathcal Z}.
\]
A depth-\(d\) tree \(\mathbf z\) is shattered by \(\mathcal L\) if, for every
path \(\epsilon\in\{\pm1\}^d\), there exists
\(\ell_\epsilon\in\mathcal L\) such that
\[
\ell_\epsilon\bigl(\mathbf z_t(\epsilon)\bigr)
=
\mathbf 1\{\epsilon_t=1\},
\qquad
t=1,\ldots,d .
\]
The Littlestone dimension \(\operatorname{Ldim}(\mathcal L)\) is the largest
depth of such a shattered tree.

A fixed tuple \(z_1,\ldots,z_d\) is the special case of a tree satisfying
\[
\mathbf z_t(\epsilon)=z_t
\qquad
\text{for all }\epsilon.
\]
Therefore every VC-shattered set gives a shattered tree, and hence
\[
\operatorname{VCdim}(\mathcal L)
\le
\operatorname{Ldim}(\mathcal L).
\]
The inequality can be strict: a class may have small VC dimension but infinite
Littlestone dimension, reflecting the additional power of adaptive,
tree-indexed data.

\paragraph{Exact \(0\)-covers and sequential Sauer--Shelah.}
When \(\mathcal L\subseteq\{0,1\}^{\mathcal Z}\), the exact sequential
\(0\)-covering number is
\[
\mathcal N(0,\mathcal L,\mathbf z)
\coloneqq 
\mathcal N_p(0,\mathcal L,\mathbf z),
\]
which is independent of \(p\).  Equivalently, it is the minimum number of
binary-valued code trees needed so that, for every
\(\ell\in\mathcal L\) and every path \(\epsilon\), one code tree reproduces
the binary sequence
\[
\left(
\ell(\mathbf z_1(\epsilon)),
\ldots,
\ell(\mathbf z_n(\epsilon))
\right)
\]
along that path.  Define
\[
\mathcal N(0,\mathcal L,n)
\coloneqq 
\sup_{\mathbf z}\mathcal N(0,\mathcal L,\mathbf z).
\]

In the notation of the main text, we write
\[
\Pi_{\rm seq}(\mathcal L,n)
\coloneqq 
\mathcal N(0,\mathcal L,n).
\]
The sequential Sauer--Shelah lemma of \citet{rakhlin2015online} gives,
for \(d=\operatorname{Ldim}(\mathcal L)<\infty\),
\[
\Pi_{\rm seq}(\mathcal L,n)
=
\mathcal N(0,\mathcal L,n)
\le
\sum_{i=0}^{\min\{d,n\}}\binom{n}{i}.
\]
This is the binary case \(k=1\) of the bound
\(\mathcal N(0,\mathcal G,T)\le\sum_{i=0}^{d}\binom{T}{i}k^i\)
in \citet{rakhlin2015online}. Their displayed bound \((ekT)^d\) is a
coarser consequence; for \(k=1\), the exact sum above also gives the
sharper standard estimate
\[
\Pi_{\rm seq}(\mathcal L,n)
\le
\left(\frac{en}{d}\right)^d,
\qquad 1\le d\le n.
\]
This is precisely the tree-valued analogue of the classical Sauer--Shelah
bound on the number of projections of a VC class on a fixed sample.

\paragraph{SCMI as a pathwise pattern-counting quantity.}
We now connect the exact sequential pattern count to the online SCMI term.
The selector variables \(U_1,\ldots,U_n\) are independent Bernoulli variables,
and the signed selectors are
\[
\epsilon_t=2U_t-1.
\]
Consider a full counterfactual construction of the online process.  For every
prefix \(u_{<t}\in\{0,1\}^{t-1}\), let
\[
H_{t-1}(u_{<t})
\]
be the counterfactual pre-round history obtained by following that prefix.
At the corresponding node, draw the paired row
\[
\widetilde Z_t(u_{<t})
=
\bigl(Z_{t,0}(u_{<t}),Z_{t,1}(u_{<t})\bigr)
\]
from the conditional law specified by \(H_{t-1}(u_{<t})\).  The branch
\(u_t\) selects the training coordinate \(Z_{t,u_t}(u_{<t})\) and determines
the next history.  The actually observed online process is the single path
\(U^n\):
\[
H_{t-1}^{\mathrm{on}}
=
H_{t-1}(U_{<t}),
\qquad
\widetilde Z_t
=
\widetilde Z_t(U_{<t}).
\]
This construction is random, but after conditioning on its full realization
it induces a deterministic prefix-indexed tree of loss-relevant objects.

Let \(\mathsf T_n\) denote this full counterfactual construction, including
all histories, paired rows, and any auxiliary learner randomness. This is only a
{coupling for the analysis}: it samples all node-level randomness so that the realized
online process can be read off along the random selector path. It does not give
the learner access to counterfactual rows.

\paragraph{Pathwise realizability assumption.}
After conditioning on \(\mathsf T_n\), assume that, for every selector path
\(u^n\in\{0,1\}^n\), there exists \(\ell_{u^n}\in\mathcal L\) such that the
corresponding loss path satisfies
\[
L_t^+(u^n)
=
\ell_{u^n}\!\left(
\mathbf z_t(2u_1-1,\ldots,2u_{t-1}-1)
\right),
\qquad
t=1,\ldots,n.
\]
Here \(\mathbf z\) is the \(\mathcal Z\)-valued tree induced by the fixed
realization of \(\mathsf T_n\).  If the loss depends on a history, a
post-update state, or learner randomness, those quantities are included in the
loss-relevant object labeling \(\mathbf z\).

\begin{theorem}[Sequential pattern-growth bound]
\label{thm:online-seq-growth}
Under the pathwise realizability assumption above,
\begin{align}
\sum_{t=1}^n
I(L_t^+;U_t\mid G_{t-1}^{\mathrm{on}})
\le
\log \Pi_{\rm seq}(\mathcal L,n).
\end{align}
If \(d=\operatorname{Ldim}(\mathcal L)<\infty\), then
\begin{align}
\sum_{t=1}^n
I(L_t^+;U_t\mid G_{t-1}^{\mathrm{on}})
\le
\log\!\left(
\sum_{i=0}^{\min\{d,n\}}\binom{n}{i}
\right).
\end{align}
In particular, if \(1\le d\le n\), then
\begin{align}
\sum_{t=1}^n
I(L_t^+;U_t\mid G_{t-1}^{\mathrm{on}})
\le
d\log\!\left(\frac{en}{d}\right).
\end{align}
If \(d=0\), the right-hand side is interpreted as \(0\).
\end{theorem}

\begin{proof}
We regard the full counterfactual construction \(\mathsf T_n\) as generated independently of {the selector vector \(U^n\)}; the realized online process is obtained by evaluating this construction along the random path \(U^n\).  For each realization \(\tau\) of \(\mathsf T_n\), let
\[
\mathcal V(\tau)
=
\{\mathbf v^1,\ldots,\mathbf v^{M(\tau)}\}
\]
be a minimum exact \(0\)-cover for \(\mathcal L\) on the tree induced by \(\tau\).  Then
\[
M(\tau)
=
\mathcal N(0,\mathcal L,\mathbf z(\tau))
\le
\Pi_{\rm seq}(\mathcal L,n).
\]
For each selector path \(u^n\), choose a code index
\[
C=C(u^n,\tau)\in[M(\tau)]
\]
using a fixed tie-breaking rule so that the selected code tree reproduces the loss path along \(u^n\).  Along the realized selector path this gives
\[
L_t^+
=
\mathbf v^C_t(2U_1-1,\ldots,2U_{t-1}-1),
\qquad
 t=1,\ldots,n.
\]
Thus, conditional on \((\mathsf T_n,U_{<t})\), the variable \(L_t^+\) is a deterministic function of \(C\).

Set
\[
\Xi_t\coloneqq(\mathsf T_n,U_{<t}).
\]
The conditioning object
\[
G_{t-1}^{\mathrm{on}}
=
(H_{t-1}^{\mathrm{on}},\widetilde Z_t)
\]
is a measurable function of \(\Xi_t\).  {Since the selectors are generated} independently of \(\mathsf T_n\), the current selector \(U_t\) is independent of \(\Xi_t\).  Because \(G_{t-1}^{\mathrm{on}}\) is a function of \(\Xi_t\), this also implies
\[
U_t\perp \Xi_t\mid G_{t-1}^{\mathrm{on}}.
\]
Therefore,
\[
\begin{aligned}
I(L_t^+;U_t\mid G_{t-1}^{\mathrm{on}})
&\le
I(C,\Xi_t;U_t\mid G_{t-1}^{\mathrm{on}})\\
&=
I(\Xi_t;U_t\mid G_{t-1}^{\mathrm{on}})
+
I(C;U_t\mid \Xi_t,G_{t-1}^{\mathrm{on}})\\
&=
I(C;U_t\mid \Xi_t)\\
&=
I(C;U_t\mid \mathsf T_n,U_{<t}).
\end{aligned}
\]
Here the first term in the chain rule vanishes by the conditional independence above, and the third line uses that \(G_{t-1}^{\mathrm{on}}\) is already determined by \(\Xi_t\).
Summing over \(t\) and using the chain rule for MI yields
\[
\sum_{t=1}^n
I(L_t^+;U_t\mid G_{t-1}^{\mathrm{on}})
\le
\sum_{t=1}^n
I(C;U_t\mid \mathsf T_n,U_{<t})
=
I(C;U^n\mid \mathsf T_n).
\]
Finally,
\[
\begin{aligned}
I(C;U^n\mid \mathsf T_n)
&\le
H(C\mid \mathsf T_n)\\
&=
\mathbb E\!\left[H(C\mid \mathsf T_n=\tau)\right]\\
&\le
\mathbb E\!\left[\log M(\tau)\right]
\le
\log \Pi_{\rm seq}(\mathcal L,n).
\end{aligned}
\]
Thus
\[
\sum_{t=1}^n
I(L_t^+;U_t\mid G_{t-1}^{\mathrm{on}})
\le
\log \Pi_{\rm seq}(\mathcal L,n).
\]
Combining this with the sequential Sauer--Shelah lemma gives, for
\(d=\operatorname{Ldim}(\mathcal L)<\infty\),
\[
\sum_{t=1}^n
I(L_t^+;U_t\mid G_{t-1}^{\mathrm{on}})
\le
\log\!\left(\sum_{i=0}^{\min\{d,n\}}\binom{n}{i}\right).
\]
If \(1\le d\le n\), this is further bounded by
\[
 d\log\!\left(\frac{en}{d}\right),
\]
while for \(d=0\) the logarithmic term is \(0\).
\end{proof}

\paragraph{Interpretation.}
The proof is the sequential counterpart of the {batch argument relating functional CMI to the VC dimension}.
In the batch case, after conditioning on the fixed paired supersample, the
random prediction pattern takes values in a finite set whose size is
controlled by VC dimension.  In the online case, conditioning on the full
counterfactual experiment leaves a finite set of pathwise {loss trees}
whose size is controlled by Littlestone dimension.  The SCMI term is therefore
controlled by the number of loss patterns visible along selector paths, not
by the amount of information stored in the learner's state or posterior.

\section{Fast-rate methods by Bernstein and Shifted-Rademacher bounds}
\label{app:bernstein-consequences}

The slow-rate theorem controls {a correlation between the loss and \(\varepsilon_t\)} by the square root of an information term. Fast rates require additional structure. This appendix records the two mechanisms used later: the Bernstein refinement and its explicit consequences in Appendices~\ref{app:proofs-bernstein-framework}--\ref{subsec:bernstein-additional-results}, and the Shifted-Rademacher refinement in Appendix~\ref{app:wang-refinement}, where nonnegativity lets the training risk appear multiplicatively.

\subsection{Proof of the Bernstein bound}
\label{app:proofs-bernstein-framework}
\label{app:proof-general-bernstein}

If \(b=0\), the signed variables vanish almost surely and the Bernstein inequality is trivial.  In the rest of this subsection assume \(b>0\) and define the standard Bennett--Bernstein envelope
\begin{align}
\psi_b(\lambda) \coloneqq  \frac{e^{\lambda b}-\lambda b-1}{b^2},
\qquad 0 \le \lambda < \frac{3}{b}.
\label{eq:def-psi-b}
\end{align}
The following one-step inequality is required in the proof.

\begin{lemma}
\label{lem:one-step-bernstein}
Let \(G\) be a random element, let {\(L\) be a real-valued random variable with \(|L|\le b\) almost surely}, and let \(U\in\{0,1\}\) satisfy
\[
\mathbb P(U=0\mid G)=\mathbb P(U=1\mid G)=\frac12
\qquad\text{a.s.}
\]
If \(\varepsilon\coloneqq 2U-1\), then for every \(\lambda\in[0,3/b)\),
\begin{align}
\lambda\,\mathbb E\!\left[\mathbb E\!\left[\varepsilon {L}\middle|G\right]\right]
\le
I({L};U\mid G)
+
\psi_b(\lambda)\,
\mathbb E\!\left[\mathbb E\!\left[{L}^2\middle|G\right]\right].
\label{eq:one-step-bernstein}
\end{align}
Moreover,
\begin{align}
\psi_b(\lambda)
\le
\frac{\lambda^2}{2(1-\lambda b/3)}.
\label{eq:psi-b-bound}
\end{align}
\end{lemma}

\begin{proof}
The conditional Donsker--Varadhan variational formula applied to the conditional law of \(({L},U)\) given \(G\) gives
\begin{align}
I({L};U\mid G)
&\ge
\lambda\,\mathbb E\!\left[\mathbb E[\varepsilon {L}\mid G]\right]
-
\mathbb E\!\left[
\log \mathbb E\!\left[
e^{\lambda \varepsilon'{L}}
\middle|G
\right]\right],
\label{eq:proof-bernstein-dv}
\end{align}
where \(\varepsilon'\) is an independent Rademacher sign under the product conditional law given \(G\). We use the standard Bennett mgf envelope. Let
\[
\varphi(u)\coloneqq e^u-u-1,
\qquad
r(u)\coloneqq
\begin{cases}
\varphi(u)/u^2, & u\ne 0,\\[1mm]
1/2, & u=0.
\end{cases}
\]
An elementary calculus check shows that \(r\) is nondecreasing on \(\mathbb R\). Indeed, for \(u\ne0\),
\[
r'(u)=\frac{a(u)}{u^3},
\qquad
 a(u):=(u-2)e^u+u+2.
\]
Since \(a(0)=a'(0)=0\) and \(a''(u)=ue^u\), we have \(a(u)\ge0\) for \(u>0\) and \(a(u)\le0\) for \(u<0\). Hence \(a(u)/u^3\ge0\), so \(r'(u)\ge0\). Therefore, for every \(x\in[-b,b]\),
\[
\varphi(\lambda x)
=
\lambda^2x^2 r(\lambda x)
\le
\lambda^2x^2 r(\lambda b)
=
\frac{x^2}{b^2}\varphi(\lambda b).
\]
Set \(X\coloneqq\varepsilon'{L}\). Under the product conditional law, \(\mathbb E[X\mid G]=0\) and \(|X|\le b\). Therefore
\begin{align*}
\mathbb E[e^{\lambda \varepsilon'{L}}\mid G]
&=
1+\mathbb E[\varphi(\lambda X)\mid G]  \\
&\le
1+\frac{\varphi(\lambda b)}{b^2}\,\mathbb E[{L}^2\mid G]
\le
\exp\!\left\{\psi_b(\lambda)\,\mathbb E[{L}^2\mid G]\right\}.
\end{align*}
Taking logarithms and substituting into \eqref{eq:proof-bernstein-dv} proves \eqref{eq:one-step-bernstein}.

It remains to bound \(\psi_b\). For \(u\in[0,3)\), the standard Bernstein scalar inequality is
\[
\varphi(u)=e^u-u-1\le \frac{u^2}{2(1-u/3)}.
\]
For completeness, this follows without using a power series: if
\(F(u)\coloneqq u^2/(2(1-u/3))-\varphi(u)\), then
\(F(0)=F'(0)=0\) and
\[
F''(u)=(1-u/3)^{-3}-e^u\ge0,
\qquad 0\le u<3,
\]
because \(e^u\le(1-u/3)^{-3}\), equivalently
\(u\le -3\log(1-u/3)\). Substituting \(u=\lambda b\) and dividing by \(b^2\) proves \eqref{eq:psi-b-bound}.
\end{proof}

\begin{proof}[Proof of Theorem~\ref{thm:general-bernstein}]
By the definitions of \(R_n^{\mathrm{ho}}\) and \(R_n^{\mathrm{tr}}\), the holdout--training excess gap is the signed coordinate gap:
\begin{align}
R_n^{\mathrm{ho}}-R_n^{\mathrm{tr}}
=
\frac{2}{n}\sum_{t=1}^n
\mathbb E\!\left[
\mathbb E\!\left[
\varepsilon_t e_{t,0}(w^\star)
\middle| G_{t-1}
\right]\right].
\label{eq:bernstein-direct-gap}
\end{align}
This is the same one-coordinate calculation as Lemma~\ref{lem:basic-seq-sym}, applied to the signed excess losses.  The equality is purely algebraic and linear in the evaluated coordinate loss, so nonnegativity of the original loss is not needed here.

Apply Lemma~\ref{lem:one-step-bernstein} at each time \(t\) with
\[
G=G_{t-1},
\qquad
{L}=e_{t,0}(w^\star),
\qquad
U=U_t.
\]
Summing the resulting inequalities and using \eqref{eq:bernstein-direct-gap} gives
\begin{align}
R_n^{\mathrm{ho}}-R_n^{\mathrm{tr}}
&\le
\frac{2}{\lambda n}\sum_{t=1}^n I(e_{t,0}(w^\star);U_t\mid G_{t-1})
+
\frac{2\psi_b(\lambda)}{\lambda}\,
\frac1n\sum_{t=1}^n
\mathbb E\!\left[
\mathbb E\!\left[(e_{t,0}(w^\star))^2\middle|G_{t-1}\right]\right].
\end{align}
{This bound also holds with \(\lvert R_n^{\mathrm{ho}}-R_n^{\mathrm{tr}}\rvert\) on the left: apply the same lemma to the negative excess loss as well.}
By \eqref{eq:psi-b-bound},
\[
\frac{2\psi_b(\lambda)}{\lambda}
\le
\frac{\lambda}{1-\lambda b/3}.
\]
Using the Bernstein condition \eqref{eq:general-bernstein-condition} therefore yields
\begin{align}
R_n^{\mathrm{ho}}-R_n^{\mathrm{tr}}
&\le
\frac{2}{\lambda n}\sum_{t=1}^n I(e_{t,0}(w^\star);U_t\mid G_{t-1})
+
\frac{B\lambda}{1-\lambda b/3}\,R_n^{\mathrm{ho}}.
\end{align}
Rearranging proves \eqref{eq:bernstein-fast-main}.
\end{proof}

\subsection{Additional results for the Bernstein bound}
\label{subsec:bernstein-additional-results}

\begin{corollary}[Explicit fast-rate form under the Bernstein condition]
\label{cor:bernstein-explicit}
Under the assumptions of Theorem~\ref{thm:general-bernstein}, let
\begin{align}
\lambda \coloneqq  \frac{C}{B + Cb/3},\qquad C\in(0,1).
\label{eq:bernstein-lambda-choice}
\end{align}
Then
\begin{align}
R_n^{\mathrm{ho}}
\le
\frac{1}{1-C}\,R_n^{\mathrm{tr}}
+
\frac{2(B+Cb/3)}{C(1-C)n}\sum_{t=1}^n I(e_{t,0}(w^\star);U_t \mid G_{t-1}),
\label{eq:bernstein-explicit-C}
\end{align}
and, in particular, choosing \(C=1/2\) gives
\begin{align}
R_n^{\mathrm{ho}}
\le
2R_n^{\mathrm{tr}}
+
\frac{8(B+b/6)}{n}\sum_{t=1}^n I(e_{t,0}(w^\star);U_t \mid G_{t-1}).
\label{eq:bernstein-explicit-half}
\end{align}
\end{corollary}

\begin{proposition}[Bounded unit-interval excess process]
\label{prop:bernstein-bounded-special}
Assume, in addition to the setup of Theorem~\ref{thm:general-bernstein}, that
\[
0 \le e_{t,0}(w^\star),e_{t,1}(w^\star) \le 1
\qquad\text{a.s. for all }t,
\]
and that
\begin{align}
R_n^{\mathrm{tr}} \le R_n^{\mathrm{ho}}.
\label{eq:bernstein-bounded-ordering}
\end{align}
Then the Bernstein condition \eqref{eq:general-bernstein-condition} holds with
\(B=1\). Consequently,
\begin{align}
R_n^{\mathrm{ho}}
\le
2R_n^{\mathrm{tr}}
+
\frac{28}{3n}\sum_{t=1}^n I(e_{t,0}(w^\star);U_t \mid G_{t-1}).
\label{eq:bernstein-bounded-special}
\end{align}
\end{proposition}

\paragraph{Proof of Corollary~\ref{cor:bernstein-explicit}.}
\begin{proof}
{If \(b=0\), all excess losses vanish and the claim is trivial; hence assume \(b>0\).}
Fix \(C\in(0,1)\) and choose \(\lambda\) as in \eqref{eq:bernstein-lambda-choice}. First,
\[
\lambda=\frac{C}{B+Cb/3}<\frac3b,
\]
so Theorem~\ref{thm:general-bernstein} applies. Moreover,
\[
1-\frac{\lambda b}{3}
=
\frac{B}{B+Cb/3},
\qquad
\frac{B\lambda}{1-\lambda b/3}=C.
\]
Substituting this into \eqref{eq:bernstein-fast-main} gives \eqref{eq:bernstein-explicit-C}. Setting \(C=1/2\) gives \eqref{eq:bernstein-explicit-half}.
\end{proof}

\paragraph{Proof of Proposition~\ref{prop:bernstein-bounded-special}.}
\label{app:proof-bernstein-bounded-special}
\begin{proof}
We first verify the second-moment condition with \(B=1\). Since
\(0\le e_{t,0}(w^\star)\le1\),
\begin{align}
(e_{t,0}(w^\star))^2 \le e_{t,0}(w^\star).
\label{eq:bounded-special-step1}
\end{align}
Hence
\begin{align}
\frac1n\sum_{t=1}^n
\mathbb E\!\left[
\mathbb E\!\left[(e_{t,0}(w^\star))^2\middle|G_{t-1}\right]\right]
&\le
\frac1n\sum_{t=1}^n
\mathbb E\!\left[
\mathbb E\!\left[e_{t,0}(w^\star)\middle|G_{t-1}\right]\right].
\label{eq:bounded-special-square-reduction}
\end{align}
Conditioning on \(U_t\) and using row-wise exchangeability as in Lemma~\ref{lem:basic-seq-sym},
\begin{align}
\mathbb E\!\left[
\mathbb E\!\left[e_{t,0}(w^\star)\middle|G_{t-1}\right]\right]
=
\frac12\mathbb E\!\left[
\mathbb E\!\left[e_{t,1-U_t}(w^\star)\middle|G_{t-1}\right]\right]
+
\frac12\mathbb E\!\left[
\mathbb E\!\left[e_{t,U_t}(w^\star)\middle|G_{t-1}\right]\right].
\label{eq:bounded-special-branch-average}
\end{align}
Summing over \(t\) gives
\begin{align}
\frac1n\sum_{t=1}^n
\mathbb E\!\left[
\mathbb E\!\left[e_{t,0}(w^\star)\middle|G_{t-1}\right]\right]
=
\frac{R_n^{\mathrm{tr}}+R_n^{\mathrm{ho}}}{2}
\le
R_n^{\mathrm{ho}},
\end{align}
where the last inequality uses \eqref{eq:bernstein-bounded-ordering}. Combining this with
\eqref{eq:bounded-special-square-reduction} verifies \eqref{eq:general-bernstein-condition} with \(B=1\).

Since the bounded unit-interval assumptions also give \(b=1\), Corollary~\ref{cor:bernstein-explicit} with \(B=1\) and \(b=1\) yields
\begin{align*}
R_n^{\mathrm{ho}}
&\le
2R_n^{\mathrm{tr}}
+
\frac{8(1+1/6)}{n}\sum_{t=1}^n I(e_{t,0}(w^\star);U_t\mid G_{t-1})  \\
&=
2R_n^{\mathrm{tr}}
+
\frac{28}{3n}\sum_{t=1}^n I(e_{t,0}(w^\star);U_t\mid G_{t-1}),
\end{align*}
which is \eqref{eq:bernstein-bounded-special}.
\end{proof}

\subsection{Shifted-Rademacher fast-rate bounds}
\label{app:wang-refinement}
\label{subsec:general-wang}

Throughout this subsection we work under {Eq.}~\eqref{eq:selected-feedback-kernel} and row-wise exchangeability~\eqref{eq:row-exchange-general}. The same one-coordinate symmetrization can be shifted so that the training risk appears multiplicatively. This gives a sequential fast-rate refinement for nonnegative losses: once the training risk is small, the holdout risk is controlled by a linear information penalty rather than by a square-root penalty. The following results are straightforward sequential extensions of the corresponding arguments in \citet{wang23}. For completeness, we provide the proofs.

\begin{theorem}[General Shifted-Rademacher fast-rate bound]
\label{thm:general-wang}
Suppose that the SCMI setting satisfies {Eq.}~\eqref{eq:selected-feedback-kernel} and row-wise exchangeability~\eqref{eq:row-exchange-general}. Assume that \(\ell : \mathcal W \times \mathcal Z \to [0,1]\). Let \(C,\eta > 0\) satisfy
\begin{align}
e^{2\eta} + e^{-2\eta(C+1)} \le 2.
\label{eq:wang-feasible}
\end{align}
Then
\begin{align}
L_n^{\mathrm{ho}}
&\le
(1+C)L_n^{\mathrm{tr}}
+
\frac{1}{\eta n}\sum_{t=1}^n I(L_t^+;U_t \mid G_{t-1})
\label{eq:wang-main}
\\
&\le
(1+C)L_n^{\mathrm{tr}}
+
\frac{1}{\eta n}\sum_{t=1}^n I(W_t;U_t \mid G_{t-1}),
\label{eq:wang-directed}
\end{align}
where the second line follows by data processing.
\end{theorem}

\begin{corollary}[Explicit Shifted-Rademacher form]
\label{cor:general-wang-explicit}
Under the assumptions of Theorem~\ref{thm:general-wang}, let
\[
    J_L\coloneqq \sum_{t=1}^n I(L_t^+;U_t\mid G_{t-1}),
    \qquad
    J_W\coloneqq \sum_{t=1}^n I(W_t;U_t\mid G_{t-1}).
\]
Then
\begin{align}
L_n^{\mathrm{ho}}
&\le
L_n^{\mathrm{tr}}
+
\frac{8}{n}J_L
+
4\sqrt{\frac{2L_n^{\mathrm{tr}}}{n}J_L}
\label{eq:wang-explicit-loss}
\\
&\le
L_n^{\mathrm{tr}}
+
\frac{8}{n}J_W
+
4\sqrt{\frac{2L_n^{\mathrm{tr}}}{n}J_W}.
\label{eq:wang-explicit}
\end{align}
If \(L_n^{\mathrm{tr}}=0\), then
\begin{align}
L_n^{\mathrm{ho}}
&\le
\frac{2}{n\log 2}J_L
\le
\frac{2}{n\log 2}J_W.
\label{eq:wang-interpolation}
\end{align}
\end{corollary}

In the following proofs, we define
\begin{align}
\widetilde\varepsilon_t \coloneqq  \varepsilon_t - \frac{C}{C+2}.
\label{eq:def-shifted-sign}
\end{align}

\begin{lemma}
\label{lem:general-shifted-symmetrization}
Under row-wise exchangeability and \eqref{eq:selected-feedback-kernel},
\begin{align}
L_n^{\mathrm{ho}} - (1+C)L_n^{\mathrm{tr}}
=
\frac{C+2}{n}\sum_{t=1}^n \mathbb E\!\left[\mathbb E\!\left[\widetilde\varepsilon_t L_t^+\middle| G_{t-1}\right]\right].
\label{eq:general-shifted-symmetrization}
\end{align}
\end{lemma}

\begin{proof}
Fix \(t\in[n]\) and define \(\varphi_t\) as in \eqref{eq:proof-basic-phi}. The identities \eqref{eq:proof-basic-tr}--\eqref{eq:proof-basic-diag} imply
\begin{align}
\E\bigl[Q_t\ell(W_t,Z_{t,U_t})\bigr]
&=
\E\!\left[\varphi_t(H_{t-1},Z_{t,0},Z_{t,0})\right],
\label{eq:proof-shift-tr}
\\
\E\bigl[Q_t\ell(W_t,Z_{t,1-U_t})\bigr]
&=
\E\!\left[\varphi_t(H_{t-1},Z_{t,1},Z_{t,0})\right].
\label{eq:proof-shift-ho}
\end{align}
By selector symmetry,
\begin{align}
(C+2)\,\E[\widetilde\varepsilon_t L_t^+]
&=
\E\!\left[\varphi_t(H_{t-1},Z_{t,1},Z_{t,0})\right]
-
(1+C)\,\E\!\left[\varphi_t(H_{t-1},Z_{t,0},Z_{t,0})\right].
\label{eq:proof-shift-onebranch}
\end{align}
Combining \eqref{eq:proof-shift-tr}--\eqref{eq:proof-shift-onebranch}, summing over \(t\), and applying the tower property prove \eqref{eq:general-shifted-symmetrization}.
\end{proof}

\paragraph{Auxiliary one-step inequality.}
\label{app:proof-general-wang}
We first record the one-step conditional inequality used in the proof of
Theorem~\ref{thm:general-wang}.

\begin{lemma}
\label{lem:one-step-wang}
Let \(G\) be a random element, let {\(L\) be a \([0,1]\)-valued random variable}, and let \(U\in\{0,1\}\) satisfy
\[
\mathbb P(U=0\mid G)=\mathbb P(U=1\mid G)=\frac12
\qquad\text{a.s.}
\]
Define
\[
\varepsilon\coloneqq 2U-1,
\qquad
\widetilde\varepsilon\coloneqq \varepsilon-\frac{C}{C+2}.
\]
If \(C,\eta>0\) satisfy \eqref{eq:wang-feasible}, then
\begin{align}
(C+2)\,\mathbb E\!\left[\mathbb E\!\left[\widetilde\varepsilon {L}\middle|G\right]\right]
\le
\frac1\eta I({L};U\mid G).
\label{eq:one-step-wang}
\end{align}
\end{lemma}

\begin{proof}
{Let \(U'\) be an independent selector with law \(\mathrm{Bernoulli}(1/2)\) under the product conditional law given \(G\), and set \(\widetilde\varepsilon'=2U'-1-C/(C+2)\).}
The conditional Donsker--Varadhan variational formula applied to the conditional law of \(({L},U)\) given \(G\) yields
\begin{align}
I({L};U\mid G)
&\ge
\eta(C+2)\,
\mathbb E\!\left[\mathbb E[\widetilde\varepsilon{L}\mid G]\right]
-
\mathbb E\!\left[
\log \mathbb E\!\left[
\exp\{\eta(C+2){\widetilde\varepsilon'}{L}\}
\middle|G
\right]\right],
\label{eq:cond-dv}
\end{align}
Conditionally on \(({L},G)\),
\[
\mathbb E\!\left[
\exp\{\eta(C+2){\widetilde\varepsilon'}{L}\}
\middle|{L},G
\right]
=
\frac12 e^{2\eta {L}}
+
\frac12 e^{-2\eta(C+1){L}}.
\]
Let
\[
\phi(x)\coloneqq \frac12 e^{2\eta x}+\frac12 e^{-2\eta(C+1)x},
\qquad x\in[0,1].
\]
The function \(\phi\) is convex on \([0,1]\), so its maximum is attained at an endpoint. Since \(\phi(0)=1\) and \(\phi(1)\le1\) by \eqref{eq:wang-feasible}, we have \(\phi(x)\le1\) for all \(x\in[0,1]\). Hence the logarithmic term in \eqref{eq:cond-dv} is nonpositive, and \eqref{eq:one-step-wang} follows after dividing by \(\eta\).
\end{proof}

\begin{proof}[Proof of Theorem~\ref{thm:general-wang}]
Apply Lemma~\ref{lem:one-step-wang} at each time \(t\) with \(G=G_{t-1}\), \({L}=L_t^+\), and \(U=U_t\). Then
\[
(C+2)\,
\mathbb E\!\left[
\mathbb E[\widetilde\varepsilon_t L_t^+\mid G_{t-1}]
\right]
\le
\frac1\eta I(L_t^+;U_t\mid G_{t-1}).
\]
Summing over \(t\) and using Lemma~\ref{lem:general-shifted-symmetrization} gives \eqref{eq:wang-main}. Since \(L_t^+\) is a measurable function of \((W_t,\widetilde Z_t,Q_t)\), and \((\widetilde Z_t,Q_t)\) is fixed after conditioning on \(G_{t-1}\), data processing gives
\[
I(L_t^+;U_t\mid G_{t-1})
\le
I(W_t;U_t\mid G_{t-1}).
\]
Substituting this into \eqref{eq:wang-main} proves \eqref{eq:wang-directed}.
\end{proof}

\paragraph{Proof of Corollary~\ref{cor:general-wang-explicit}.}
\label{app:proof-general-wang-explicit}
\begin{proof}
For \(0<C\le1\), the choice \(\eta=C/8\) satisfies \eqref{eq:wang-feasible}. Indeed,
\[
e^{C/4}+e^{-C(C+1)/4}\le2
\]
for \(0<C\le1\). Therefore {the bound using \(L_t^+\) in Eq.~\eqref{eq:wang-main}} in Theorem~\ref{thm:general-wang} implies
\begin{align}
L_n^{\mathrm{ho}}
\le
(1+C)L_n^{\mathrm{tr}}
+
\frac{8}{Cn}J_L.
\label{eq:proof-explicit-start}
\end{align}
Set
\[
L\coloneqq L_n^{\mathrm{tr}},
\qquad
x\coloneqq \frac{J_L}{n}.
\]
Then \eqref{eq:proof-explicit-start} becomes
\[
L_n^{\mathrm{ho}}\le L+CL+\frac{8x}{C}.
\]
If \(L=0\), choosing \(C=1\) gives \(L_n^{\mathrm{ho}}\le8x\). If \(L>0\) and \(x=0\), letting \(C\downarrow0\) gives \(L_n^{\mathrm{ho}}\le L\). If both \(L>0\) and \(x>0\), choose \(C=\sqrt{8x/L}\) when this value is at most one, and choose \(C=1\) otherwise. These two cases give
\[
L_n^{\mathrm{ho}}
\le
L+8x+4\sqrt{2xL}.
\]
This proves \eqref{eq:wang-explicit-loss}.

By data processing, \(J_L\le J_W\). Substituting this inequality into \eqref{eq:wang-explicit-loss} gives \eqref{eq:wang-explicit}.

If \(L_n^{\mathrm{tr}}=0\), then \eqref{eq:wang-main} gives
\[
L_n^{\mathrm{ho}}
\le
\frac{J_L}{\eta n}
\]
for every feasible pair \((C,\eta)\). Letting \(C\to\infty\), the constraint \eqref{eq:wang-feasible} allows every \(\eta<(\log2)/2\). Hence
\[
L_n^{\mathrm{ho}}\le \frac{2}{n\log2}J_L.
\]
The bound with \(J_W\) follows again from \(J_L\le J_W\), proving \eqref{eq:wang-interpolation}.
\end{proof}

\section{Online learning specialization and fast rates}
\label{app:online-specialization}
\label{app:online-setting-details}

This appendix is a direct specialization of the general SCMI framework. We identify the online filtration and then set \(Q_t\equiv1\) in the slow-rate and Shifted-Rademacher bounds.

In the online specialization, \(\widetilde Z_t=(Z_{t,0},Z_{t,1})\), \(Z_t=Z_{t,U_t}\),
\[
H_0^{\mathrm{on}}=W_0,
\qquad
H_t^{\mathrm{on}}=({W_0,}Z^t,W^t),
\qquad
G_{t-1}^{\mathrm{on}}=(H_{t-1}^{\mathrm{on}},\widetilde Z_t).
\]
{Eq.~\eqref{eq:selected-feedback-kernel}} becomes
\[
\Law(W_t\mid H_{t-1}^{\mathrm{on}},\widetilde Z_t,U_t)
=
K_t^{\mathrm{on}}(\cdot\mid H_{t-1}^{\mathrm{on}},Z_t).
\]
Thus the online learner may update adaptively from the past and the current selected observation, but it cannot use {}{the other observation \(Z_{t,1-U_t}\)}.

The slow-rate result is obtained by setting \(Q_t\equiv1\) in the general theorem. The fast-rate statement is the online specialization of the Shifted-Rademacher bound from Appendix~\ref{app:wang-refinement}.

\begin{theorem}[Online slow-rate SCMI bound]
\label{thm:online-slow}
Assume the online specialization above, {Eq.~\eqref{eq:selected-feedback-kernel}}, and row-wise exchangeability. Let \(L_t^+=\ell(W_t,Z_{t,0})\) and set \(Q_t\equiv1\). Then
\begin{align}
\bigl|L_n^{\mathrm{ho}}-L_n^{\mathrm{tr}}\bigr|
&\le
\frac{2}{n}\sum_{t=1}^n
\sqrt{2I(L_t^+;U_t\mid G_{t-1}^{\mathrm{on}})}
\le
\frac{2}{n}\sum_{t=1}^n
\sqrt{2I(W_t;U_t\mid G_{t-1}^{\mathrm{on}})}
\notag\\
&\le
2\sqrt{\frac{2}{n}\sum_{t=1}^n I(W_t;U_t\mid G_{t-1}^{\mathrm{on}})}.
\label{eq:online-slow-loss}
\end{align}
\end{theorem}

{
\paragraph{Connection to pre-update regret.}
\label{app:online-pre-post}
For a fixed comparator $w^\star$, add and subtract $\ell(W_{t-1},Z_t)$ in the selected excess loss.
The comparator's {training and holdout risks agree} by selector fairness, so Theorem~\ref{thm:online-slow} gives
\begin{align}
R_n^{\mathrm{ho}}
&\le \frac1n\sum_{t=1}^n\E[\ell(W_{t-1},Z_t)-\ell(w^\star,Z_t)]\notag\\
&\quad+\frac1n\sum_{t=1}^n\E[\ell(W_t,Z_t)-\ell(W_{t-1},Z_t)]
+\frac2n\sum_{t=1}^n\sqrt{2I(L_t^+;U_t\mid G_{t-1}^{\mathrm{on}})}.
\label{eq:online-pre-regret-scmi}
\end{align}
The first {term is the expected average pre-update regret relative to $w^\star$}; the second is the {average change in training risk} due to the updates and requires a separate bound.
}

\subsection{Online specialization of the Shifted-Rademacher bound}
\label{app:online-wang-refinement}

\begin{corollary}[Online specialization of the Shifted-Rademacher bound]
\label{cor:online-wang}
Assume that the conditions of Theorem~\ref{thm:general-wang} hold under the online specialization above. Then
\begin{align}
L_n^{\mathrm{ho}}
\le
(1+C)L_n^{\mathrm{tr}}
+
\frac{1}{\eta n}\sum_{t=1}^n I(L_t^+;U_t \mid G_{t-1}^{\mathrm{on}})
\label{eq:online-wang-main}
\end{align}
and therefore
\begin{align}
L_n^{\mathrm{ho}}
\le
(1+C)L_n^{\mathrm{tr}}
+
\frac{1}{\eta n}\sum_{t=1}^n I(W_t;U_t \mid G_{t-1}^{\mathrm{on}}).
\label{eq:online-wang-directed}
\end{align}
Moreover,
\begin{align}
L_n^{\mathrm{ho}}
\le
L_n^{\mathrm{tr}}
+
\frac{8}{n}\sum_{t=1}^n I(W_t;U_t \mid G_{t-1}^{\mathrm{on}})
+
4\sqrt{\frac{2L_n^{\mathrm{tr}}}{n}\sum_{t=1}^n I(W_t;U_t \mid G_{t-1}^{\mathrm{on}})}.
\label{eq:online-wang-explicit}
\end{align}
\end{corollary}

\subsection{Proofs for the online specialization}
\label{app:proofs-online-application}

\subsubsection{Proof of Theorem~\ref{thm:online-slow}}
\label{app:proof-online-slow}

\begin{proof}
Set $Q_t\equiv1$ in Theorem~\ref{thm:general-rademacher}. Then the general training and holdout risks reduce to the online risks, giving \eqref{eq:online-slow-loss}. The state-based inequality follows from the state-based {inequality} in Theorem~\ref{thm:general-rademacher}.
\end{proof}

\subsubsection{Proof of Corollary~\ref{cor:online-wang}}
\label{app:proof-online-wang}

\begin{proof}
Set $Q_t \equiv 1$ in Theorem~\ref{thm:general-wang} and in Corollary~\ref{cor:general-wang-explicit}. The SCMI terms in those results are exactly $I(L_t^+;U_t \mid G_{t-1}^{\mathrm{on}})$ and $I(W_t;U_t \mid G_{t-1}^{\mathrm{on}})$. The general training and holdout risks reduce to the online risks, and $\sum_{t=1}^n I(W_t;U_t \mid G_{t-1}^{\mathrm{on}})$ is exactly the information term in \eqref{eq:online-wang-directed}. This gives \eqref{eq:online-wang-main}, \eqref{eq:online-wang-directed}, and \eqref{eq:online-wang-explicit}.
\end{proof}

\section{{Application: Stopping under the learner filtration}}
\label{sec:stopping-application}

{We bound the expected cumulative holdout loss and the difference between the expected cumulative holdout and training losses up to a bounded stopping time. The argument uses the fact that whether round \(t\) is included is known from the learner history before that round.}

Fix a deterministic horizon $N \in \mathbb N$, and let
\begin{align}
\tau \in \{0,1,\dots,N\}
\end{align}
{be a bounded stopping time for the learner filtration. Then}
\begin{align}
\{\tau \ge t\} \in \mathcal H_{t-1}
\qquad \text{for every } t \in [N].
\label{eq:tau-predictable}
\end{align}
Define the predictable indicator
\begin{align}
I_t \coloneqq  \mathbf 1\{\tau \ge t\}.
\label{eq:def-stopping-indicator}
\end{align}
Then $I_t$ is $\mathcal H_{t-1}$-measurable and can therefore be used as the predictable weight in Section~\ref{sec:general-framework}. For the stopping-time application, we use the {auxiliary random element}
\begin{align}
G_{t-1}^{\mathrm{st}}
&\coloneqq 
(H_{t-1},\widetilde Z_t,I_t),
\end{align}
for which
\begin{align}
\sigma(G_{t-1}^{\mathrm{st}})=\sigma(H_{t-1},\widetilde Z_t,I_t),
\qquad
{I_t\text{ is }\mathcal H_{t-1}\text{-measurable}.}
\end{align}
Define the {expected cumulative training and holdout risks} by
\begin{align}
L_{\tau}^{\mathrm{tr}}
&\coloneqq 
\sum_{t=1}^N
\mathbb E\!\left[
\mathbb E\!\left[
I_t\,\ell(W_t,Z_{t,U_t})
\middle| G_{t-1}^{\mathrm{st}} \right] \right],
\label{eq:def-stopping-tr}
\\
L_{\tau}^{\mathrm{ho}}
&\coloneqq 
\sum_{t=1}^N
\mathbb E\!\left[
\mathbb E\!\left[
I_t\,\ell(W_t,Z_{t,1-U_t})
\middle| G_{t-1}^{\mathrm{st}} \right] \right].
\label{eq:def-stopping-ho}
\end{align}
We also write
\begin{align}
\ell_t \coloneqq I_t\,\ell(W_t,Z_{t,0}).
\label{eq:def-stopping-Xt}
\end{align}

\begin{theorem}[{Stopping-time specialization}]
\label{thm:stopping-wang}
Assume the SCMI setting with {Eq.}~\eqref{eq:selected-feedback-kernel} and row-wise exchangeability~\eqref{eq:row-exchange-general}. Assume that $\ell : \mathcal W \times \mathcal Z \to [0,1]$. Let $C,\eta > 0$ satisfy \eqref{eq:wang-feasible}. Then
\begin{align}
\bigl|L_{\tau}^{\mathrm{ho}}-L_{\tau}^{\mathrm{tr}}\bigr|
&\le
2\sum_{t=1}^N \sqrt{2\,I(\ell_t;U_t \mid G_{t-1}^{\mathrm{st}})},
\label{eq:stopping-slow}
\\
L_{\tau}^{\mathrm{ho}}
&\le
(1+C)L_{\tau}^{\mathrm{tr}}
+
\frac{1}{\eta}\sum_{t=1}^N I(\ell_t;U_t \mid G_{t-1}^{\mathrm{st}}),
\label{eq:stopping-wang-main}
\\
L_{\tau}^{\mathrm{ho}}
&\le
(1+C)L_{\tau}^{\mathrm{tr}}
+
\frac{1}{\eta}\sum_{t=1}^N I(W_t;U_t \mid G_{t-1}^{\mathrm{st}}).
\label{eq:stopping-wang-directed}
\end{align}
Moreover,
\begin{align}
L_{\tau}^{\mathrm{ho}}
\le
L_{\tau}^{\mathrm{tr}}
+
8\sum_{t=1}^N I(\ell_t;U_t \mid G_{t-1}^{\mathrm{st}})
+
4\sqrt{2L_{\tau}^{\mathrm{tr}}\sum_{t=1}^N I(\ell_t;U_t \mid G_{t-1}^{\mathrm{st}})}.
\label{eq:stopping-wang-explicit}
\end{align}
If $L_{\tau}^{\mathrm{tr}}=0$, then
\begin{align}
L_{\tau}^{\mathrm{ho}}
\le
\frac{2}{\log 2}\sum_{t=1}^N I(\ell_t;U_t \mid G_{t-1}^{\mathrm{st}}).
\label{eq:stopping-wang-interp}
\end{align}
\end{theorem}

\begin{remark}
This {stopping-time specialization} is one reason to keep the learner filtration explicit in the general formulation. The indicator $I_t$ must be chosen before the round-$t$ selector is sampled, which is exactly the meaning of predictability with respect to $\mathcal H_{t-1}$. If the stopping rule were allowed to inspect {}{the observation \(Z_{t,1-U_t}\) withheld from the learner}, it would no longer be a valid learner-side specialization of the SCMI framework.
\end{remark}

\begin{remark}
Theorem~\ref{thm:stopping-wang} is stated in cumulative form rather than after division by $N$. {An average normalized by the random length requires a separate argument.}
\end{remark}

\subsection{Proof of Theorem~\ref{thm:stopping-wang}}
\label{app:proofs-stopping-application}
\label{app:proof-stopping-wang}

\begin{proof}
Set $Q_t \coloneqq  I_t$ in Section~\ref{sec:general-framework}, with horizon $n=N$. Then $Q_t$ is $\mathcal H_{t-1}$-measurable by \eqref{eq:tau-predictable}, and
\begin{align}
L_N^{\mathrm{tr}} = \frac{1}{N}L_{\tau}^{\mathrm{tr}},
\qquad
L_N^{\mathrm{ho}} = \frac{1}{N}L_{\tau}^{\mathrm{ho}}.
\end{align}
Moreover, the weighted first-copy loss from Section~\ref{sec:general-framework} is exactly
\begin{align}
L_t^+ = I_t\,\ell(W_t,Z_{t,0}) = \ell_t.
\end{align}
Applying Theorem~\ref{thm:general-rademacher} with this choice of $Q_t$ and multiplying by $N$ proves \eqref{eq:stopping-slow}. Because $I_t$ is a measurable function of {$H_{t-1}$, $G_{t-1}^{\mathrm{st}}$ and $G_{t-1}$ generate the same sigma-field. Thus} the SCMI term in that bound is exactly $I(\ell_t;U_t \mid G_{t-1}^{\mathrm{st}})$.

Applying Theorem~\ref{thm:general-wang} with the same choice of $Q_t$ and multiplying by $N$ proves \eqref{eq:stopping-wang-main}. Under this choice, $L_t^+=I_t\ell(W_t,Z_{t,0})=\ell_t$. Hence
\begin{align}
I(\ell_t;U_t \mid G_{t-1}^{\mathrm{st}})
\le
I(L_t^+;U_t \mid G_{t-1}^{\mathrm{st}})
\le
I(W_t;U_t \mid G_{t-1}^{\mathrm{st}}),
\end{align}
which proves \eqref{eq:stopping-wang-directed}.

Finally, applying {the bound in Eq.~\eqref{eq:wang-explicit-loss} involving \(J_L\)} of Corollary~\ref{cor:general-wang-explicit} with $n=N$ and multiplying by $N$ gives \eqref{eq:stopping-wang-explicit}. The interpolation bound \eqref{eq:stopping-wang-interp} follows in the same way from the first inequality in \eqref{eq:wang-interpolation}.
\end{proof}

\section{Active learning: detailed setup and proofs}
\label{app:active-learning}
\label{app:active-learning-details}

This appendix gives the details behind Theorem~\ref{thm:active-iw-slow}. {We first recall the query decisions from Section~\ref{sec:active-learning-application}}, then prove the population-risk and row-swap identities that feed into the slow-rate and fast-rate bounds. {}{}{The remark following the row-swap proof explains how conditioning differs between the two identities.}

\subsection{Details on the active-learning filtration and update}
\label{app:active-query-filtration}

{The query decision \(Q_t\) is made after the learner observes \(X_t\), so it need not be measurable with respect to the past learner history. Given \(G_{t-1}^{\rm act}\), both candidate decisions \(Q_{t,0}\) and \(Q_{t,1}\) are fixed, but the realized decision \(Q_t=Q_{t,U_t}\) need not be.} At round \(t\), the learner first observes \(X_t\), chooses
\[
    p_t=\pi_t(W_{t-1},X_t)\in[p_{\min},1],
\]
draws \(V_t\sim{\rm Unif}[0,1]\), and sets
\[
    Q_t=\mathbf1\{V_t\le p_t\}.
\]
With the intermediate fields
\[
\mathcal F_t^{X}\coloneqq \sigma(H_{t-1}^{\mathrm{act}},X_t),
\qquad
\mathcal F_t^{Q}\coloneqq \sigma(H_{t-1}^{\mathrm{act}},X_t,V_t),
\]
the probability \(p_t\) is \(\mathcal F_t^X\)-measurable and the decision \(Q_t\) is \(\mathcal F_t^Q\)-measurable. Thus the query decision is fixed before the label is revealed. The observed label variable is
\[
\bar Y_t=
\begin{cases}
Y_t, & Q_t=1,\\
\bot, & Q_t=0,
\end{cases}
\]
where \(\bot\) is only a bookkeeping symbol and is never passed to the loss function.

{The importance-weighted sample \(S_t\) from Section~\ref{sec:active-learning-application} satisfies}
\[
S_t
=
\begin{cases}
S_{t-1}\cup\{(X_t,Y_t,1/p_t)\}, & Q_t=1,\\
S_{t-1}, & Q_t=0.
\end{cases}
\]
The state update is written as
\[
W_t=\Psi_t(S_t,W_{t-1}).
\]
This notation covers online updates, warm-started optimization, and deterministic tie-breaking based on previous state. On an unqueried round, no pseudo-label is created and no supervised loss involving \((X_t,\bot)\) or the unobserved label \(Y_t\) is evaluated; the state may still evolve by deterministic computation based on \(S_{t-1}\) and \(W_{t-1}\).

For the supersample proof, each row contains {}{the two feature--label pairs and the uniform random variables that determine whether their labels would be requested}. {Here}{, given \(H_{t-1}^{\rm act}\),} \((Z_{t,0},V_{t,0})\) and \((Z_{t,1},V_{t,1})\) are independent copies of \((Z,V)\) with {\((Z,V)\sim P_{XY}\otimes{\rm Unif}[0,1]\)}. The selector chooses
\[
(X_t,Y_t,V_t)=(X_{t,U_t},Y_{t,U_t},V_{t,U_t}).
\]
{By the definitions in Section~\ref{sec:active-learning-application}, conditioning on \(G_{t-1}^{\mathrm{act}}\)} fixes \(p_{t,0},p_{t,1}\) and \(Q_{t,0},Q_{t,1}\), {}{while the selector \(U_t\) remains conditionally uniform}:
\[
\mathbb P(U_t=0\mid\mathcal G_{t-1}^{\mathrm{act}})
=
\mathbb P(U_t=1\mid\mathcal G_{t-1}^{\mathrm{act}})
=
\frac12
\qquad \text{a.s.}
\]
{Eq.~\eqref{eq:selected-feedback-kernel}} holds because the learner receives only
\[
(X_t,p_t,Q_t,\bar Y_t)
=
(X_{t,U_t},p_{t,U_t},Q_{t,U_t},\bar Y_{t,U_t}),
\]
where \(\bar Y_{t,u}=Y_{t,u}\) if \(Q_{t,u}=1\) and \(\bar Y_{t,u}=\bot\) otherwise. {}{The pair \((Z_{t,1-U_t},V_{t,1-U_t})\), which the learner does not observe,} is used only for {evaluation in the theoretical analysis}.

The importance-weighted contribution is interpreted as
\[
\frac{Q_t}{p_t}\ell(W_n,(X_t,Y_t))
=
\begin{cases}
p_t^{-1}\ell(W_n,(X_t,Y_t)), & Q_t=1,\\
0, & Q_t=0.
\end{cases}
\]
Hence the loss is never evaluated at \(\bot\). For a fixed predictor \(w\),
\[
\mathbb E\!\left[
\frac{Q_t}{p_t}\ell(w,(X_t,Y_t))
\middle|
\mathcal F_t^{X},Y_t
\right]
=
\ell(w,(X_t,Y_t)),
\]
because \(\mathbb E[Q_t\mid\mathcal F_t^X{,Y_t}]=p_t\).

{By the definition of \(L_{n,t}^{\pm}\) in Section~\ref{sec:active-learning-application} and since} {\(\ell\) takes values in \([0,1]\)} and \(p_{t,u}\ge p_{\min}\),
\[
0\le L_{n,t}^+,L_{n,t}^- \le \frac{1}{p_{\min}},
\]
so \(p_{\min}L_{n,t}^+\) and \(p_{\min}L_{n,t}^-\) lie in \([0,1]\).

\subsection{Population-risk and terminal row-swap identities}
\label{app:active-pop-row-swap}
\label{app:proof-active-iw-population}

The next two identities are the bridge from the terminal predictor to the roundwise SCMI terms. The first is the importance-weighting identity for {}{the observation not shown to the learner}. The second is the terminal row-swap identity used in both the slow-rate and fast-rate active-learning bounds.

{Throughout this subsection, we use the sequential construction of Section~\ref{sec:active-learning-application}: each new row (including the uniforms for the query decisions) and its selector are generated from the learner history and fresh randomness independent of past draws. The learner updates from its history, the feedback it observes and fresh randomness; any learner randomness reused later is included in its history.}

\begin{lemma}[Population-risk representation for the active holdout]
\label{lem:active-pop-representation}
Under the active-learning supersample construction,
\begin{align}
L_n^{\mathrm{IW,ho}}=\mathbb E[R(W_n)].
\label{eq:active-pop-representation}
\end{align}
\end{lemma}

\begin{proof}
Fix \(t\in[n]\).  Write
\[
Z_t'\coloneqq Z_{t,1-U_t},\qquad
V_t'\coloneqq V_{t,1-U_t},\qquad
p_t'\coloneqq p_{t,1-U_t},\qquad
Q_t'\coloneqq Q_{t,1-U_t}.
\]
Let
\begin{align}
\mathcal T_t
\coloneqq
\sigma(H_n^{\mathrm{act}},U_t)
\label{eq:def-active-selected-transcript-field}
\end{align}
be {}{the conditioning sigma-field} for the population-risk calculation.  {Both \(W_{t-1}\) and \(W_n\) are \(\mathcal T_t\)-measurable. We first determine the conditional law of \((Z_t',V_t')\) given \(\mathcal T_t\).}

{Given \(H_{t-1}^{\rm act}\), the two candidates are independent and the selector is fair given both. Thus even after conditioning on \(U_t\) and the chosen candidate, \((Z_t',V_t')\) has law \(P_{XY}\otimes{\rm Unif}[0,1]\). {Under this conditioning, the chosen feature, uniform draw and observed label are fixed; the state at round \(t\) uses only the chosen feedback and fresh learner randomness. Every later round starts from the resulting learner history and uses fresh draws independent of \((Z_t',V_t')\). Thus the conditional law of the final history does not depend on this unused pair, so conditioning also on \(H_n^{\rm act}\) leaves the pair's law unchanged.} Since it is the same for every past and choice, conditioning on \((H_n^{\rm act},U_t)\) alone gives
\[
 \Law\!\left((Z_t',V_t')\mid\mathcal T_t\right)=P_{XY}\otimes{\rm Unif}[0,1].
\]}

First condition on \(\mathcal T_t\vee\sigma(Z_t')\).  Under this conditioning, \(W_{t-1}\), \(W_n\), and \(Z_t'=(X_t',Y_t')\) are fixed, so
\[
p_t'=
\pi_t(W_{t-1},X_t')
\]
is fixed.  {}{The variable \(V_t'\)} is still an independent \(\mathrm{Unif}[0,1]\) variable. Hence
\begin{align}
\mathbb E\!\left[
\frac{Q_t'}{p_t'}
\middle|
\mathcal T_t\vee\sigma(Z_t')
\right]
&=
\mathbb E\!\left[
\frac{\mathbf 1\{V_t'\le p_t'\}}{p_t'}
\middle|
\mathcal T_t\vee\sigma(Z_t')
\right]
=1.
\label{eq:active-iw-pop-query-coin}
\end{align}
Therefore,
\begin{align}
\mathbb E\!\left[
\frac{Q_t'}{p_t'}\,\ell(W_n,Z_t')
\right]
=
\mathbb E\!\left[
\ell(W_n,Z_t')
\right].
\label{eq:active-iw-pop-remove-query}
\end{align}
Moreover, conditional on \(\mathcal T_t\), {}{the observation \(Z_t'=Z_{t,1-U_t}\)} is an independent draw from \(P_{XY}\), while \(W_n\) is \(\mathcal T_t\)-measurable. Thus
\begin{align}
\mathbb E\!\left[
\ell(W_n,Z_t')\mid\mathcal T_t
\right]
=
\int \ell(W_n,z)P_{XY}(dz)
=
R(W_n).
\label{eq:active-iw-pop-risk-after-selected-transcript}
\end{align}
Taking expectations in \eqref{eq:active-iw-pop-risk-after-selected-transcript} and using \eqref{eq:active-iw-pop-remove-query} gives
\[
\mathbb E\!\left[
\frac{Q_{t,1-U_t}}{p_{t,1-U_t}}\,
\ell(W_n,Z_{t,1-U_t})
\right]
=
\mathbb E[R(W_n)].
\]
{By the tower property, the unconditional expectation on the left equals the corresponding iterated expectation in \eqref{eq:active-iw-train}. Averaging over \(t=1,\ldots,n\) gives \eqref{eq:active-pop-representation}.}
\end{proof}

\paragraph*{Remark (choice of predictor in Lemma~\ref{lem:active-pop-representation}).}
For each \(t\), the same identity holds for any predictor \(\bar W_t\) measurable with respect to \(\sigma(H_n^{\rm act})\). Indeed, \(\mathcal T_t\) fixes \(\bar W_t\), and the holdout pair \((Z_t',V_t')\) still has conditional law \(P_{XY}\otimes\mathrm{Unif}[0,1]\). Averaging first over the query uniform and then over the observation gives
\[
\mathbb E\!\left[\frac{Q_t'}{p_t'}\ell(\bar W_t,Z_t')\,\middle|\,\mathcal T_t\right]
=\int\ell(\bar W_t,z)\,P_{XY}(dz)=R(\bar W_t).
\]
After taking expectations and averaging over \(t\), using the same \(W_m\) in every term, for a fixed \(m\le n\), gives \(\mathbb E[R(W_m)]\). Using \(W_t\) in term \(t\) instead gives \(n^{-1}\sum_{t=1}^n\mathbb E[R(W_t)]\).

Of the \(n\) rounds, fix the round \(t\) whose candidates will be exchanged. In Section~\ref{sec:active-learning-application}, \(L_{n,t}^{+}\) and \(L_{n,t}^{-}\) are the importance-weighted losses of \(W_n\) at coordinates \(0\) and \(1\), respectively. Indeed, the selected triple is \((X_t,Y_t,V_t)=(X_{t,U_t},Y_{t,U_t},V_{t,U_t})\), so the query rule gives \(p_t=p_{t,U_t}\) and \(Q_t=\mathbf1\{V_{t,U_t}\le p_{t,U_t}\}=Q_{t,U_t}\). Thus the training loss is \(L_{n,t}^{+}\) on \(U_t=0\) and \(L_{n,t}^{-}\) on \(U_t=1\); the holdout loss uses the other coordinate. To express both risks using only \(L_{n,t}^{+}\), it suffices to pair the two training contributions in expectation, and likewise the two holdout contributions. The point to check is that both losses use \(W_n\): exchanging the two candidates at round \(t\) must also account for what the learner does \emph{after} round \(t\).

Write \(A_u=(Z_{t,u},V_{t,u})\) for candidate \(u\) at this round, and put \(\nu=P_{XY}\otimes\mathrm{Unif}[0,1]\). Let \(\mathcal F_{t-1}^{\rm full}\) be the sigma-field generated by all draws before round \(t\), including past candidate pairs and selectors. At every round, the sequential construction specified above draws two candidates independently from \(\nu\) given the full past, then chooses either position with probability \(1/2\) given the past and this pair. At the fixed round this gives
\[
\operatorname{Law}(A_0,A_1,U_t\mid\mathcal F_{t-1}^{\rm full})
  =\nu\otimes\nu\otimes\operatorname{Bernoulli}(1/2).
\]
The history update uses the previous learner history and the feedback from the chosen candidate, with fresh randomness if the update is randomized. The displayed law is a property of this construction; specifying a row law conditional only on the learner history would not by itself specify its law given the full past.

\begin{gprelemma}[Effect of exchanging the candidates at round \(t\)]
\label{lem:g-round-swap}
Let \(n\ge1\) be the total number of rounds and fix \(t\in\{1,\ldots,n\}\). At round \(t\), exchange the candidates \(A_0,A_1\) and change the selector \(U_t\) to \(1-U_t\). The resulting process has the original law. If the two runs use the same fresh random inputs, their learner histories agree from the update at \(t\) through the final update at \(n\); in particular, they produce the same final predictor. With \(H=H_{t-1}^{\rm act}\), there is one probability kernel \(K_{n,t}\) for both choices of position such that
\[
\operatorname{Law}(W_n\mid H,A_0,A_1,U_t)
   =K_{n,t}(\,\cdot\mid H,A_{U_t})\qquad\text{a.s.}
\]
\end{gprelemma}
\begin{proof}
Let \(E\) be the fresh random inputs for the state update at \(t\) and for the remaining \(n-t\) rounds, and let \(\mu\) be their law. The conditional i.i.d. assumption in Section~\ref{sec:active-learning-application} (Eq.~\eqref{eq:conditional-product-row} for the augmented observations) gives the candidate pair law \(\nu\otimes\nu\) given \(H\). Under the fresh-randomness construction stated at the start of this subsection, the same product law holds given the full past. Thus, for any measurable candidate sets \(B_0,B_1\),
\[
\begin{aligned}
\Pr(A_0\in B_0,A_1\in B_1\mid\mathcal F_{t-1}^{\rm full})
 &=\nu(B_0)\nu(B_1)\\
 &=\Pr(A_1\in B_0,A_0\in B_1\mid\mathcal F_{t-1}^{\rm full}).
\end{aligned}
\]
The selector is conditionally fair by Eq.~\eqref{eq:selector-sym}; in this sequential construction it remains fair given the full past and the pair. Also, \(E\) is drawn independently of those variables and of \(U_t\). Consequently the joint conditional law of \((A_0,A_1,U_t,E)\) is \(\nu\otimes\nu\otimes\operatorname{Bernoulli}(1/2)\otimes\mu\). The displayed equality permits exchanging the first two entries, and conditional fairness permits replacing \(u\) by \(1-u\). Therefore
\[
T(a_0,a_1,u,e)=(a_1,a_0,1-u,e)
\]
preserves this joint law. Applying the original generation rules to the transformed inputs gives a run with the original law.

Use a superscript \(T\) for the exchanged run. We show that the two learner histories agree after the update at \(t\) and remain equal through \(n\), so that \(W_n^T=W_n\).

At round \(t\), \(A_u^T=A_{1-u}\) and \(U_t^T=1-U_t\), so \(A^T_{U_t^T}=A_{U_t}\): both runs use the same candidate. The previous state, selected feature and query uniform are the same, so the query decisions agree. Both runs observe the same label or both record \(\bot\). The earlier weighted sample is determined by the common history \(H\). Thus the sample, state and history updates in Section~\ref{sec:active-learning-application} give \(H_t^{\rm act,T}=H_t^{\rm act}\).

For the induction step, suppose the learner histories agree at the start of the next round. The generation rules use that history and fresh random inputs. With the same inputs in both runs they generate the same candidate pair and selector, and then the same observed feedback and updated state. Thus the histories still agree after the update. Applying this step for the remaining \(n-t\) rounds gives \(H_n^{\rm act,T}=H_n^{\rm act}\) and hence \(W_n^T=W_n\). If \(t=n\), the update at \(t\) already proves the claim.

This completes the induction: with the same fresh inputs, the exchange gives \(W_n^T=W_n\). To obtain the conditional law in the statement, we now use the transition rules from the base case and the induction step. Fix the past learner history \(h\) and the chosen candidate \(a\). In the base case the round-\(t\) update uses only \(h\) and the feedback from \(a\); neither the other candidate nor the position of \(a\) enters this update. In each of the remaining \(n-t\) rounds the transition uses only the learner history produced so far and fresh inputs. Composing these measurable transitions and taking the marginal law of the final predictor defines a probability kernel \(K_{n,t}(\cdot\mid h,a)\).

Given \((H,A_0,A_1,U_t=u)\), the initial inputs to this composition are \(H\) and \(A_u\), and the fresh inputs still have law \(\mu\). Thus, for every measurable set \(B\) of predictors,
\[
\begin{aligned}
\Pr(W_n\in B\mid H,A_0,A_1,U_t=0)&=K_{n,t}(B\mid H,A_0),\\
\Pr(W_n\in B\mid H,A_0,A_1,U_t=1)&=K_{n,t}(B\mid H,A_1).
\end{aligned}
\]
The same kernel appears in both lines because both use the same sequence of transitions.
\end{proof}

\begin{gprelemma}[Training and holdout contributions at round \(t\)]
\label{lem:g-paired-losses}
Let \(n\ge1\) be the total number of rounds and fix \(t\in\{1,\ldots,n\}\). When the final predictor \(W_n\) is evaluated on the two candidates at round \(t\), the two contributions to the training loss have equal expectations, as do the two contributions to the holdout loss:
\[
\begin{aligned}
\mathbb E[(1-U_t)L_{n,t}^{+}]&=\mathbb E[U_tL_{n,t}^{-}],\\
\mathbb E[U_tL_{n,t}^{+}]&=\mathbb E[(1-U_t)L_{n,t}^{-}].
\end{aligned}
\]
\end{gprelemma}
\begin{proof}
If candidate \(a\) is chosen and the loss is evaluated at \(b=((x,y),v)\), the query weight is determined by \(b\) and the last state \(w\) of the past history \(h\). Lemma~\ref{lem:g-round-swap} gives the predictor law for the choice \(a\). Thus the mean loss for these two choices is
\[
\Phi_{n,t}(h;a,b)=
\frac{\mathbf1\{v\le\pi_t(w,x)\}}{\pi_t(w,x)}
\int\ell(w',(x,y))\,K_{n,t}(dw'\mid h,a).
\]
First condition on \(G:=G_{t-1}^{\rm act}\), which fixes \((H,A_0,A_1)\) but not \(U_t\). In \(\Phi_{n,t}(H;a,b)\), the first candidate \(a\) is used for the update; the second, \(b\), supplies both the observation and its weight for evaluation. For example, \(U_tL_{n,t}^{+}\) is zero on \(U_t=0\). On \(U_t=1\), the update uses \(A_1\) and the evaluation uses \(A_0\). Conditional fairness and Lemma~\ref{lem:g-round-swap} give
\[
\mathbb E[U_tL_{n,t}^{+}\mid G]
=\Pr(U_t=1\mid G)\,\mathbb E[L_{n,t}^{+}\mid G,U_t=1]
=\tfrac12\Phi_{n,t}(H;A_1,A_0).
\]
There are two training contributions: the loss is evaluated on the candidate used for the update. Their conditional expectations are
\[
\begin{aligned}
\mathbb E[(1-U_t)L_{n,t}^{+}\mid G]&=\tfrac12\Phi_{n,t}(H;A_0,A_0),\\
\mathbb E[U_tL_{n,t}^{-}\mid G]&=\tfrac12\Phi_{n,t}(H;A_1,A_1).
\end{aligned}
\]
For the two holdout contributions, the loss is evaluated on the other candidate:
\[
\begin{aligned}
\mathbb E[U_tL_{n,t}^{+}\mid G]&=\tfrac12\Phi_{n,t}(H;A_1,A_0),\\
\mathbb E[(1-U_t)L_{n,t}^{-}\mid G]&=\tfrac12\Phi_{n,t}(H;A_0,A_1).
\end{aligned}
\]
Now condition only on \(H\). By the tower property we may average each right-hand side over \((A_0,A_1)\) given \(H\). For training, each candidate has conditional marginal law \(\nu\), so
\[
\begin{aligned}
\mathbb E[(1-U_t)L_{n,t}^{+}\mid H]
 &=\tfrac12\mathbb E[\Phi_{n,t}(H;A_0,A_0)\mid H]\\
 &=\tfrac12\int\Phi_{n,t}(H;a,a)\,\nu(da)\\
 &=\tfrac12\mathbb E[\Phi_{n,t}(H;A_1,A_1)\mid H]
 =\mathbb E[U_tL_{n,t}^{-}\mid H].
\end{aligned}
\]
For holdout, the conditionally i.i.d. assumption in Eq.~\eqref{eq:conditional-product-row} gives the joint law \(\nu\otimes\nu\) for \((A_0,A_1)\) given \(H\). Since both factors are \(\nu\), this law is invariant under \((a,b)\mapsto(b,a)\). This is conditional row exchangeability, Eq.~\eqref{eq:row-exchange-general}, and it justifies the marked equality below. The factor \(1/2\) comes from selector fairness, Eq.~\eqref{eq:selector-sym}.
\[
\begin{aligned}
\mathbb E[U_tL_{n,t}^{+}\mid H]
 &=\tfrac12\mathbb E[\Phi_{n,t}(H;A_1,A_0)\mid H]\\
 &=\tfrac12\iint\Phi_{n,t}(H;b,a)\,\nu(da)\nu(db)\\
 &\overset{\eqref{eq:row-exchange-general}}{=}\tfrac12\iint\Phi_{n,t}(H;a,b)\,\nu(da)\nu(db)\\
 &=\tfrac12\mathbb E[\Phi_{n,t}(H;A_0,A_1)\mid H]
 =\mathbb E[(1-U_t)L_{n,t}^{-}\mid H].
\end{aligned}
\]
Thus the equalities hold after averaging over the pair given \(H\); they need not hold for a fixed ordered pair given \(G\). Averaging the two conditional equalities over \(H\) gives
\[
\begin{aligned}
\mathbb E[(1-U_t)L_{n,t}^{+}]
 &=\mathbb E\!\left[\mathbb E[U_tL_{n,t}^{-}\mid H]\right]=\mathbb E[U_tL_{n,t}^{-}],\\
\mathbb E[U_tL_{n,t}^{+}]
 &=\mathbb E\!\left[\mathbb E[(1-U_t)L_{n,t}^{-}\mid H]\right]=\mathbb E[(1-U_t)L_{n,t}^{-}].
\end{aligned}
\]
\end{proof}

\begin{lemma}[Terminal active-learning row-swap identity]
\label{lem:active-row-swap-identity}
Assume the {active learning} supersample construction and {Eq.~\eqref{eq:selected-feedback-kernel}}. This is the step that connects the final predictor \(W_n\) to {}{the correlations between the loss and the sign \(\varepsilon_t=2U_t-1\) at each round} used in Section~\ref{sec:general-framework}. Define
\[
M_{n,t}^+\coloneqq p_{\min}L_{n,t}^+,
\qquad
M_{n,t}^-\coloneqq p_{\min}L_{n,t}^-.
\]
Then \(M_{n,t}^+,M_{n,t}^-\in[0,1]\).  For every \(C\ge0\), with
\(\widetilde\varepsilon_t^{(C)}\coloneqq
\varepsilon_t-C/(C+2)\),
\begin{align}
p_{\min}L_n^{\mathrm{IW,ho}}
-
(1+C)p_{\min}L_n^{\mathrm{IW}}
=
\frac{C+2}{n}
\sum_{t=1}^n
\mathbb E\!\left[
\mathbb E\!\left[
\widetilde\varepsilon_t^{(C)}M_{n,t}^+
\middle|G_{t-1}^{\mathrm{act}}
\right]
\right].
\label{eq:active-row-swap-shifted}
\end{align}
In particular, for \(C=0\),
\begin{align}
p_{\min}\!\left(
L_n^{\mathrm{IW,ho}}
-
L_n^{\mathrm{IW}}
\right)
=
\frac{2}{n}
\sum_{t=1}^n
\mathbb E\!\left[
\mathbb E\!\left[
\varepsilon_tM_{n,t}^+
\middle|G_{t-1}^{\mathrm{act}}
\right]
\right].
\label{eq:active-row-swap-slow}
\end{align}
\end{lemma}

{\begin{proof}[Proof of Lemma~\ref{lem:active-row-swap-identity}]
Since \(0\le\ell\le1\), \(Q_{t,u}\in\{0,1\}\) and \(p_{t,u}\ge p_{\min}>0\), both normalized losses \(M_{n,t}^{+},M_{n,t}^{-}\) lie in \([0,1]\). In the risks defined in~\eqref{eq:active-iw-train}, training uses coordinate \(U_t\) and holdout uses \(1-U_t\). Splitting each expectation according to \(U_t=0,1\) therefore gives
\[
\begin{aligned}
p_{\min}L_n^{\mathrm{IW}}
 &=\frac1n\sum_{t=1}^n\mathbb E[(1-U_t)M_{n,t}^{+}+U_tM_{n,t}^{-}],\\
p_{\min}L_n^{\mathrm{IW,ho}}
 &=\frac1n\sum_{t=1}^n\mathbb E[U_tM_{n,t}^{+}+(1-U_t)M_{n,t}^{-}].
\end{aligned}
\]
By the training identity in Lemma~\ref{lem:g-paired-losses}, the two summands inside the first expectation have the same mean; by its holdout identity, so do the two inside the second. Multiplication by \(p_{\min}\) preserves each equality. We obtain
\[
\begin{aligned}
p_{\min}L_n^{\mathrm{IW}}&=\frac2n\sum_{t=1}^n\mathbb E[(1-U_t)M_{n,t}^{+}],\\
p_{\min}L_n^{\mathrm{IW,ho}}&=\frac2n\sum_{t=1}^n\mathbb E[U_tM_{n,t}^{+}].
\end{aligned}
\]
Subtracting \(1+C\) times the first line from the second leaves, inside the sum with prefactor \(1/n\), the coefficient
\[
2\{U_t-(1+C)(1-U_t)\}=(C+2)(2U_t-1)-C=(C+2)\widetilde\varepsilon_t^{(C)}.
\]
Taking the inner expectation given \(G_{t-1}^{\rm act}\) by the tower property gives~\eqref{eq:active-row-swap-shifted}. At \(C=0\) the coefficient is \(2\varepsilon_t\), which gives~\eqref{eq:active-row-swap-slow}.
\end{proof}}

\paragraph*{Remark (conditioning in Lemmas~\ref{lem:active-pop-representation} and~\ref{lem:g-paired-losses}).}
Use the notation \(Z_t',V_t',Q_t',p_t'\) for the holdout coordinate as in Lemma~\ref{lem:active-pop-representation}, and recall \(\mathcal T_t=\sigma(H_n^{\rm act},U_t)\) and \(\mathcal G_{t-1}^{\rm act}=\sigma(G_{t-1}^{\rm act})\). The conditional mean of the holdout weight depends on what is fixed:
\[
\begin{aligned}
\mathbb E\!\left[\frac{Q_t'}{p_t'}\,\middle|\,\mathcal T_t\vee\sigma(Z_t')\right]&=1,\\
\mathbb E\!\left[\frac{Q_t'}{p_t'}\,\middle|\,\mathcal G_{t-1}^{\rm act}\right]
 &=\tfrac12\left(\frac{Q_{t,0}}{p_{t,0}}+\frac{Q_{t,1}}{p_{t,1}}\right).
\end{aligned}
\]
In the first line, \(W_n,Z_t',p_t'\) are fixed and \(V_t'\) is still uniform; this removes the weight from the loss in Lemma~\ref{lem:active-pop-representation}. In the second line both candidate weights are fixed, and the fair selector determines which is the holdout weight. Their average need not equal one. Lemma~\ref{lem:g-paired-losses} first conditions on \(G_{t-1}^{\rm act}\) to evaluate the two selector choices, then keeps only \(H_{t-1}^{\rm act}\) fixed to apply exchangeability when averaging over the candidate pair.

\subsection{Proof of Theorem~\ref{thm:active-iw-slow}}
\label{app:proofs-active-learning-application}
\label{app:proof-active-iw-slow}

\begin{proof}
By Lemma~\ref{lem:active-row-swap-identity} with \(C=0\),
\[
p_{\min}\left(
L_n^{\mathrm{IW,ho}}
-
L_n^{\mathrm{IW}}
\right)
=
\frac{2}{n}\sum_{t=1}^n
\mathbb E\!\left[
\mathbb E\!\left[
\varepsilon_t M_{n,t}^+
\middle| G_{t-1}^{\mathrm{act}}
\right]\right].
\]
For each $t$, the one-step bound used in Theorem~\ref{thm:general-rademacher} gives
\[
\left|
\mathbb E\!\left[
\mathbb E\!\left[
\varepsilon_t M_{n,t}^+
\middle| G_{t-1}^{\mathrm{act}}
\right]\right]
\right|
\le
\sqrt{
2\,I(M_{n,t}^+;U_t\mid G_{t-1}^{\mathrm{act}})
}.
\]
Since $M_{n,t}^+=p_{\min}L_{n,t}^+$ and $p_{\min}>0$, MI is unchanged by this deterministic rescaling:
\[
I(M_{n,t}^+;U_t\mid G_{t-1}^{\mathrm{act}})
=
I(L_{n,t}^+;U_t\mid G_{t-1}^{\mathrm{act}}).
\]
Combining these inequalities and using Lemma~\ref{lem:active-pop-representation}, which identifies $L_n^{\mathrm{IW,ho}}$ with $\mathbb E[R(W_n)]$, proves \eqref{eq:active-iw-slow-sum}.
\end{proof}

\subsection{Fast-rate refinements for active learning}
\label{app:active-fast-rate-refinements}

The main active-learning section states only the slow-rate consequence. For completeness, we also record the Shifted-Rademacher refinements obtained from the same normalized importance-weighted {loss values}.

\begin{theorem}[IWAL-type SCMI bound for the terminal predictor]
\label{thm:active-iw-wang}
Assume {the conditions of Theorem~\ref{thm:active-iw-slow} and} that \(\ell : \mathcal W \times \mathcal Z \to [0,1]\). For every \(C,\eta>0\) satisfying \eqref{eq:wang-feasible},
\begin{align}
\mathbb E\!\left[
R(W_n) \right]
\le
(1+C)L_n^{\mathrm{IW}}
+
\frac{1}{\eta p_{\min} n}\sum_{t=1}^n I(L_{n,t}^+;U_t \mid G_{t-1}^{\mathrm{act}}).
\label{eq:active-iw-wang}
\end{align}
Moreover, the following explicit specialization of \eqref{eq:active-iw-wang} holds:
\begin{align}
\mathbb E\!\left[
R(W_n) \right]
\le
L_n^{\mathrm{IW}}
+
\frac{8}{p_{\min}n}\sum_{t=1}^n I(L_{n,t}^+;U_t \mid G_{t-1}^{\mathrm{act}})
+
4\sqrt{\frac{2L_n^{\mathrm{IW}}}{p_{\min}n}\sum_{t=1}^n I(L_{n,t}^+;U_t \mid G_{t-1}^{\mathrm{act}})}.
\label{eq:active-iw-explicit}
\end{align}
If \(L_n^{\mathrm{IW}}=0\), then the explicit specialization can be sharpened to
\begin{align}
\mathbb E\!\left[
R(W_n) \right]
\le
\frac{2}{p_{\min}n\log 2}\sum_{t=1}^n I(L_{n,t}^+;U_t \mid G_{t-1}^{\mathrm{act}}).
\label{eq:active-iw-interpolation}
\end{align}
\end{theorem}

\begin{remark}
Theorem~\ref{thm:active-iw-wang} controls a genuine generalization gap. The left-hand side is the {expected population risk} of the terminal predictor \(W_n\), and the right-hand side contains the {expected importance-weighted training risk} from \eqref{eq:active-iw-train}. No comparator and no regret term appear. {Eq.~\eqref{eq:active-iw-explicit} follows by optimizing over feasible parameter choices in Eq.~\eqref{eq:active-iw-wang}. When \(L_n^{\mathrm{IW}}=0\), Eq.~\eqref{eq:active-iw-interpolation} follows by taking a limit through feasible choices.}
\end{remark}

\paragraph{Proof of Theorem~\ref{thm:active-iw-wang}.}
\label{app:proof-active-iw-wang}

\begin{proof}
Use the normalized {loss values} \(M_{n,t}^{\pm}\) from
Lemma~\ref{lem:active-row-swap-identity}, and write
\(\widetilde\varepsilon_t=\varepsilon_t-C/(C+2)\).  By
Lemma~\ref{lem:active-row-swap-identity},
\begin{align}
p_{\min}L_n^{\mathrm{IW,ho}}
-
(1+C)p_{\min}L_n^{\mathrm{IW}}
=
\frac{C+2}{n}\sum_{t=1}^n \mathbb E\!\left[\mathbb E\!\left[\widetilde\varepsilon_t M_{n,t}^+\middle| G_{t-1}^{\mathrm{act}}\right]\right].
\label{eq:proof-active-wang-shift}
\end{align}
By Lemma~\ref{lem:one-step-wang}, applied with
\begin{align}
G = G_{t-1}^{\mathrm{act}},
\qquad
{L} = M_{n,t}^+,
\qquad
U = U_t,
\end{align}
and using \(\mathbb P(U_t=0\mid\mathcal G_{t-1}^{\mathrm{act}})=\mathbb P(U_t=1\mid\mathcal G_{t-1}^{\mathrm{act}})=1/2\), we have
\begin{align}
(C+2)\,\mathbb E\!\left[
\mathbb E\!\left[
\widetilde\varepsilon_t M_{n,t}^+
\middle|G_{t-1}^{\mathrm{act}}
\right]
\right]
\le
\frac{1}{\eta}I(M_{n,t}^+;U_t \mid G_{t-1}^{\mathrm{act}})
=
\frac{1}{\eta}I(L_{n,t}^+;U_t \mid G_{t-1}^{\mathrm{act}}),
\label{eq:proof-active-wang-onestep}
\end{align}
where the equality uses the invariance of MI under multiplication by the deterministic constant \(p_{\min}>0\).

Summing \eqref{eq:proof-active-wang-onestep} over \(t\) and substituting into \eqref{eq:proof-active-wang-shift} yields
\begin{align}
L_n^{\mathrm{IW,ho}}
\le
(1+C)L_n^{\mathrm{IW}}
+
\frac{1}{\eta p_{\min}n}\sum_{t=1}^n I(L_{n,t}^+;U_t \mid G_{t-1}^{\mathrm{act}}).
\end{align}
Lemma~\ref{lem:active-pop-representation} identifies the left-hand side with \(\mathbb E\!\left[R(W_n)\right]\). This proves \eqref{eq:active-iw-wang}.

It remains to prove the explicit specialization. Fix \(0<C \le 1\) and choose \(\eta=C/8\). As in the proof of Corollary~\ref{cor:general-wang-explicit}, the feasibility condition \eqref{eq:wang-feasible} holds. Applying \eqref{eq:active-iw-wang} with this choice gives
\begin{align}
\mathbb E\!\left[R(W_n)\right]
\le
(1+C)L_n^{\mathrm{IW}}
+
\frac{8}{Cp_{\min}n}\sum_{t=1}^n I(L_{n,t}^+;U_t \mid G_{t-1}^{\mathrm{act}}).
\label{eq:proof-active-explicit-start}
\end{align}
Set
\begin{align}
L \coloneqq  L_n^{\mathrm{IW}},
\qquad
x \coloneqq  \frac{1}{p_{\min}n}\sum_{t=1}^n I(L_{n,t}^+;U_t \mid G_{t-1}^{\mathrm{act}}).
\end{align}
 {The right-hand side of \eqref{eq:proof-active-explicit-start} is \(L+CL+8x/C\). The optimization over \(0<C\le1\) in the proof of Corollary~\ref{cor:general-wang-explicit} therefore gives \eqref{eq:active-iw-explicit}. If \(L=0\), taking \(\eta\uparrow(\log 2)/2\) through feasible pairs in \eqref{eq:active-iw-wang}, as in that proof, gives \eqref{eq:active-iw-interpolation}.}
\end{proof}

\begin{proposition}[{Decomposition by the query indicator}]
\label{prop:active-iw-query-aware}
For each \(t \in [n]\), abbreviate
\[
G\coloneqq G_{t-1}^{\mathrm{act}},
\qquad
Q\coloneqq Q_{t,0},
\qquad
L\coloneqq L_{n,t}^+.
\]
For a realized value of the conditioning variable \(G\), write
\(I_G(\cdot;\cdot)\) for MI under the regular conditional law given \(G\). Since \(Q\) is {\(G\)-measurable},
\begin{align}
I(L;U_t\mid G)
=
I(L;U_t\mid G,Q).
\label{eq:active-iw-query-aware-a}
\end{align}
Moreover, \(L=0\) on \(\{Q=0\}\), while on \(\{Q=1\}\),
\[
L=\frac{1}{p_{t,0}}\ell(W_n,Z_{t,0}),
\]
and \(p_{t,0}\) is \(G\)-measurable and bounded away from zero. Hence
\begin{align}
I(L_{n,t}^+;U_t\mid G_{t-1}^{\mathrm{act}})
=
\mathbb E\!\left[
Q_{t,0}\,
I_{G_{t-1}^{\mathrm{act}}}
\!\left(\ell(W_n,Z_{t,0});U_t\right)
\right].
\label{eq:active-iw-query-aware-b}
\end{align}
In particular, by data processing,
\begin{align}
I(L_{n,t}^+;U_t\mid G_{t-1}^{\mathrm{act}})
\le
\mathbb E\!\left[
Q_{t,0}\,
I_{G_{t-1}^{\mathrm{act}}}
\!\left(W_n;U_t\right)
\right].
\label{eq:active-iw-query-aware-c}
\end{align}
\end{proposition}

\paragraph{Proof of Proposition~\ref{prop:active-iw-query-aware}.}
\label{app:proof-active-iw-query-aware}

\begin{proof}
{Fix \(t\) and use the notation of the proposition.}
By construction, \(G\) contains
\(X_{t,0},Y_{t,0},V_{t,0},X_{t,1},Y_{t,1},V_{t,1}\). Hence \(p_{t,0}\) and \(Q_{t,0}\) are
\(G\)-measurable. This immediately gives
\[
I(L;U_t\mid G)=I(L;U_t\mid G,Q),
\]
which proves \eqref{eq:active-iw-query-aware-a}.

Condition now on a realized value of \(G\). If \(Q=0\), then \(L=0\), so the
CMI between \(L\) and \(U_t\) is zero. If \(Q=1\),
then
\[
L=p_{t,0}^{-1}\ell(W_n,Z_{t,0}).
\]
The multiplier \(p_{t,0}^{-1}\) is strictly positive and \(G\)-measurable, so
deterministic rescaling does not change the CMI.
Therefore, under the conditional law given \(G\),
\[
I_G(L;U_t)
=
Q\,I_G(\ell(W_n,Z_{t,0});U_t).
\]
Averaging this identity over \(G\) proves \eqref{eq:active-iw-query-aware-b}.

Finally, \(Z_{t,0}\) is \(G\)-measurable, and \(\ell(W_n,Z_{t,0})\) is a
measurable function of \((W_n,G)\). The conditional data-processing inequality
therefore gives
\[
I_G(\ell(W_n,Z_{t,0});U_t)
\le
I_G(W_n;U_t).
\]
Multiplying by \(Q\) and averaging over \(G\) proves
\eqref{eq:active-iw-query-aware-c}.
\end{proof}

\section{Bandit proofs}
\label{app:proofs-bandit-application}

This appendix proves Theorem~\ref{thm:bandit-main}. The proof is presented first: after {specifying the conditioning variables}, we prove the smoothing, empirical-optimization, {auxiliary-arm identity}, second-moment, {Bernstein}, and {posterior KL} lemmas in the order in which they are used. {After the proof, we explain why Theorem~\ref{thm:general-bernstein} does not apply directly.} {Throughout, fix a terminal time \(t\).}

As in the main text, the behavior policy \(\pi_s\) is only required to be predictable and to satisfy the exploration lower bound \(\pi_s(a)\ge\epsilon_s\). A canonical implementation is to set \(\pi_s=\widetilde\rho_{s-1}^{\rm exp}\) for \(s\ge2\), with an initial policy satisfying \(\pi_1(a)\ge\epsilon_1\). The proof below uses only the lower bound, and therefore also covers other predictable exploration schedules.

{By Eq.~\eqref{eq:bandit-proof-field-recursive}, the conditioning variable is
\[
G_{s-1}^{\rm bd}
=(H_{s-1}^{\rm bd},\pi_s,A_{s,0},R_{s,0},A_{s,1},R_{s,1}),
\qquad s\ge1.
\]}
{As in Section~\ref{sec:bandit-application}, \(\mathcal G_{s-1}^{\rm bd}=\sigma(G_{s-1}^{\rm bd})\).} The random element \(G_{s-1}^{\rm bd}\) contains the current {feedback pair} but not the current selector, and
\[
\mathbb P(U_s=0\mid\mathcal G_{s-1}^{\rm bd})
=
\mathbb P(U_s=1\mid\mathcal G_{s-1}^{\rm bd})
=
\frac12
\qquad\text{a.s.}
\]
{We use the importance-weighted reward \(R^a_{s,u}\) and gap estimates \(G_{s,u}(a)\) and \(G_{s,u}(\rho)\) from Section~\ref{sec:bandit-application}.}
\subsection{Proof roadmap and decomposition}
\label{app:bandit-proof-roadmap}

 {We use the decomposition of the expected regret \(\mathbb E[\Delta(\widetilde\rho_t^{\rm exp})]\) in Eq.~\eqref{eq:bandit-main-decomposition}.}
{To bound the middle term, we use the auxiliary arm \(\bar A_t\) from Section~\ref{sec:bandit-application}.}

For \(s\le t\), define the {statistic \(G_{s,0}(\bar A_t)\) and the corresponding sum of CMI terms} by
\[
    \mathsf G^{\rm bd}_{s,t}:=G_{s,0}(\bar A_t),
    {}
\]
\[
    \mathcal I_t^{\rm bd,G}
    :=
    \sum_{s=1}^t
    I\!\left(\mathsf G^{\rm bd}_{s,t};U_s
        \mid G^{\rm bd}_{s-1}\right).
\]
The selector condition \(\mathbb P(U_s=0\mid G^{\rm bd}_{s-1})=\mathbb P(U_s=1\mid G^{\rm bd}_{s-1})=1/2\) implies that \(U_s\) is independent of \(G^{\rm bd}_{s-1}\), and hence \(I(U_s;G^{\rm bd}_{s-1})=0\). Therefore the same quantity can also be written as
\[
    \mathcal I_t^{\rm bd,G}
    =
    \sum_{s=1}^t
    I\!\left((G^{\rm bd}_{s-1},\mathsf G^{\rm bd}_{s,t});U_s\right).
\]
{We first bound the smoothing and empirical-optimization terms in {}{}{\eqref{eq:bandit-main-decomposition}}, then the middle term.}

\subsection{Smoothing and exponential weights}

\begin{lemma}[Smoothing cost]
\label{lem:bandit-smoothing-cost-revised}
For every \(\rho\in\mathcal P_K\) and every \(0\le\epsilon\le 1/K\), define
\[
    \widetilde\rho(a):=(1-K\epsilon)\rho(a)+\epsilon.
\]
Then
\[
    \Delta(\widetilde\rho)\le \Delta(\rho)+K\epsilon.
\]
\end{lemma}

\begin{proof}
Since rewards lie in \([0,1]\), we have \(0\le\Delta(a)\le1\) for every arm.
Therefore
\[
    \Delta(\widetilde\rho)
    =
    (1-K\epsilon)\Delta(\rho)
    +
    \epsilon\sum_{a=1}^K\Delta(a)
    \le
    \Delta(\rho)+K\epsilon.
\]
\end{proof}

\begin{lemma}[Exponential-weights empirical optimization]
\label{lem:bandit-expweights-empirical-revised}
{The exponential-weights distribution \(\rho_t^{\rm exp}\) from Section~\ref{sec:bandit-application} satisfies the following bound pathwise:}
\[
    \widehat\Delta_t(\rho_t^{\rm exp})
    \le
    \frac{\log K}{\gamma_t}.
\]
\end{lemma}

\begin{proof}
Let \(\pi_0\) be the uniform distribution on \([K]\).  The Gibbs variational
formula gives
\[
    \sum_{a=1}^K \rho_t^{\rm exp}(a)\widehat r_t(a)
    -
    \frac1{\gamma_t}\mathrm{KL}(\rho_t^{\rm exp}\Vert\pi_0)
    =
    \sup_{\rho\in\mathcal P_K}
    \left\{
        \sum_{a=1}^K\rho(a)\widehat r_t(a)
        -
        \frac1{\gamma_t}\mathrm{KL}(\rho\Vert\pi_0)
    \right\}.
\]
Comparing with the point mass \(\delta_{a^\star}\) yields
\[
    \sum_{a=1}^K \rho_t^{\rm exp}(a)\widehat r_t(a)
    \ge
    \widehat r_t(a^\star)-\frac{\log K}{\gamma_t}.
\]
Hence
\[
    \widehat\Delta_t(\rho_t^{\rm exp})
    =
    \widehat r_t(a^\star)
    -
    \sum_{a=1}^K\rho_t^{\rm exp}(a)\widehat r_t(a)
    \le
    \frac{\log K}{\gamma_t}.
\]
\end{proof}

\subsection{Preparation: auxiliary arm and row-swap identities}
\label{app:bandit-transfer-preliminaries}
\label{app:bandit-auxiliary-arm-kernel}

\paragraph{Auxiliary-arm kernel.}
The auxiliary arm construction should be read as {a randomization used only in the theoretical analysis}, not as an observation available to the learner.  The learner computes \(\rho_t^{\rm exp}\) {}{only from the history \(H_t^{\rm bd}\) recorded by the learner}.  The proof then adds an independent random seed and draws \(\bar A_t\) from {}{this distribution}.  Even when the proof later conditions on {}{both arm--reward pairs from a round, the pair the learner did not observe is not an input} to the sampling rule for \(\bar A_t\).

\paragraph{{Identities for the auxiliary arm.}}

\begin{lemma}[{Identities for the auxiliary arm}]
\label{lem:bandit-virtual-arm-identities}
For the {auxiliary arm} \(\bar A_t\) drawn from \(\rho_t^{\rm exp}\),
\begin{align}
    \mathbb E[\widehat\Delta_t(\rho_t^{\rm exp})]
    &=
    \frac1t\sum_{s=1}^t
    \mathbb E\bigl[G_{s,U_s}(\bar A_t)\bigr],
    \label{eq:bandit-selected-virtual-identity}
    \\
    \mathbb E[\Delta(\rho_t^{\rm exp})]
    &=
    \frac1t\sum_{s=1}^t
    \mathbb E\bigl[G_{s,1-U_s}(\bar A_t)\bigr].
    \label{eq:bandit-ghost-virtual-identity}
\end{align}
\end{lemma}

\begin{proof}
{As defined in Section~\ref{sec:bandit-application}, \(H_t^{\rm bd}=(H_{t-1}^{\rm bd},\pi_t,A_t,R_t)\).}  Conditional on \({H_t^{\rm bd}}\), the posterior \(\rho_t^{\rm exp}\) and {}{the gap estimates} \(G_{s,U_s}(a)\) are fixed, while \(\bar A_t\) is drawn from \(\rho_t^{\rm exp}\).  Therefore
\[
\begin{aligned}
    \mathbb E\!\left[
        \frac1t\sum_{s=1}^tG_{s,U_s}(\bar A_t)
        \middle| {H_t^{\rm bd}}
    \right]
    &=
    \sum_{a=1}^K\rho_t^{\rm exp}(a)
    \frac1t\sum_{s=1}^tG_{s,U_s}(a)  \\
    &=
    \widehat\Delta_t(\rho_t^{\rm exp}).
\end{aligned}
\]
Taking expectations proves \eqref{eq:bandit-selected-virtual-identity}.

For {}{\eqref{eq:bandit-ghost-virtual-identity}}, fix \(s\le t\).  Conditional on {}{\(H_t^{\rm bd}\)} and on \(\bar A_t\), {}{the arm--reward pair \((A_{s,1-U_s},R_{s,1-U_s})\) that the learner did not observe still has its original distribution: the arm has law \(\pi_s\) and, given that arm, its reward has law \(P_{A_{s,1-U_s}}\).} Hence, for every arm \(a\),
\[
    \mathbb E\bigl[G_{s,1-U_s}(a)
        \mid {H_t^{\rm bd}},\bar A_t\bigr]
    =
    \Delta(a).
\]
Substituting \(a=\bar A_t\) gives
\[
    \mathbb E\bigl[G_{s,1-U_s}(\bar A_t)
        \mid {H_t^{\rm bd}},\bar A_t\bigr]
    =
    \Delta(\bar A_t).
\]
Averaging over \(\bar A_t\sim\rho_t^{\rm exp}\) {}{conditionally on \(H_t^{\rm bd}\) yields}
\[
    \mathbb E\bigl[G_{s,1-U_s}(\bar A_t)
        \mid {H_t^{\rm bd}}\bigr]
    =
    \Delta(\rho_t^{\rm exp}).
\]
This identity holds for each \(s\), so averaging over \(s=1,\ldots,t\) and then taking expectation proves \eqref{eq:bandit-ghost-virtual-identity}.
\end{proof}

\paragraph{Sequential row-swap identity.}

\begin{lemma}[Sequential row-swap identity]
\label{lem:bandit-row-swap-revised}
For every \(t\ge1\),
\[
    \mathbb E[\Delta(\rho_t^{\rm exp})]-\mathbb E[\widehat\Delta_t(\rho_t^{\rm exp})]
    =
    \frac2t\sum_{s=1}^t
    \mathbb E\bigl[\varepsilon_s\mathsf G^{\rm bd}_{s,t}\bigr].
\]
\end{lemma}

\begin{proof}
By Lemma~\ref{lem:bandit-virtual-arm-identities},
\[
    \mathbb E[\Delta(\rho_t^{\rm exp})]-\mathbb E[\widehat\Delta_t(\rho_t^{\rm exp})]
    =
    \frac1t\sum_{s=1}^t
    \mathbb E\bigl[
        G_{s,1-U_s}(\bar A_t)-G_{s,U_s}(\bar A_t)
    \bigr].
\]
It remains to rewrite the {difference in the preceding display} by using {}{the gap estimate \(G_{s,0}(\bar A_t)\)}.  Fix
\(s\).  Conditional on {}{\(H_{s-1}^{\rm bd}\) and \(\pi_s\)}, the feedback pair
\((A_{s,0},R_{s,0}),(A_{s,1},R_{s,1})\) is exchangeable and \(U_s\) is an independent \(\mathrm{Bernoulli}(1/2)\)
selector.  If we swap the labels \(0\) and \(1\) in the current row and
simultaneously replace \(U_s\) by \(1-U_s\), {}{the observed pair}
\((A_{s,U_s},R_{s,U_s})\) is unchanged.  {The conditional law of {}{the subsequent observations}, and hence that of \(\bar A_t\), is unchanged because future rounds use {}{only the recorded history and fresh randomness}.}  This local row-swap symmetry
implies
\[
    \mathbb E\bigl[\varepsilon_sG_{s,1}(\bar A_t)\bigr]
    =
    -\mathbb E\bigl[\varepsilon_sG_{s,0}(\bar A_t)\bigr].
\]
Since
\[
    G_{s,1-U_s}(\bar A_t)-G_{s,U_s}(\bar A_t)
    =
    \varepsilon_s\{G_{s,0}(\bar A_t)-G_{s,1}(\bar A_t)\},
\]
we obtain
\[
    \mathbb E\bigl[
        G_{s,1-U_s}(\bar A_t)-G_{s,U_s}(\bar A_t)
    \bigr]
    =
    2\mathbb E\bigl[\varepsilon_sG_{s,0}(\bar A_t)\bigr]
    =
    2\mathbb E\bigl[\varepsilon_s\mathsf G^{\rm bd}_{s,t}\bigr].
\]
Substituting this identity into the previous display proves the lemma.
\end{proof}

\subsection{{Bernstein bound for the bandit gap}}
\label{app:bandit-scmi-transfer-proof}
\label{app:bandit-scmi-bernstein-transfer}

\subsubsection{{}{Conditional and unconditional second moments}}

The following elementary comparison is the device used later to {obtain an unconditional second-moment bound from a bound conditional on \(U=1\)}.  It keeps the
conditioning explicit because this is the nonstandard part of the bandit {Bernstein argument}.

\begin{lemma}[{}{Bounding the second moment using the event $U=1$}]
\label{lem:selector-comparison-square-revised}
Let \(G\) be a random element, and let \(U\in\{0,1\}\) satisfy
\[
    \mathbb P(U=0\mid G)=\mathbb P(U=1\mid G)=\frac12
    \qquad\text{a.s.}
\]
Let \(X\) be any real-valued random variable on the same probability space,
possibly dependent on \(U\) and \(G\), with \(|X|\le b\).  Then
\[
    \mathbb E[X^2]
    \le
    \frac32\mathbb E[X^2\mid U=1]
    +
    20b^2 I(X;U\mid G).
\]
\end{lemma}

\begin{proof}
If \(b=0\), then \(X=0\) almost surely and the claim is immediate.  Assume
\(b>0\), and normalize
\[
    Y:=\frac{X^2}{b^2}\in[0,1].
\]
Since \(Y\) is a measurable function of \(X\), data processing gives
\[
    I(Y;U\mid G)
    \le
    I(X;U\mid G).
\]
Thus it is enough to prove the normalized inequality
\[
    \mathbb E[Y]
    \le
    \frac32\mathbb E[Y\mid U=1]
    +
    20 I(Y;U\mid G).
\]

We prove this bound {conditionally on \(G\)}.  Fix a regular
conditional law given a realization \(G=g\).  By the displayed assumption, the conditional law of \(U\) given \(G=g\) is uniform.  Write
\[
    p_u(g):=\mathbb E[Y\mid G=g,U=u],
    \qquad
    m(g):=\mathbb E[Y\mid G=g]
    =\frac{p_0(g)+p_1(g)}2,
\]
and let
\[
    i(g):=I(Y;U\mid G=g)
\]
be the ordinary mutual information under this conditional law.  The key
pointwise inequality is
\[
    m(g)
    \le
    \frac32 p_1(g)+20 i(g).
\]

There are two cases.  If \(p_0(g)\le 2p_1(g)\), {}{then}
\[
    m(g)=\frac{p_0(g)+p_1(g)}2
    \le
    \frac32 p_1(g).
\]
Thus the pointwise inequality holds without using the information term.

It remains to consider the case \(p_0(g)>2p_1(g)\).  In this case,
{the conditional mean of \(Y\) {}{given \(U=0\)} exceeds twice its conditional mean {}{given \(U=1\)}.}  We turn this imbalance into {}{a lower bound on \(i(g)=I(Y;U\mid G=g)\)}.
Under the conditional law given \(G=g\), introduce a Bernoulli
variable \(D\) by
\[
    D\mid(Y,U,G=g)\sim\mathrm{Bernoulli}(Y),
\]
independently of everything else given \((Y,U,G=g)\).  Then
\[
    \mathbb P(D=1\mid U=u,G=g)=p_u(g),
    \qquad
    \mathbb P(D=1\mid G=g)=m(g).
\]
Since \(D\) is generated from \(Y\) without using any additional
information about \(U\), data processing gives
\[
    i(g)
    =I(Y;U\mid G=g)
    \ge
    I(D;U\mid G=g).
\]
On the event \(D=1\), Bayes' rule gives
\[
    q(g)
    :=
    \mathbb P(U=0\mid D=1,G=g)
    =
    \frac{p_0(g)}{p_0(g)+p_1(g)}
    \ge
    \frac23.
\]
Thus observing \(D=1\) {}{makes \(U=0\) more likely than \(U=1\)}.  Since
\(U\) is uniform given \(G=g\), the mutual information between
\(D\) and \(U\) decomposes as
\[
\begin{aligned}
    I(D;U\mid G=g)
    &=
    \sum_{d\in\{0,1\}}
    \mathbb P(D=d\mid G=g)
    \KL\!\left(
        P_{U\mid D=d,G=g}
        \middle\Vert
        P_{U\mid G=g}
    \right)  \\
    &\ge
    \mathbb P(D=1\mid G=g)
    \KL\!\left(
        \mathrm{Bernoulli}(q(g))
        \middle\Vert
        \mathrm{Bernoulli}\!\left(\frac12\right)
    \right)  \\
    &\ge
    m(g)
    \KL\!\left(
        \mathrm{Bernoulli}\!\left(\frac23\right)
        \middle\Vert
        \mathrm{Bernoulli}\!\left(\frac12\right)
    \right)
    >
    \frac{m(g)}{20}.
\end{aligned}
\]
In the second inequality we used that
\(q\mapsto\KL(\mathrm{Bernoulli}(q)\Vert\mathrm{Bernoulli}(1/2))\) is
increasing for \(q\ge1/2\).  The last inequality follows from
\[
    \KL\!\left(
        \mathrm{Bernoulli}\!\left(\frac23\right)
        \middle\Vert
        \mathrm{Bernoulli}\!\left(\frac12\right)
    \right)
    =
    \frac23\log\frac43+
    \frac13\log\frac23
    >
    \frac1{20}.
\]
Hence \(m(g)\le 20i(g)\), and the pointwise inequality follows in
this case as well.

Finally, integrate the pointwise inequality with respect to the law of
\(G\).  We have
\[
    \mathbb E[m(G)]=\mathbb E[Y],
    \qquad
    \mathbb E[i(G)]=I(Y;U\mid G).
\]
Moreover, the displayed assumption gives \(\mathbb P(U=1)=1/2\) and
\[
    \mathbb E[Y\mid U=1]
    =
    \frac{\mathbb E[p_1(G)\,\mathbb P(U=1\mid G)]}{\mathbb P(U=1)}
    =
    \mathbb E[p_1(G)].
\]
Therefore
\[
    \mathbb E[Y]
    \le
    \frac32\mathbb E[Y\mid U=1]
    +20I(Y;U\mid G).
\]
Multiplying by \(b^2\) and using
\(I(Y;U\mid G)\le I(X;U\mid G)\) completes the proof.
\end{proof}

\subsubsection{{Conditional and unconditional second-moment bounds}}

The next two lemmas provide the {second-moment bounds for the Bernstein argument}.

\begin{lemma}[{Second moment conditional on \(U_s=1\)}]
\label{lem:bandit-ghost-second-moment-revised}
For every \(s\le t\),
\[
    \mathbb E\bigl[(\mathsf G^{\rm bd}_{s,t})^2\mid U_s=1\bigr]
    \le
    \frac{2}{\epsilon_t\Delta_{\min}}
    \mathbb E\bigl[\Delta(\rho_t^{\rm exp})\mid U_s=1\bigr].
\]
\end{lemma}

\begin{proof}
On the event \(U_s=1\), {}{the learner observes \((A_{s,1},R_{s,1})\), so \((A_{s,0},R_{s,0})\) is not recorded in \(H_t^{\rm bd}\).}  The terminal posterior \(\rho_t^{\rm exp}\) is {}{a function of \(H_t^{\rm bd}\)}, and \(\bar A_t\) is drawn from this posterior using fresh randomness.  Hence, {}{after conditioning on \(H_t^{\rm bd}\)}, on \(\rho_t^{\rm exp}\), and on the behavior policy \(\pi_s\), {}{the pair \((A_{s,0},R_{s,0})\) has the law obtained by drawing an arm from \(\pi_s\) and a reward from that arm}, independent of \(\bar A_t\).

It remains to do a fixed-distribution calculation.  Fix \(\rho\in\mathcal P_K\), draw \(\bar A\sim\rho\), and draw an {}{independent arm--reward pair} \((A,R)\) under a policy \(\pi\) satisfying \(\pi(a)\ge\epsilon_t\) for all arms.  Define
\[
    R^a:=\frac{\mathbf 1\{A=a\}R}{\pi(a)},
    \qquad
    G(a):=R^{a^\star}-R^a.
\]
If \(a=a^\star\), then \(G(a)=0\).  If \(a\ne a^\star\), then
\[
    G(a)
    =
    \frac{\mathbf 1\{A=a^\star\}R}{\pi(a^\star)}
    -
    \frac{\mathbf 1\{A=a\}R}{\pi(a)}.
\]
The two indicators cannot be one at the same time.  Since \(0\le R\le1\),
\[
\begin{aligned}
    \mathbb E[G(a)^2]
    &\le
    \mathbb E\!\left[\frac{\mathbf 1\{A=a^\star\}R^2}{\pi(a^\star)^2}\right]
    +
    \mathbb E\!\left[\frac{\mathbf 1\{A=a\}R^2}{\pi(a)^2}\right]  \\
    &\le
    \frac1{\pi(a^\star)}+
    \frac1{\pi(a)}
    \le
    \frac2{\epsilon_t}.
\end{aligned}
\]
Averaging over \(\bar A\sim\rho\) gives
\[
\begin{aligned}
    \mathbb E[G(\bar A)^2\mid \rho]
    &\le
    \frac2{\epsilon_t}
    \sum_{a\ne a^\star}\rho(a)   \\
    &\le
    \frac2{\epsilon_t\Delta_{\min}}
    \sum_{a\ne a^\star}\rho(a)\Delta(a)
    =
    \frac2{\epsilon_t\Delta_{\min}}\Delta(\rho).
\end{aligned}
\]
Applying this conditional calculation with \(\rho=\rho_t^{\rm exp}\) on the event \(U_s=1\) proves the claim.
\end{proof}

The preceding lemma {}{bounds the second moment given \(U_s=1\)}. The next lemma {provides the input to the Bernstein bound}: it bounds the unconditional second moment needed by the exponential-moment step. {Lemma~\ref{lem:selector-comparison-square-revised} controls any excess over {}{the term involving the second moment given \(U_s=1\)} using} \(I(\mathsf G^{\rm bd}_{s,t};U_s\mid G^{\rm bd}_{s-1})\).

\begin{lemma}[{}{Average second moment of the gap estimates}]
\label{lem:bandit-selector-square-revised}
For every \(t\ge1\),
\[
    \frac1t\sum_{s=1}^t
    \mathbb E\bigl[(\mathsf G^{\rm bd}_{s,t})^2\bigr]
    \le
    \frac{6}{\epsilon_t\Delta_{\min}}\mathbb E[\Delta(\rho_t^{\rm exp})]
    +
    \frac{20}{t\epsilon_t^2}\mathcal I_t^{\rm bd,G}.
\]
\end{lemma}

\begin{proof}
Fix \(s\le t\).  Because \((\epsilon_s)\) is nonincreasing, \(\pi_s(a)\ge\epsilon_t\) for every arm.  Since {}{each arm--reward pair contains only one arm},
\[
    |G_{s,0}(a)|\le \epsilon_t^{-1},
    \qquad a\in[K].
\]
Thus \(|\mathsf G^{\rm bd}_{s,t}|=|G_{s,0}(\bar A_t)|\le\epsilon_t^{-1}\).

Apply Lemma~\ref{lem:selector-comparison-square-revised} with
\[
    U=U_s,
    \qquad
    G=G^{\rm bd}_{s-1},
    \qquad
    X=\mathsf G^{\rm bd}_{s,t},
    \qquad
    b=\epsilon_t^{-1}.
\]
By construction,
\(\mathbb P(U_s=0\mid G_{s-1}^{\rm bd})=\mathbb P(U_s=1\mid G_{s-1}^{\rm bd})=1/2\), so Lemma~\ref{lem:selector-comparison-square-revised} gives
\[
\begin{aligned}
    \mathbb E\bigl[(\mathsf G^{\rm bd}_{s,t})^2\bigr]
    &\le
    \frac32
    \mathbb E\bigl[(\mathsf G^{\rm bd}_{s,t})^2\mid U_s=1\bigr]  \\
    &\quad+
    20\epsilon_t^{-2}
    I\!\left(\mathsf G^{\rm bd}_{s,t};U_s
        \mid G^{\rm bd}_{s-1}\right).
\end{aligned}
\]
The first term {is a second moment conditional on \(U_s=1\)}, so Lemma~\ref{lem:bandit-ghost-second-moment-revised} yields
\[
    \mathbb E\bigl[(\mathsf G^{\rm bd}_{s,t})^2\mid U_s=1\bigr]
    \le
    \frac{2}{\epsilon_t\Delta_{\min}}
    \mathbb E\bigl[\Delta(\rho_t^{\rm exp})\mid U_s=1\bigr].
\]
Finally, \(\Delta(\rho_t^{\rm exp})\ge0\) and \(\mathbb P(U_s=1)=1/2\), so
\[
    \mathbb E\bigl[\Delta(\rho_t^{\rm exp})\mid U_s=1\bigr]
    \le
    2\mathbb E[\Delta(\rho_t^{\rm exp})].
\]
Combining the last three displays gives
\[
    \mathbb E\bigl[(\mathsf G^{\rm bd}_{s,t})^2\bigr]
    \le
    \frac{6}{\epsilon_t\Delta_{\min}}\mathbb E[\Delta(\rho_t^{\rm exp})]
    +
    20\epsilon_t^{-2}
    I\!\left(\mathsf G^{\rm bd}_{s,t};U_s
        \mid G^{\rm bd}_{s-1}\right).
\]
Averaging this inequality over \(s=1,\ldots,t\) proves the claim.
\end{proof}

\subsubsection{{Bernstein inequality}}

The {Bernstein} lemma now combines the preceding ingredients. Lemma~\ref{lem:bandit-row-swap-revised} rewrites the {difference between the expected and empirical gaps} {}{in terms of the expectations \(\mathbb E[\varepsilon_s\mathsf G^{\rm bd}_{s,t}]\)}. The conditional Donsker--Varadhan step below bounds {}{each such expectation} by one {CMI term} plus a second-moment term. Lemma~\ref{lem:bandit-selector-square-revised} supplies the second-moment bound and leaves a copy of the target regret that can be absorbed into the left-hand side.

\begin{lemma}[{Bernstein bound for the bandit gap}]
\label{lem:bandit-scmi-bernstein-transfer-revised}
For every \(t\ge1\),
\[
    \mathbb E[\Delta(\rho_t^{\rm exp})]
    \le
    2\mathbb E[\widehat\Delta_t(\rho_t^{\rm exp})]
    +
    \frac{52}{t\epsilon_t\Delta_{\min}}
    \mathcal I_t^{\rm bd,G}.
\]
\end{lemma}

\begin{proof}
By Lemma~\ref{lem:bandit-row-swap-revised}, {the difference between the expected and empirical gaps satisfies}
\[
    \mathbb E[\Delta(\rho_t^{\rm exp})]-\mathbb E[\widehat\Delta_t(\rho_t^{\rm exp})]
    =
    \frac2t\sum_{s=1}^t
    \mathbb E\bigl[\varepsilon_s\mathsf G^{\rm bd}_{s,t}\bigr].
\]
We bound {}{each expectation \(\mathbb E[\varepsilon_s\mathsf G^{\rm bd}_{s,t}]\)}.  Conditionally on \(G^{\rm bd}_{s-1}\), the selector \(U_s\) {has the \(\mathrm{Bernoulli}(1/2)\) distribution}.  The conditional Donsker--Varadhan inequality, applied to \((\mathsf G^{\rm bd}_{s,t},U_s)\) given \(G^{\rm bd}_{s-1}\), gives, for every \(\lambda>0\),
\[
    \mathbb E\bigl[\varepsilon_s\mathsf G^{\rm bd}_{s,t}\bigr]
    \le
    \frac1\lambda
    I\!\left(\mathsf G^{\rm bd}_{s,t};U_s
        \mid G^{\rm bd}_{s-1}\right)
    +
    \frac\lambda2
    \mathbb E\bigl[(\mathsf G^{\rm bd}_{s,t})^2\bigr].
\]
where we used the following inequality; under the reference law in which \(U_s'\) is an independent uniform copy of \(U_s\), for every \(x\in\mathbb R\),
\[
    \mathbb E_{U_s'}\exp\left\{
        \lambda(2U_s'-1)x-\frac{\lambda^2x^2}{2}
    \right\}
    =
    e^{-\lambda^2x^2/2}\cosh(\lambda x)
    \le 1.
\]
Summing over \(s\le t\) gives
\[
    \mathbb E[\Delta(\rho_t^{\rm exp})]-\mathbb E[\widehat\Delta_t(\rho_t^{\rm exp})]
    \le
    \frac{2}{\lambda t}\mathcal I_t^{\rm bd,G}
    +
    \frac\lambda t\sum_{s=1}^t
    \mathbb E\bigl[(\mathsf G^{\rm bd}_{s,t})^2\bigr].
\]
We now insert Lemma~\ref{lem:bandit-selector-square-revised}{:}
\[
\begin{aligned}
    \mathbb E[\Delta(\rho_t^{\rm exp})]-\mathbb E[\widehat\Delta_t(\rho_t^{\rm exp})]
    &\le
    \frac{6\lambda}{\epsilon_t\Delta_{\min}}\mathbb E[\Delta(\rho_t^{\rm exp})]  \\
    &\quad+
    \left(
        \frac{2}{\lambda t}
        +
        \frac{20\lambda}{t\epsilon_t^2}
    \right)
    \mathcal I_t^{\rm bd,G}.
\end{aligned}
\]
Choose \(\lambda:=\epsilon_t\Delta_{\min}/12\).  Then the coefficient of \(\mathbb E[\Delta(\rho_t^{\rm exp})]\) is \(1/2\).  Also, since rewards lie in \([0,1]\), \(\Delta_{\min}\le1\), and therefore
\[
    \frac2\lambda+\frac{20\lambda}{\epsilon_t^2}
    =
    \frac{24}{\epsilon_t\Delta_{\min}}
    +
    \frac{5\Delta_{\min}}{3\epsilon_t}
    \le
    \frac{77}{3\epsilon_t\Delta_{\min}}.
\]
Thus
\[
    \mathbb E[\Delta(\rho_t^{\rm exp})]-\mathbb E[\widehat\Delta_t(\rho_t^{\rm exp})]
    \le
    \frac12\mathbb E[\Delta(\rho_t^{\rm exp})]
    +
    \frac{77}{3t\epsilon_t\Delta_{\min}}\mathcal I_t^{\rm bd,G}.
\]
Moving the half of \(\mathbb E[\Delta(\rho_t^{\rm exp})]\) to the left and multiplying by two gives
\[
    \mathbb E[\Delta(\rho_t^{\rm exp})]
    \le
    2\mathbb E[\widehat\Delta_t(\rho_t^{\rm exp})]
    +
    \frac{154}{3t\epsilon_t\Delta_{\min}}
    \mathcal I_t^{\rm bd,G}.
\]
Since \(154/3\le52\), the stated bound follows.
\end{proof}

\subsection{{Posterior KL bound} and proof of Theorem~\ref{thm:bandit-main}}

\begin{lemma}[{Bounding SCMI by posterior KL divergence}]
\label{lem:bandit-entropy-reduction-revised}
For every prior \(\pi_0\in\mathcal P_K\),
\[
    \mathcal I_t^{\rm bd,G}
    \le
    \mathbb E\bigl[\mathrm{KL}(\rho_t^{\rm exp}\Vert\pi_0)\bigr].
\]
In particular, for the uniform prior,
\[
    \mathcal I_t^{\rm bd,G}\le \log K.
\]
\end{lemma}

\begin{proof}
Since
\[
    \mathsf G^{\rm bd}_{s,t}=G_{s,0}(\bar A_t)
\]
is a measurable function of \(\bar A_t\) and the {conditioning random element}
\(G^{\rm bd}_{s-1}\), data processing gives
\[
    I\!\left(\mathsf G^{\rm bd}_{s,t};U_s
        \mid G^{\rm bd}_{s-1}\right)
    \le
    I\!\left(\bar A_t;U_s
        \mid G^{\rm bd}_{s-1}\right).
\]
Therefore
\[
    \mathcal I_t^{\rm bd,G}
    \le
    \sum_{s=1}^t
    I\!\left(\bar A_t;U_s
        \mid G^{\rm bd}_{s-1}\right).
\]
For notational clarity, write
\[
    V_s:=(\pi_s,A_{s,0},R_{s,0},A_{s,1},R_{s,1}),
    \qquad
    {G_{s-1}^{\rm bd}=(H_{s-1}^{\rm bd},V_s).}
\]
{{}{The observed history \(H_s^{\rm bd}\) is determined by}
\((H_{s-1}^{\rm bd},V_s,U_s)\). Given \(H_s^{\rm bd}\), future play uses
{}{only this history} and fresh randomness, and \(\bar A_t\) is drawn
from the resulting posterior with an independent seed. Thus
\(\bar A_t\) is conditionally independent of \((V_s,U_s)\) given
\(H_s^{\rm bd}\). The chain rule gives
\[
\begin{aligned}
 I(\bar A_t;V_s,U_s\mid H_{s-1}^{\rm bd})
 &=\underbrace{I(\bar A_t;V_s\mid H_{s-1}^{\rm bd})}_{\ge 0}
   + I(\bar A_t;U_s\mid G_{s-1}^{\rm bd}),
\end{aligned}
\]
because \(G_{s-1}^{\rm bd}=(H_{s-1}^{\rm bd},V_s)\). Now adjoin \(H_s^{\rm bd}\) and apply the chain rule in two orders:
\[
\begin{aligned}
 &I(\bar A_t;V_s,U_s,H_s^{\rm bd}\mid H_{s-1}^{\rm bd})\\
 &\quad= I(\bar A_t;V_s,U_s\mid H_{s-1}^{\rm bd})
       + \underbrace{I(\bar A_t;H_s^{\rm bd}\mid H_{s-1}^{\rm bd},V_s,U_s)}_{=0}\\
 &\quad= I(\bar A_t;H_s^{\rm bd}\mid H_{s-1}^{\rm bd})
       + \underbrace{I(\bar A_t;V_s,U_s\mid H_{s-1}^{\rm bd},H_s^{\rm bd})}_{=0}.
\end{aligned}
\]
The first zero follows because \((H_{s-1}^{\rm bd},V_s,U_s)\) determines
\(H_s^{\rm bd}\). For the second, conditioning on \((H_{s-1}^{\rm bd},H_s^{\rm bd})\)
is equivalent to conditioning on \(H_s^{\rm bd}\) alone, and
\(\bar A_t\) is conditionally independent of \((V_s,U_s)\) given \(H_s^{\rm bd}\).
Combining the two displays gives
\[
I(\bar A_t;U_s\mid G_{s-1}^{\rm bd}) \le I(\bar A_t;V_s,U_s\mid H_{s-1}^{\rm bd}) = I(\bar A_t;H_s^{\rm bd}\mid H_{s-1}^{\rm bd}).
\]
{}{Since \(H_{s-1}^{\rm bd}\) is part of \(H_s^{\rm bd}\),} the chain rule also gives
\[
I(\bar A_t;H_s^{\rm bd}) = I(\bar A_t;H_{s-1}^{\rm bd},H_s^{\rm bd}) = I(\bar A_t;H_{s-1}^{\rm bd}) + I(\bar A_t;H_s^{\rm bd}\mid H_{s-1}^{\rm bd}).
\]
Thus summing over \(s\) yields
\[
\begin{aligned}
 \sum_{s=1}^t I(\bar A_t;U_s\mid G_{s-1}^{\rm bd})
 &\le \sum_{s=1}^t I(\bar A_t;H_s^{\rm bd}\mid H_{s-1}^{\rm bd})\\
 &= \sum_{s=1}^t\bigl[I(\bar A_t;H_s^{\rm bd})-I(\bar A_t;H_{s-1}^{\rm bd})\bigr]\\
 &= I(\bar A_t;H_t^{\rm bd})-I(\bar A_t;H_0^{\rm bd})\\
 &\le I(\bar A_t;H_t^{\rm bd}).
\end{aligned}
\]}
By the auxiliary-arm construction, the conditional law of \(\bar A_t\) {given \(H_t^{\rm bd}\) is {}{the posterior}} \(\rho_t^{\rm exp}\). Let \(q\) denote the marginal law of \(\bar A_t\). Then
\[
{    I(\bar A_t;H_t^{\rm bd})
    =
}    \mathbb E\bigl[\mathrm{KL}(\rho_t^{\rm exp}\Vert q)\bigr].
\]
For any prior \(\pi_0\), the prior comparison identity gives
\[
    \mathbb E\bigl[\mathrm{KL}(\rho_t^{\rm exp}\Vert\pi_0)\bigr]
    =
    \mathbb E\bigl[\mathrm{KL}(\rho_t^{\rm exp}\Vert q)\bigr]
    +
    \mathrm{KL}(q\Vert\pi_0)
{    \ge
    I(\bar A_t;H_t^{\rm bd}).}
\]
Combining the previous displays proves
\[
    \mathcal I_t^{\rm bd,G}
    \le
    \mathbb E\bigl[\mathrm{KL}(\rho_t^{\rm exp}\Vert\pi_0)\bigr].
\]
If \(\pi_0\) is uniform on \([K]\), then
\(\mathrm{KL}(\rho_t^{\rm exp}\Vert\pi_0)\le\log K\) almost surely, and the
second claim follows.
\end{proof}

\paragraph{Proof of Theorem~\ref{thm:bandit-main}.}
\begin{proof}
We follow the decomposition in {}{}{\eqref{eq:bandit-main-decomposition}}.  Lemma~\ref{lem:bandit-smoothing-cost-revised} bounds the smoothing term by \(K\epsilon_{t+1}\).  Lemma~\ref{lem:bandit-scmi-bernstein-transfer-revised} gives
\[
    \mathbb E[\Delta(\rho_t^{\rm exp})]
    \le
    2\mathbb E[\widehat\Delta_t(\rho_t^{\rm exp})]
    +
    \frac{52}{t\epsilon_t\Delta_{\min}}
    \mathcal I_t^{\rm bd,G},
\]
which controls the {middle term in {}{}{\eqref{eq:bandit-main-decomposition}}}.  Finally, Lemma~\ref{lem:bandit-expweights-empirical-revised} gives
\[
    \mathbb E[\widehat\Delta_t(\rho_t^{\rm exp})]
    \le
    \frac{\log K}{\gamma_t}.
\]
Combining these three bounds gives
\[
    \mathbb E\bigl[\Delta(\widetilde\rho_t^{\rm exp})\bigr]
    \le
    K\epsilon_{t+1}
    +
    \frac{2\log K}{\gamma_t}
    +
    \frac{52}{t\epsilon_t\Delta_{\min}}
    \mathcal I_t^{\rm bd,G}.
\]
This proves the first claim.

By Lemma~\ref{lem:bandit-entropy-reduction-revised}, for every prior
\(\pi_0\in\mathcal P_K\),
\[
    \mathcal I_t^{\rm bd,G}
    \le
    \mathbb E\bigl[\mathrm{KL}(\rho_t^{\rm exp}\Vert\pi_0)\bigr].
\]
Taking \(\pi_0\) to be the uniform distribution gives
\[
    \mathcal I_t^{\rm bd,G}\le\log K.
\]
Therefore
\[
    \mathbb E\bigl[\Delta(\widetilde\rho_t^{\rm exp})\bigr]
    \le
    K\epsilon_{t+1}
    +
    \frac{2\log K}{\gamma_t}
    +
    \frac{52\log K}{t\epsilon_t\Delta_{\min}}.
\]

It remains to verify the cumulative rate.  Choose
\[
    \gamma_t=\frac{2\log K}{K\epsilon_t}.
\]
Then
\[
    \frac{2\log K}{\gamma_t}=K\epsilon_t.
\]
Since \((\epsilon_t)\) is nonincreasing,
\[
    K\epsilon_{t+1}\le K\epsilon_t.
\]
Thus
\[
    \mathbb E\bigl[\Delta(\widetilde\rho_t^{\rm exp})\bigr]
    \le
    2K\epsilon_t
    +
    \frac{52\log K}{t\epsilon_t\Delta_{\min}}.
\]
Set
\[
    \epsilon_t
    =
    \min\left\{
        \frac1K,
        \sqrt{\frac{52\log K}{Kt\Delta_{\min}}}
    \right\}.
\]
In the nontrivial range where the second term in the minimum is active,
\[
    K\epsilon_t
    =
    \sqrt{\frac{52K\log K}{t\Delta_{\min}}},
    \qquad
    \frac{52\log K}{t\epsilon_t\Delta_{\min}}
    =
    \sqrt{\frac{52K\log K}{t\Delta_{\min}}}.
\]
Hence the one-step regret is bounded by a universal numerical constant times
\[
    \sqrt{\frac{K\log K}{t\Delta_{\min}}}.
\]
In the initial range where \(\epsilon_t=1/K\), we use the trivial bound
\(\Delta(\widetilde\rho_t^{\rm exp})\le1\).  The length of this initial range is
at most of order
\[
    \frac{K\log K}{\Delta_{\min}}.
\]
Therefore its total contribution up to horizon \(T\) is at most a universal
constant times
\[
    \min\left\{T,\frac{K\log K}{\Delta_{\min}}\right\}
    \le
    \sqrt{\frac{KT\log K}{\Delta_{\min}}}.
\]
For the remaining range,
\[
    \sum_{t=1}^T
    \sqrt{\frac{K\log K}{t\Delta_{\min}}}
    \le
    C\sqrt{\frac{KT\log K}{\Delta_{\min}}}
\]
for a universal constant \(C>0\).  Combining the two ranges proves
\[
    \sum_{t=1}^T
    \mathbb E\bigl[\Delta(\widetilde\rho_t^{\rm exp})\bigr]
    \le
    C\sqrt{\frac{KT\log K}{\Delta_{\min}}}.
\]
\end{proof}

\begin{remark}[Gap-dependent tuning]
The one-step bound above holds for any nonincreasing exploration schedule
satisfying \(\pi_s(a)\ge\epsilon_s\). The displayed cumulative rate uses
\(\Delta_{\min}\) only to optimize this schedule, and should be read as the
standard gap-dependent, oracle-tuned consequence of the one-step inequality.
Removing this gap knowledge by an adaptive schedule is orthogonal to the
{SCMI} argument and is not pursued here.
\end{remark}

\subsection{Why the direct supersample Bernstein step fails}
\label{app:bandit-direct-bernstein-discussion}

The proof above is deliberately not written as a direct application of Theorem~\ref{thm:general-bernstein}. The issue is not the Bernstein information inequality itself. Rather, the second-moment condition required by Theorem~\ref{thm:general-bernstein} is not verified for the terminal bandit statistic by the naive fixed-distribution calculation. There is one obstruction, but it has two useful manifestations.

{First, fix \(\rho\in\mathcal P_K\). Conditionally on \(\mathcal F_{s-1}^{\rm bd}\), draw \(\bar A\sim\rho\) independently of \((A_{s,u},R_{s,u})\). For \(s\le t\), the calculation in the proof of Lemma~\ref{lem:bandit-ghost-second-moment-revised} gives
\[
\mathbb E\!\left[G_{s,u}(\bar A)^2\mid \mathcal F^{\rm bd}_{s-1}\right]
\le
\frac{2}{\epsilon_t\Delta_{\min}}\Delta(\rho),
\]}
where \(\mathcal F^{\rm bd}_{s-1}=\mathcal H^{\rm bd}_{s-1}\vee\sigma(\pi_s)\) is {}{the information recorded before round \(s\), together with the policy \(\pi_s\)}. This is the variance input behind the bandit rate.

The statistic used in the {supersample proof} is not fixed in this sense. It is
\[
\mathsf G^{\rm bd}_{s,t}=G_{s,0}(\bar A_t),
\qquad
\bar A_t\mid \mathcal H_t^{\rm bd}\sim \rho_t^{\rm exp},
\]
and \(\rho_t^{\rm exp}\) is computed from {}{the history \(H_t^{\rm bd}\)}. A direct invocation of Theorem~\ref{thm:general-bernstein} would require a condition of the form
\[
\frac1t\sum_{s=1}^t
\mathbb E\!\left[
\mathbb E\!\left[
G_{s,0}(\bar A_t)^2
\middle|
G^{\rm bd}_{s-1}
\right]
\right]
\le
B\,\mathbb E[\Delta(\rho_t^{\rm exp})].
\]
This does not follow from the fixed-\(\rho\) calculation, because \(\bar A_t\) depends on {}{the observed arms and rewards in \(H_t^{\rm bd}\)}.

Second, {}{the selector \(U_s\) determines which estimate uses the observed pair}:
\[
\begin{array}{c|cc}
U_s & G_{s,0}(\bar A_t) & G_{s,1}(\bar A_t)\\
\hline
0 & \text{{}{uses observed pair}} & \text{{}{uses other pair}}\\
1 & \text{{}{uses other pair}} & \text{{}{uses observed pair}}
\end{array}
\]
Thus \(G_{s,0}(\bar A_t)\) {{}{uses the pair the learner did not observe only when}} \(U_s=1\). {When} \(U_s=0\), it is {}{evaluated on the observed pair} and may be coupled to \(\bar A_t\) through the terminal posterior. The {{}{second-moment calculation for the pair not observed by the learner}} {  }{}{therefore applies to \(G_{s,0}(\bar A_t)\) given \(U_s=1\), but does not directly bound its second moment given \(U_s=0\).}

Lemma~\ref{lem:selector-comparison-square-revised} repairs this precise gap. Applied to \(X=\mathsf G^{\rm bd}_{s,t}\), it gives
\[
\mathbb E[X^2]
\le
\frac32\mathbb E[X^2\mid U_s=1]
+
20\epsilon_t^{-2}
I\!\left(\mathsf G^{\rm bd}_{s,t};U_s\mid G^{\rm bd}_{s-1}\right).
\]
The conditional term is the {second moment conditional on \(U_s=1\)} controlled by Lemma~\ref{lem:bandit-ghost-second-moment-revised}; {the CMI term controls any additional contribution from \(U_s=0\)}. This is the additional correction needed for the supersample bandit proof.

\section[Local and recursive conditioning in SCMI bounds]{{}{Local and recursive conditioning in SCMI bounds}}
\label{app:bandit-proof-filtration-choice}

{At each round, the analysis can either condition on the learner history and the current pair, as in Sections~\ref{sec:general-framework}--\ref{sec:bandit-application}, or also retain the pairs and selectors from earlier rounds. {We examine when the larger conditioning variable gives a valid SCMI bound and why the bandit proof needs only the local one.} To describe the second choice, write \(G_{t-1}\) for the conditioning variable used by the application at round \(t\). One can start at \(\overline G_0=G_0\) and, when the next pair is generated, update recursively:
\[
    \overline G_t=(\overline G_{t-1},U_t,G_t).
\]
Here \(G_t\) contains the pair for round \(t+1\). Thus, just before \(U_t\) is drawn, \(\overline G_{t-1}\) contains both members of every pair through round \(t\) and the selectors through round \(t-1\). The fields \(\sigma(\overline G_{t-1})\) are nested and form a filtration for the proof. The local fields \(\sigma(G_{t-1})\) need not be nested: an observation not supplied to the learner at one round need not be retained at the next.

In Section~\ref{sec:general-framework}, the local variable is \(G_{t-1}=(H_{t-1},\widetilde Z_t)\). In active learning, each member of the current pair includes the feature, label, and random variable for the query decision. In the bandit setting, \(G_{s-1}^{\rm bd}=(H_{s-1}^{\rm bd},V_s)\), where \(V_s\) from the proof of Lemma~\ref{lem:bandit-entropy-reduction-revised} contains \(\pi_s\) and the two candidate arm--reward pairs. Since \((H_{s-1}^{\rm bd},V_s,U_s)\) determines \(H_s^{\rm bd}\), appending just \(U_s\) and \(V_{s+1}\) at the next round generates the same enlarged field:
\[
    \sigma(\overline G_s^{\rm bd})=\sigma(\overline G_{s-1}^{\rm bd},U_s,V_{s+1}).
\]

For the SCMI bound with \(\overline G_{t-1}\), the selector must be fair given this larger field: \(\Pr(U_t=u\mid\overline G_{t-1})=1/2\) for \(u=0,1\). {Drawing \(U_t\sim\operatorname{Bernoulli}(1/2)\) independently of the preceding rounds and the current pair satisfies this condition}; fairness given only \(G_{t-1}\) in Eq.~\eqref{eq:selector-sym} does not imply it. Under the stronger condition, \(I(U_t;\overline G_{t-1}\mid G_{t-1})=0\). Expanding \(I(U_t;X,\overline G_{t-1}\mid G_{t-1})\) in both orders gives, for the same statistic \(X\),
\[
    I(X;U_t\mid\overline G_{t-1})
    =I(X;U_t\mid G_{t-1})+I(U_t;\overline G_{t-1}\mid X,G_{t-1}).
\]
The last term is nonnegative, so retaining the earlier pairs and selectors cannot reduce the roundwise CMI.

The bandit bound in Section~\ref{sec:bandit-application} uses the local \(G_{s-1}^{\rm bd}\). Given \(H_s^{\rm bd}\), subsequent observations and the draw of \(\bar A_t\) do not depend on the pair not observed at round \(s\) or on \(U_s\). Lemma~\ref{lem:bandit-entropy-reduction-revised} therefore bounds \(I(\bar A_t;U_s\mid G_{s-1}^{\rm bd})\) by \(I(\bar A_t;H_s^{\rm bd}\mid H_{s-1}^{\rm bd})\). These upper bounds telescope along the learner histories. The proof needs no record of earlier unobserved arm--reward pairs or selectors.}

\section{Ordinary-MI baseline for the bandit posterior}
\label{app:bandit-ordinary-mi-bernstein}

We also record a more direct bound based on ordinary mutual information.  This
baseline is useful for comparison: in the finite-arm setting it gives the same
horizon order, but it charges the dependence of the whole terminal arm on the
{}{observed history} rather than the roundwise {SCMI} increments.  The
argument treats the terminal arm \(\bar A_t\) as a randomized output of the
{}{observed bandit history} and controls its adaptivity bias by the ordinary MI
between \(\bar A_t\) and that history.  Since the behavior policy \(\pi_s\) is
predictable, after conditioning on the {}{history before the round and on \(\pi_s\)} it is just a
fixed sampling distribution.  Thus the proof is the ordinary-MI analogue of the
Bernstein argument in Section~\ref{sec:general-framework}: the only
bandit-specific step is to verify the fixed-arm Bernstein condition.

{We use the learner history \(H_t^{\rm bd}\) and the auxiliary arm \(\bar A_t\) from Section~\ref{sec:bandit-application}. Define}
\[
    \mathcal J_t^{\rm bd}
    :=
    I(\bar A_t;H_t^{\rm bd}).
\]
The empirical gap is written using the {}{observed actions and rewards}:
\[
    \widehat\Delta_t(\rho)
    =
    \frac1t\sum_{s=1}^t G_s(\rho),
    \qquad
    G_s(\rho):=\sum_{a=1}^K\rho(a)G_s(a),
\]
{where \(G_s(a):=R_s^{a^\star}-R_s^a\) uses the importance-weighted rewards \(R_s^a\) from Section~\ref{sec:bandit-application}.}

\begin{lemma}[Fixed-arm bandit Bernstein condition]
\label{lem:ordinary-mi-fixed-arm-bandit-bernstein}
For every arm \(a\in[K]\) and every \(s\le t\), define
\[
    M_s(a):=\Delta(a)-G_s(a).
\]
Then, conditionally on the {}{conditioning field}
\(\mathcal H_{s-1}^{\rm bd}\vee\sigma(\pi_s)\),
\[
    \mathbb E[M_s(a)\mid\mathcal H_{s-1}^{\rm bd},\pi_s]=0,
    \qquad
    |M_s(a)|\le 2\epsilon_t^{-1},
\]
and
\[
    \mathbb E[M_s(a)^2\mid\mathcal H_{s-1}^{\rm bd},\pi_s]
    \le
    \frac{2}{\epsilon_t\Delta_{\min}}\Delta(a).
\]
Consequently, the fixed-arm process satisfies the Bernstein condition with
\[
    b=2\epsilon_t^{-1},
    \qquad
    B=\frac{2}{\epsilon_t\Delta_{\min}}.
\]
\end{lemma}

\begin{proof}
Conditioning on \(\mathcal H_{s-1}^{\rm bd}\vee\sigma(\pi_s)\) fixes the
behavior policy \(\pi_s\).  Hence {}{the observed arm--reward pair} at round \(s\) is an
ordinary bandit sample from this fixed policy.  The importance-weighting identity
gives
\[
    \mathbb E[R_s^a\mid\mathcal H_{s-1}^{\rm bd},\pi_s]=r(a),
\]
and therefore
\(\mathbb E[G_s(a)\mid\mathcal H_{s-1}^{\rm bd},\pi_s]=\Delta(a)\).  This proves
the martingale-difference property of \(M_s(a)\).

For the range, \(0\le\Delta(a)\le1\) and
\(|G_s(a)|\le\epsilon_t^{-1}\) for \(s\le t\), because
\(\pi_s(a)\ge\epsilon_s\ge\epsilon_t\).  Thus
\[
    |M_s(a)|\le 1+\epsilon_t^{-1}\le 2\epsilon_t^{-1}.
\]
For the conditional second moment, the case \(a=a^\star\) is trivial.  If
\(a\ne a^\star\), then
\[
    G_s(a)
    =
    \frac{\mathbf1\{A_s=a^\star\}R_s}{\pi_s(a^\star)}
    -
    \frac{\mathbf1\{A_s=a\}R_s}{\pi_s(a)}.
\]
Only one of the two terms can be nonzero.  Since rewards lie in \([0,1]\) and
\(\pi_s(\cdot)\ge\epsilon_t\),
\[
    \mathbb E[G_s(a)^2\mid\mathcal H_{s-1}^{\rm bd},\pi_s]
    \le
    \frac1{\pi_s(a^\star)}+\frac1{\pi_s(a)}
    \le
    \frac2{\epsilon_t}.
\]
Because \(M_s(a)=\Delta(a)-G_s(a)\) is centered conditionally on the pre-round
field,
\[
    \mathbb E[M_s(a)^2\mid\mathcal H_{s-1}^{\rm bd},\pi_s]
    \le
    \mathbb E[G_s(a)^2\mid\mathcal H_{s-1}^{\rm bd},\pi_s]
    \le
    \frac{2}{\epsilon_t\Delta_{\min}}\Delta(a),
\]
where the last step uses \(\Delta(a)\ge\Delta_{\min}\) for \(a\ne a^\star\).
This is the desired Bernstein condition.
\end{proof}

\begin{lemma}[{Ordinary-MI Bernstein bound}]
\label{lem:ordinary-mi-bandit-bernstein-transfer}
For every \(t\ge1\),
\[
    \mathbb E[\Delta(\rho_t^{\rm exp})]
    \le
    \frac{10}{7}\mathbb E[\widehat\Delta_t(\rho_t^{\rm exp})]
    +
    \frac{40}{7t\epsilon_t\Delta_{\min}}\mathcal J_t^{\rm bd}.
\]
Consequently,
\[
    \mathbb E[\Delta(\rho_t^{\rm exp})]
    \le
    \frac{10}{7}\mathbb E[\widehat\Delta_t(\rho_t^{\rm exp})]
    +
    \frac{6}{t\epsilon_t\Delta_{\min}}\mathcal J_t^{\rm bd}.
\]
\end{lemma}

\begin{proof}
Since \(\bar A_t\mid H_t^{\rm bd}\sim\rho_t^{\rm exp}\),
\[
    \mathbb E[\Delta(\rho_t^{\rm exp})]
    -
    \mathbb E[\widehat\Delta_t(\rho_t^{\rm exp})]
    =
    \frac1t
    \mathbb E\!\left[
        \sum_{s=1}^t M_s(\bar A_t)
    \right].
\]
{For each fixed arm \(a\), Lemma~\ref{lem:ordinary-mi-fixed-arm-bandit-bernstein} verifies the conditions for applying the Bernstein exponential-moment estimate from the proof of Theorem~\ref{thm:general-bernstein} to \(M_s(a)\). We iterate this estimate along the bandit filtration for \(\sum_{s=1}^t M_s(a)\), then average over the marginal law of \(\bar A_t\). Applying the Donsker--Varadhan inequality to the joint law of \((\bar A_t,H_t^{\rm bd})\) relative to \(\Law(\bar A_t)\otimes\Law(H_t^{\rm bd})\) gives the ordinary-MI penalty \(\mathcal J_t^{\rm bd}=I(\bar A_t;H_t^{\rm bd})\). Thus, for every \(0<\lambda<3\epsilon_t/2\),}
\[
    \mathbb E[\Delta(\rho_t^{\rm exp})]
    -
    \mathbb E[\widehat\Delta_t(\rho_t^{\rm exp})]
    \le
    \frac{1}{\lambda t}\mathcal J_t^{\rm bd}
    +
    \frac{\lambda}{\epsilon_t\Delta_{\min}(1-2\lambda/(3\epsilon_t))}
    \mathbb E[\Delta(\rho_t^{\rm exp})].
\]
This is the same Bernstein calculation as in Section~\ref{sec:general-framework},
with {}{the mutual information \(I(\bar A_t;H_t^{\rm bd})\)} in place of the {SCMI} sum.

Choose \(\lambda=\epsilon_t\Delta_{\min}/4\).  Since \(\Delta_{\min}\le1\),
\[
    1-\frac{2\lambda}{3\epsilon_t}
    =
    1-\frac{\Delta_{\min}}6
    \ge
    \frac56,
\]
and hence
\[
    \frac{\lambda}{\epsilon_t\Delta_{\min}(1-2\lambda/(3\epsilon_t))}
    \le
    \frac{3}{10}.
\]
Therefore
\[
    \mathbb E[\Delta(\rho_t^{\rm exp})]
    -
    \mathbb E[\widehat\Delta_t(\rho_t^{\rm exp})]
    \le
    \frac{4}{t\epsilon_t\Delta_{\min}}\mathcal J_t^{\rm bd}
    +
    \frac{3}{10}\mathbb E[\Delta(\rho_t^{\rm exp})].
\]
Rearranging yields
\[
    \mathbb E[\Delta(\rho_t^{\rm exp})]
    \le
    \frac{10}{7}\mathbb E[\widehat\Delta_t(\rho_t^{\rm exp})]
    +
    \frac{40}{7t\epsilon_t\Delta_{\min}}\mathcal J_t^{\rm bd}.
\]
The looser display follows from \(40/7<6\).
\end{proof}

\begin{theorem}[Bandit bound with ordinary MI]
\label{thm:bandit-ordinary-mi-bernstein}
Under the assumptions of Theorem~\ref{thm:bandit-main}, for every \(t\ge1\),
\[
    \mathbb E[\Delta(\widetilde\rho_t^{\rm exp})]
    \le
    K\epsilon_{t+1}
    +
    \frac{2\log K}{\gamma_t}
    +
    \frac{6}{t\epsilon_t\Delta_{\min}}\mathcal J_t^{\rm bd}.
\]
Moreover, for every prior \(\pi_0\in\mathcal P_K\),
\[
    \mathcal J_t^{\rm bd}
    \le
    \mathbb E[\KL(\rho_t^{\rm exp}\Vert\pi_0)].
\]
In particular, for the uniform prior,
\[
    \mathbb E[\Delta(\widetilde\rho_t^{\rm exp})]
    \le
    K\epsilon_{t+1}
    +
    \frac{2\log K}{\gamma_t}
    +
    \frac{6\log K}{t\epsilon_t\Delta_{\min}}.
\]
Consequently, with
\[
    \epsilon_t
    =
    \min\left\{
        \frac1K,
        \sqrt{\frac{6\log K}{Kt\Delta_{\min}}}
    \right\},
    \qquad
    \gamma_t=\frac{2\log K}{K\epsilon_t},
\]
there exists a universal constant \(C>0\) such that
\[
    \sum_{t=1}^T
    \mathbb E[\Delta(\widetilde\rho_t^{\rm exp})]
    \le
    C\sqrt{\frac{KT\log K}{\Delta_{\min}}}.
\]
\end{theorem}

\begin{proof}
Lemma~\ref{lem:bandit-smoothing-cost-revised} gives
\[
    \mathbb E[\Delta(\widetilde\rho_t^{\rm exp})]
    \le
    K\epsilon_{t+1}+\mathbb E[\Delta(\rho_t^{\rm exp})].
\]
By Lemma~\ref{lem:ordinary-mi-bandit-bernstein-transfer},
\[
    \mathbb E[\Delta(\rho_t^{\rm exp})]
    \le
    \frac{10}{7}\mathbb E[\widehat\Delta_t(\rho_t^{\rm exp})]
    +
    \frac{6}{t\epsilon_t\Delta_{\min}}\mathcal J_t^{\rm bd}.
\]
Lemma~\ref{lem:bandit-expweights-empirical-revised} gives
\(\mathbb E[\widehat\Delta_t(\rho_t^{\rm exp})]\le\log K/\gamma_t\).  Therefore
\[
    \mathbb E[\Delta(\widetilde\rho_t^{\rm exp})]
    \le
    K\epsilon_{t+1}
    +
    \frac{10\log K}{7\gamma_t}
    +
    \frac{6}{t\epsilon_t\Delta_{\min}}\mathcal J_t^{\rm bd},
\]
which implies the stated bound because \(10/7\le2\).

Let \(q\) be the marginal law of \(\bar A_t\).  Since
\(\mathbb P(\bar A_t=a\mid H_t^{\rm bd})=\rho_t^{\rm exp}(a)\),
\[
    \mathcal J_t^{\rm bd}
    =
    \mathbb E[\KL(\rho_t^{\rm exp}\Vert q)].
\]
For any prior \(\pi_0\), the prior comparison identity gives
\[
    \mathbb E[\KL(\rho_t^{\rm exp}\Vert\pi_0)]
    =
    \mathbb E[\KL(\rho_t^{\rm exp}\Vert q)]
    +
    \KL(q\Vert\pi_0),
\]
so \(\mathcal J_t^{\rm bd}\le
\mathbb E[\KL(\rho_t^{\rm exp}\Vert\pi_0)]\).  The uniform prior gives
\(\mathcal J_t^{\rm bd}\le\log K\).

The cumulative bound follows by the same balancing argument as in
Theorem~\ref{thm:bandit-main}, with the constant \(52\) replaced by \(6\).
\end{proof}

\paragraph{Comparison with the supersample bound.}
The ordinary-MI proof is shorter because, after conditioning on the {}{history before the round and on \(\pi_s\)}, the policy is fixed and the fixed-arm Bernstein condition
above is exactly the required condition.  No row-swap identity,
{additional second-moment comparison, or chain rule along {}{the observed history}} is needed.  The price
is that the information term is coarser: \(\mathcal J_t^{\rm bd}\) charges the
dependence of the whole terminal arm on {}{the observed history}, whereas the
supersample theorem charges the roundwise {SCMI} increments
\(I(\mathsf G_{s,t}^{\rm bd};U_s\mid G_{s-1}^{\rm bd})\).  In the finite
arm case both information quantities are bounded by \(\log K\), so both
arguments give the same square-root-in-the-horizon cumulative rate up to logarithmic
and gap-dependent factors. The ordinary-MI argument should therefore be viewed as a
finite-arm posterior baseline, not as a replacement for the supersample result. It
controls the adaptivity of the terminal randomized arm as a whole, whereas the
{SCMI} proof {}{quantifies, at each round, what the gap estimate reveals about the selector \(U_s\) in the }{comparison of {}{gap estimates computed from the two arm--reward pairs}}. This localization is the part that extends naturally to
the supersample, active-learning, and online-learning interpretations developed in
the paper.

% Numerical experiment; the three plots share the results page with both tables.
\section{Numerical experiment: SCMI bounds for online Gibbs learning}\label{app:scmi-numerical}
We evaluate the two SCMI bounds of Theorem~\ref{thm:general-rademacher}, one based on the loss and the other on the classifier, in a threshold classification problem in which the observation distribution depends on the learning history. We compare the expected generalization gap of the classifier immediately after each update with the two bounds over different numbers of rounds. We also examine how the SCMI of the loss changes with the inverse temperature of the Gibbs update.

\subsection{Experimental setup}
For the updated classifier $W_t$, let $D_t$ be the loss on the observation held out at round $t$ minus the loss on the observation used for the update, and let $\Delta_T$ be its expected value averaged over $T$ rounds:
\begin{equation}
 D_t:=\ell(W_t,Z_{t,1-U_t})-\ell(W_t,Z_{t,U_t}),
 \qquad
 \Delta_T:=\frac1T\sum_{t=1}^T\mathbb E[D_t]
 \label{eq:sg-gap}
\end{equation}
As in the main text, $G_{t-1}$ contains the history of observations used by the learner and the current pair of candidates. Here $Q_t=1$ at every round, and the two SCMI quantities are
\begin{equation}
 I_t^{\mathrm{loss}}:=I(L_t^+;U_t\mid G_{t-1}),
 \qquad
 L_t^+:=\ell(W_t,Z_{t,0}),
 \qquad
 I_t^{\mathrm{state}}:=I(W_t;U_t\mid G_{t-1})
 \label{eq:sg-info-def}
\end{equation}
Under the conditions satisfied by this experiment, Theorem~\ref{thm:general-rademacher} gives
\begin{equation}
 |\Delta_T|
 \leq
 \underbrace{\frac2T\sum_{t=1}^T\sqrt{2I_t^{\mathrm{loss}}}}_{B_T^{\mathrm{loss}}}
 \leq
 \underbrace{2\sqrt{\frac2T\sum_{t=1}^T I_t^{\mathrm{state}}}}_{B_T^{\mathrm{state}}}
 \label{eq:sg-bounds}
\end{equation}
We compute the expected gap $\Delta_T$ in \eqref{eq:sg-gap} and both bounds in \eqref{eq:sg-bounds}, and check whether the bounds fall below the trivial upper bound of $1$ for zero--one loss.

The hypothesis class is $\mathcal H=\{h_1,\ldots,h_{21}\}$, where $h_j(x)=\mathbf 1\{x\geq\theta_j\}$ and the thresholds $\theta_j=-1+(j-1)/10$ are equally spaced. The initial distribution $\rho_0$ is uniform on $\mathcal H$. Let $H_{t-1}$ be the history of observations selected before round $t$ and let $\bar\theta_{t-1}=\sum_{j=1}^{21}\rho_{t-1}(h_j)\theta_j$ be the mean threshold before the update. Given $H_{t-1}$, generate the two inputs independently. For each $u\in\{0,1\}$, draw $X_{t,u}$ uniformly from $[-1,1]$ with probability $0.3$; with the remaining probability $0.7$, set
\[
 X_{t,u}=\operatorname{clip}_{[-1,1]}\!\left(\bar\theta_{t-1}+0.2\,\varepsilon_{t,u}\right),
 \qquad \varepsilon_{t,u}\sim\mathcal N(0,1)
\]
Here $\operatorname{clip}_{[-1,1]}$ maps a value outside the interval to its nearest endpoint. Generate the labels as
\[
 Y_{t,u}=\mathbf 1\{X_{t,u}\geq0\}\oplus\xi_{t,u},
 \qquad
 \xi_{t,u}\sim\operatorname{Bernoulli}(0.1)
\]
where $\oplus$ denotes exclusive OR. The label flips are independent of the inputs and of the other observations. Write $Z_{t,u}=(X_{t,u},Y_{t,u})$ and use the loss $\ell(h,(x,y))=\mathbf 1\{h(x)\neq y\}$.

A selector $U_t\sim\operatorname{Bernoulli}(1/2)$, independent of the history and the current pair, chooses the observation used for learning. For inverse temperature $\eta>0$, the distribution obtained by updating on coordinate $u$ and the distribution actually used are
\begin{equation}
 \rho_t^{(u)}(h)=
 \frac{\rho_{t-1}(h)\exp\{-\eta\ell(h,Z_{t,u})\}}
 {\sum_{h'\in\mathcal H}\rho_{t-1}(h')\exp\{-\eta\ell(h',Z_{t,u})\}},
 \qquad
 \rho_t=\rho_t^{(U_t)}
 \label{eq:sg-update}
\end{equation}
The classifier $W_t$ evaluated in this experiment is drawn from $\rho_t$. Given $H_{t-1}$, the two candidates are independent and identically distributed, so their joint law is invariant under a swap. The update uses only the chosen coordinate, and the loss lies in $[0,1]$. Thus the hypotheses of Theorem~\ref{thm:general-rademacher} hold. The expected loss on the other coordinate equals the expected loss of the updated classifier on a fresh draw from the conditional observation distribution at that round.

\subsection{Estimating the generalization gap and SCMI}
Fix $G_{t-1}$ and $U_t=u$. The conditional distribution of $W_t$ is $\rho_t^{(u)}$, so its expected loss on coordinate $v$ is
\begin{equation}
 r_t(u,v)
 :=\mathbb E[\ell(W_t,Z_{t,v})\mid G_{t-1},U_t=u]
 =\sum_{j=1}^{21}\rho_t^{(u)}(h_j)\ell(h_j,Z_{t,v})
 \label{eq:sg-ruv}
\end{equation}
The conditional expectation of $D_t$ when $U_t=u$ is therefore $r_t(u,1-u)-r_t(u,u)$. Since $U_t=0$ and $U_t=1$ are equally likely given $G_{t-1}$,
\begin{equation}
 d_t:=\mathbb E[D_t\mid G_{t-1}]=\frac12\{r_t(0,1)-r_t(0,0)+r_t(1,0)-r_t(1,1)\}.
 \label{eq:sg-dt}
\end{equation}
By iterated expectation, the expected generalization gap in \eqref{eq:sg-gap} is
\begin{equation}
 \Delta_T=\frac1T\sum_{t=1}^T\mathbb E[d_t]
 \label{eq:sg-gap-final}
\end{equation}
For a fixed history and candidate pair, \eqref{eq:sg-ruv}--\eqref{eq:sg-dt} can be evaluated as finite sums over the 21 hypotheses.

Each conditional mutual information in \eqref{eq:sg-info-def} is the average, over $G_{t-1}$, of the mutual information under the conditional law for each realization of $G_{t-1}$. Fix $G_{t-1}=g$ and denote the local values for the classifier and the loss by $j_t^{\mathrm{state}}(g)$ and $j_t^{\mathrm{loss}}(g)$, respectively. We suppress the argument $g$ below. We use natural logarithms and take every summand with zero joint probability to be zero.

\paragraph{Local mutual information for the classifier.}
Given $G_{t-1}$, the probability of $U_t=u$ is $1/2$, and given $U_t=u$ the probability of $W_t=h_j$ is $\rho_t^{(u)}(h_j)$. Thus
\begin{equation}
 \Pr(U_t=u,W_t=h_j\mid G_{t-1})
 =\frac12\rho_t^{(u)}(h_j).
 \label{eq:sg-state-joint}
\end{equation}
When $U_t$ is not observed, the marginal distribution of $W_t$ is the equal mixture of the two update distributions:
\begin{equation}
 \bar\rho_t(h_j)
 :=\Pr(W_t=h_j\mid G_{t-1})
 =\frac{\rho_t^{(0)}(h_j)+\rho_t^{(1)}(h_j)}2
 \label{eq:sg-rhobar}
\end{equation}
This marginal appears in the denominator because mutual information uses the ratio of the joint probability to the product of its marginals. Applying the definition to \eqref{eq:sg-state-joint}--\eqref{eq:sg-rhobar} gives
\begin{equation}
 \begin{aligned}
 j_t^{\mathrm{state}}
 &:=I(W_t;U_t\mid G_{t-1}=g)\\
 &=\sum_{u=0}^1\sum_{j=1}^{21}
 \Pr(U_t=u,W_t=h_j\mid g)
 \log\frac{\Pr(U_t=u,W_t=h_j\mid g)}
 {\Pr(U_t=u\mid g)\Pr(W_t=h_j\mid g)}\\
 &=\frac12\sum_{j=1}^{21}\rho_t^{(0)}(h_j)
 \log\frac{\rho_t^{(0)}(h_j)}{\bar\rho_t(h_j)}
 +\frac12\sum_{j=1}^{21}\rho_t^{(1)}(h_j)
 \log\frac{\rho_t^{(1)}(h_j)}{\bar\rho_t(h_j)}.
 \end{aligned}
 \label{eq:sg-jstate-explicit}
\end{equation}
We compute both sums in \eqref{eq:sg-jstate-explicit} directly over $j=1,\ldots,21$, without estimating the mutual information from repeated draws of $W_t$.

\paragraph{Local mutual information for the loss.}
Whether the update uses coordinate $0$ or $1$, we evaluate the loss on the same coordinate, namely
$L_t^+=\ell(W_t,Z_{t,0})\in\{0,1\}$. When the update uses coordinate $u$, the conditional probability of loss one is
\begin{equation}
 p_{t,u}
 :=\Pr(L_t^+=1\mid G_{t-1},U_t=u)
 =\sum_{j=1}^{21}\rho_t^{(u)}(h_j)\ell(h_j,Z_{t,0})
 =r_t(u,0).
 \label{eq:sg-ptu-explicit}
\end{equation}
To compute $p_{t,u}$, we sum the probabilities $\rho_t^{(u)}(h_j)$ of those classifiers among the 21 that misclassify $Z_{t,0}$. The conditional probability of loss zero is $1-p_{t,u}$. When $U_t$ is not observed, the marginal probabilities are
\begin{equation}
 \bar p_t:=\Pr(L_t^+=1\mid G_{t-1})=\frac{p_{t,0}+p_{t,1}}2,
 \qquad
 \Pr(L_t^+=0\mid G_{t-1})=1-\bar p_t
 \label{eq:sg-pbar}
\end{equation}
The joint probabilities are therefore
\begin{equation}
 \Pr(U_t=u,L_t^+=1\mid G_{t-1})=\frac{p_{t,u}}2,
 \qquad
 \Pr(U_t=u,L_t^+=0\mid G_{t-1})=\frac{1-p_{t,u}}2.
 \label{eq:sg-loss-joint}
\end{equation}
Applying the definition of mutual information to the four combinations of $u\in\{0,1\}$ and $l\in\{0,1\}$ gives
\begin{equation}
 \begin{aligned}
 j_t^{\mathrm{loss}}
 &:=I(L_t^+;U_t\mid G_{t-1}=g)\\
 &=\sum_{u=0}^1\sum_{l=0}^1
 \Pr(U_t=u,L_t^+=l\mid g)
 \log\frac{\Pr(U_t=u,L_t^+=l\mid g)}
 {\Pr(U_t=u\mid g)\Pr(L_t^+=l\mid g)}\\[-1pt]
 &={\textstyle\frac12}\left[p_{t,0}\log\frac{p_{t,0}}{\bar p_t}+(1-p_{t,0})\log\frac{1-p_{t,0}}{1-\bar p_t}\right]+{\textstyle\frac12}\left[p_{t,1}\log\frac{p_{t,1}}{\bar p_t}+(1-p_{t,1})\log\frac{1-p_{t,1}}{1-\bar p_t}\right].
 \end{aligned}
 \label{eq:sg-jloss-explicit}
\end{equation}
We compute $p_{t,0}$ and $p_{t,1}$ from the finite sums in \eqref{eq:sg-ptu-explicit}, obtain $\bar p_t$ from \eqref{eq:sg-pbar}, and evaluate the four terms in \eqref{eq:sg-jloss-explicit} directly. This calculation likewise requires no repeated sampling of the loss.

The SCMI terms at round $t$ in \eqref{eq:sg-info-def} are the averages of these local mutual informations over the random $G_{t-1}$:
\begin{equation}
 I_t^{\mathrm{loss}}=\mathbb E_{G_{t-1}}[j_t^{\mathrm{loss}}(G_{t-1})],
 \qquad
 I_t^{\mathrm{state}}=\mathbb E_{G_{t-1}}[j_t^{\mathrm{state}}(G_{t-1})]
 \label{eq:sg-It-local}
\end{equation}

\paragraph{Monte Carlo evaluation.} For a fixed history and candidate pair, we can evaluate \eqref{eq:sg-dt}, \eqref{eq:sg-jstate-explicit}, and \eqref{eq:sg-jloss-explicit} exactly. We use Monte Carlo only for the outer expectations over the history and candidate pair in \eqref{eq:sg-gap-final} and \eqref{eq:sg-It-local}.

Fix an inverse temperature $\eta$ and generate $N$ independent trajectories. At round $t$ on trajectory $i$, compute the mean threshold under the preceding distribution $\rho_{t-1}^{(i)}$,
\[
 \bar\theta_{t-1}^{(i)}:=\sum_{j=1}^{21}\rho_{t-1}^{(i)}(h_j)\theta_j
\]
and use it to generate one candidate pair $(Z_{t,0}^{(i)},Z_{t,1}^{(i)})$ from the conditional distribution specified above. Starting from the same preceding distribution and using the same candidate pair, compute both updates $\rho_t^{(0,i)}$ and $\rho_t^{(1,i)}$ in \eqref{eq:sg-update}. Then compute $d_t^{(i)}$ from \eqref{eq:sg-ruv}--\eqref{eq:sg-dt}, $j_t^{\mathrm{state},(i)}$ from \eqref{eq:sg-jstate-explicit}, and $j_t^{\mathrm{loss},(i)}$ from \eqref{eq:sg-ptu-explicit}--\eqref{eq:sg-jloss-explicit}, in each case as a finite sum. Evaluating the current round does not require sampling $W_t$ or $L_t^+$.

To generate the history for the next round, draw $U_t^{(i)}\sim\operatorname{Bernoulli}(1/2)$ once on each trajectory and set
\[
 \rho_t^{(i)}=\rho_t^{(U_t^{(i)},i)}
\]
Only this chosen update distribution carries over to the next round. The observations selected along different trajectories give rise to different histories, and the distribution of future candidate pairs changes with each history.

Averaging the $N$ values obtained at the same round $t$ gives
\begin{equation}
 \widehat d_t=\frac1N\sum_{i=1}^N d_t^{(i)},
 \qquad
 \widehat I_t^{\mathrm{loss}}=\frac1N\sum_{i=1}^N j_t^{\mathrm{loss},(i)},
 \qquad
 \widehat I_t^{\mathrm{state}}=\frac1N\sum_{i=1}^N j_t^{\mathrm{state},(i)}
 \label{eq:sg-mc}
\end{equation}
On each trajectory we first compute $d_t^{(i)}$, $j_t^{\mathrm{loss},(i)}$, and $j_t^{\mathrm{state},(i)}$ exactly from the two possible updates for the same history and candidate pair. We then average over trajectories. This average approximates the expectations over the random $G_{t-1}$ in \eqref{eq:sg-gap-final} and \eqref{eq:sg-It-local}. We do not pool update distributions or losses from different trajectories to compute a single mutual information.

The reported averages are $\widehat\Delta_T=\frac1T\sum_{t=1}^T\widehat d_t$ and $\widehat{\bar I}_T^{\mathrm{loss}}=\frac1T\sum_{t=1}^T\widehat I_t^{\mathrm{loss}}$. Substituting the roundwise estimates into \eqref{eq:sg-bounds} gives $\widehat B_T^{\mathrm{loss}}=\frac2T\sum_{t=1}^T\sqrt{2\widehat I_t^{\mathrm{loss}}}$ and $\widehat B_T^{\mathrm{state}}=2\sqrt{\frac2T\sum_{t=1}^T\widehat I_t^{\mathrm{state}}}$.
For $\widehat B_T^{\mathrm{loss}}$, we average over trajectories at each round, take the square root, and then average over rounds. For $\widehat B_T^{\mathrm{state}}$, we also average the roundwise trajectory averages over time before taking the square root. Neither bound is computed by first forming a bound for each trajectory and then averaging.

We estimate the standard error of $\widehat\Delta_T$ from the $N$ independent time averages $g_i=\frac1T\sum_{t=1}^T d_t^{(i)}$, accounting for dependence between rounds within a trajectory. To compare horizons, we generate $N=20{,}000$ trajectories of length 200 at $\eta=0.2$ and use the first $T$ rounds of each trajectory for $T=25,50,100,200$. To compare inverse temperatures, we generate $N=10{,}000$ independent trajectories of length $T=200$ for each $\eta$.

\subsection{Results}
\paragraph{Effect of the number of rounds.}
At inverse temperature $\eta=0.2$, we generated 20,000 independent trajectories of length 200 and used the first $T$ rounds of each trajectory (Table~\ref{tab:scmi-gibbs-horizon}). At $T=200$, the estimated expected generalization gap was $0.01224$, the bound estimated from the loss was $0.04691$, and the bound estimated from the classifier was $0.06991$. For every reported $T$, both estimated bounds were below the trivial upper bound of $1$ for zero--one loss, and the bound based on the loss was smaller, as in \eqref{eq:sg-bounds}.

\begin{table}[H]\centering\small
\caption{Estimated expected generalization gap and SCMI bounds at different horizons $T$ ($\eta=0.2$).}\label{tab:scmi-gibbs-horizon}
\begin{tabular}{r r r r r}\toprule
$T$ & Expected gap & Standard error & $\widehat B_T^{\mathrm{loss}}$ & $\widehat B_T^{\mathrm{state}}$ \\ \midrule
25 & 0.03644 & $4.22\times10^{-5}$ & 0.10139 & 0.12063 \\
50 & 0.03143 & $2.76\times10^{-5}$ & 0.09102 & 0.11204 \\
100 & 0.02189 & $1.85\times10^{-5}$ & 0.07137 & 0.09351 \\
200 & 0.01224 & $1.12\times10^{-5}$ & 0.04691 & 0.06991 \\ \bottomrule
\end{tabular}\end{table}

\paragraph{Effect of the inverse temperature.}
With $T=200$ fixed, we generated 10,000 new independent trajectories for each $\eta$ (Table~\ref{tab:scmi-gibbs-eta}). As $\eta$ increased from $0.05$ to $1.00$, the average SCMI of the loss per round, $T^{-1}\sum_{t=1}^T I_t^{\mathrm{loss}}$, increased about 28-fold, from $6.88\times10^{-5}$ to $1.91\times10^{-3}$. The corresponding bound increased from $0.02324$ to $0.06565$, and the expected generalization gap from $0.00812$ to $0.01304$. The inverse temperature also affects the distribution of future inputs, so this comparison concerns the entire adaptive learning process.

\begin{table}[H]\centering\small
\caption{Estimates of the average SCMI of the loss, the expected generalization gap, and the two bounds at different inverse temperatures $\eta$ ($T=200$).}\label{tab:scmi-gibbs-eta}
\begin{tabular}{r r r r r}\toprule
$\eta$ & $\widehat{\bar I}_T^{\mathrm{loss}}$ & Expected gap & $\widehat B_T^{\mathrm{loss}}$ & $\widehat B_T^{\mathrm{state}}$ \\ \midrule
0.05 & $6.88\times10^{-5}$ & 0.00812 & 0.02324 & 0.02850 \\
0.10 & $1.77\times10^{-4}$ & 0.01104 & 0.03579 & 0.04698 \\
0.20 & $3.83\times10^{-4}$ & 0.01220 & 0.04680 & 0.06980 \\
0.50 & $9.59\times10^{-4}$ & 0.01255 & 0.05668 & 0.11149 \\
1.00 & $1.91\times10^{-3}$ & 0.01304 & 0.06565 & 0.15860 \\ \bottomrule
\end{tabular}\end{table}

\begin{figure}[H]\centering
\noindent
\includegraphics[width=.328\linewidth]{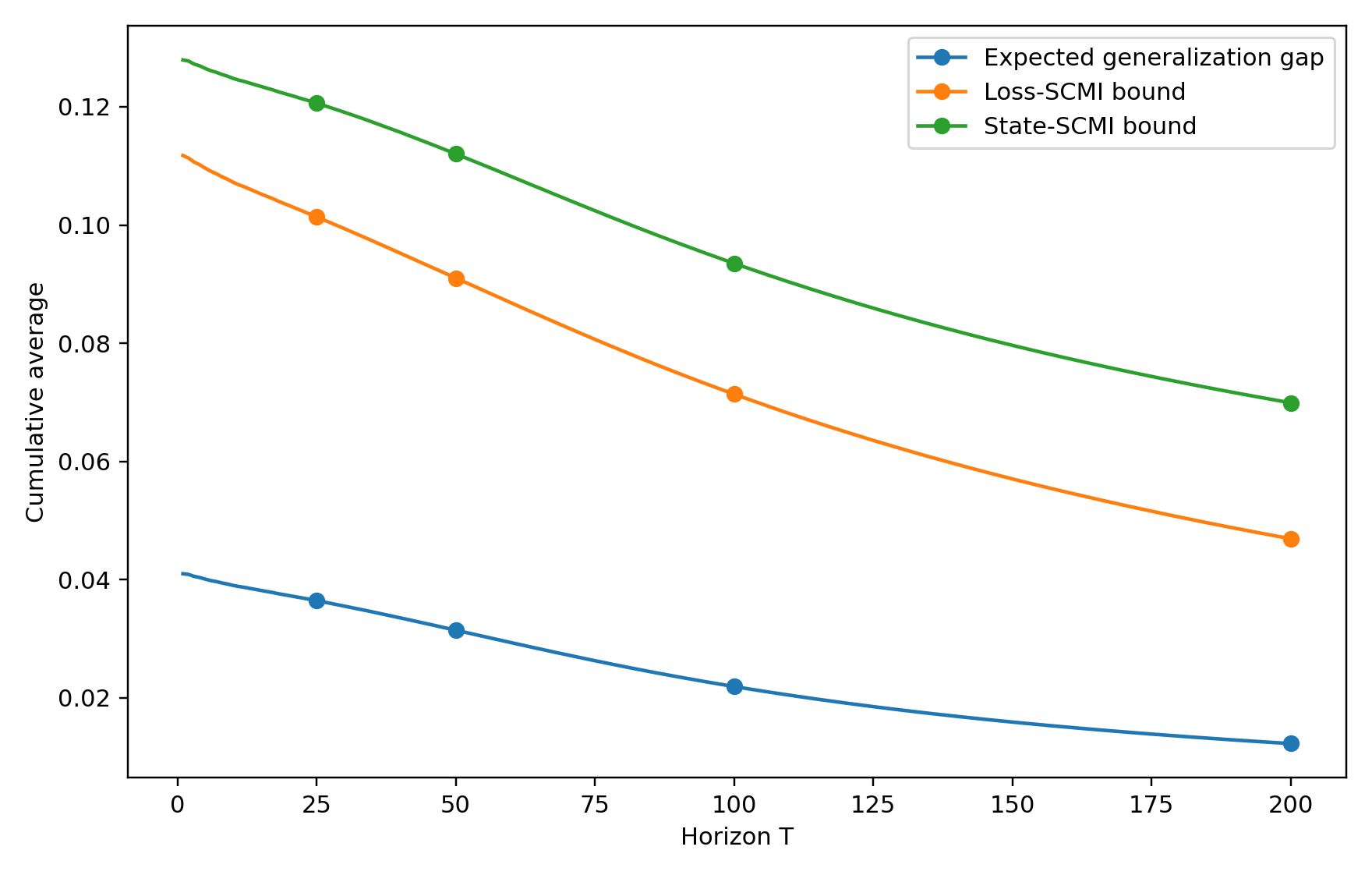}\hfill%
\includegraphics[width=.328\linewidth]{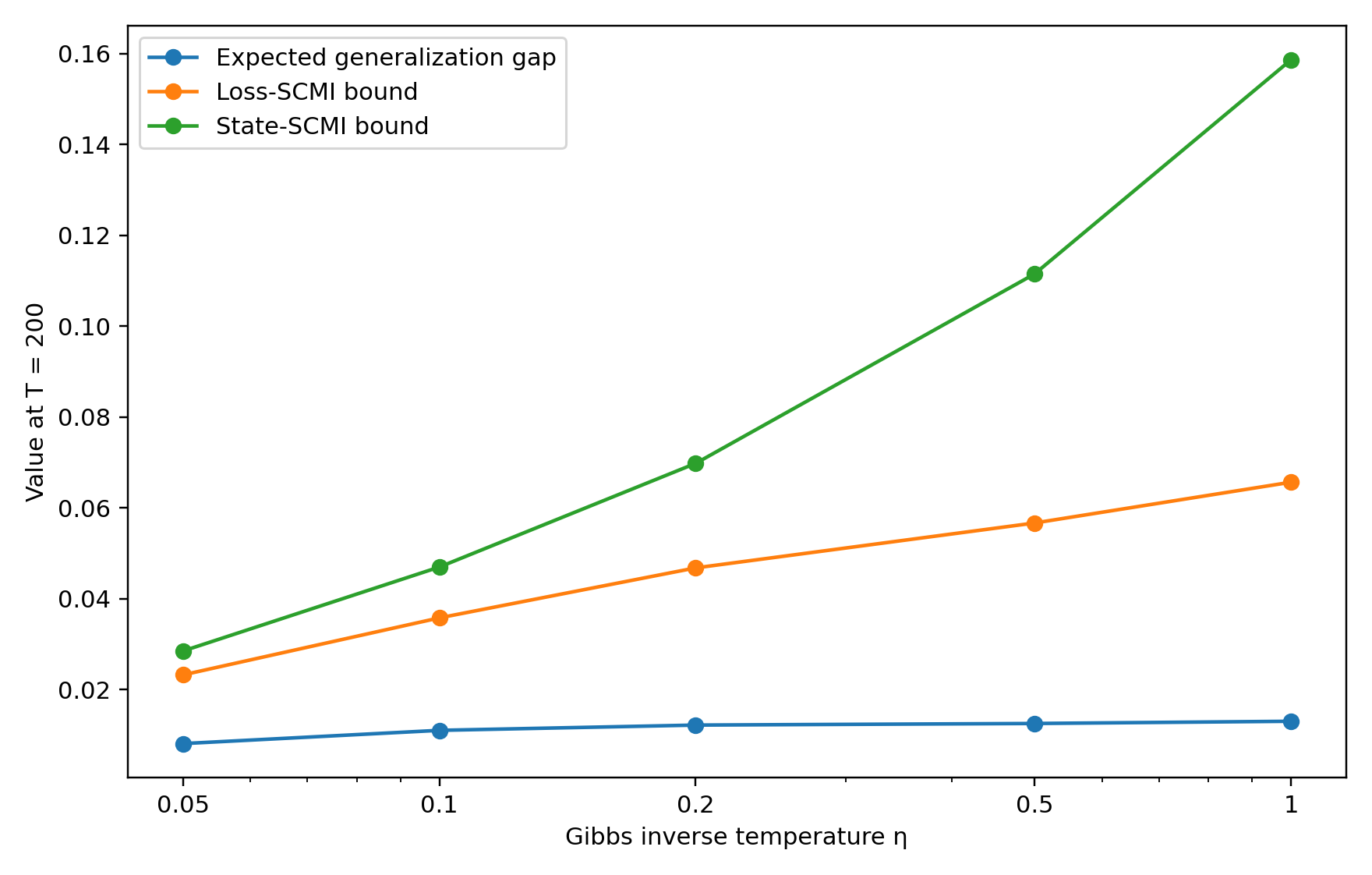}\hfill%
\includegraphics[width=.328\linewidth]{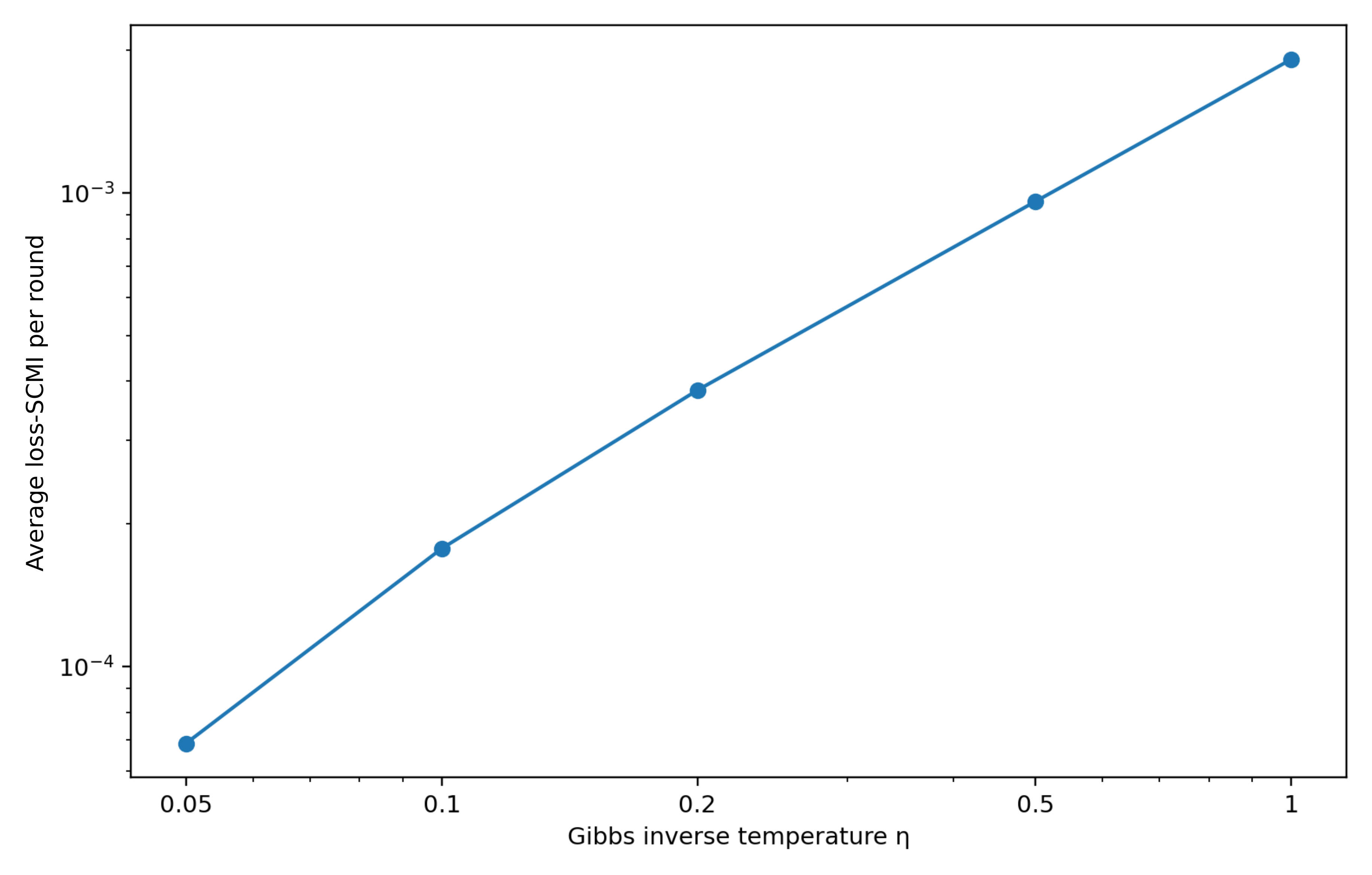}
\caption{Left: gap and bounds versus $T$ at $\eta=0.2$ (circles: horizons in Table~\ref{tab:scmi-gibbs-horizon}). Center and right: gap and bounds, and mean SCMI of the loss per round, respectively, versus $\eta$ at $T=200$.}\label{fig:scmi-gibbs}
\end{figure}

\section{Analytical SCMI bounds in toy models and their asymptotic orders}\label{app:scmi-toy-models}

We use three binary models to show why the usual batch CMI swap need not preserve the law of data generated from the learner's history, and why the square-root dependence in Theorem~\ref{thm:general-rademacher} and the linear dependence in Theorem~\ref{thm:general-bernstein} cannot in general be improved. Throughout, $Q_t=1$, and at each round the fair selector $U_t$ is drawn independently of the history and the current pair. For $a,b\in\{0,1\}$, write $a\oplus b:=(a+b)\bmod 2$. Thus $0\oplus0=1\oplus1=0$ and $0\oplus1=1\oplus0=1$.

\subsection{Limits of the usual batch CMI argument for a history-dependent process}

In this example, a past selector changes the law of future observations. We first show why changing the selector while holding the entire supersample fixed, as in the usual batch CMI argument, fails to preserve the sequential law. We then compare the two candidates in the current round conditional on the past, as in the SCMI argument.

Consider two constant classifiers $h_w(x)=w$, $w\in\{0,1\}$, and identify the learner's state $W_t\in\{0,1\}$ with the classifier $h_{W_t}$. Let $W_0\sim\operatorname{Bernoulli}(1/2)$ and fix $0<\eta<1/2$. At each round, draw fresh noise variables
\[
E_{t,0},E_{t,1}\stackrel{\mathrm{i.i.d.}}{\sim}\operatorname{Bernoulli}(\eta)
\]
and set
\begin{equation}
Y_{t,j}=W_{t-1}\oplus E_{t,j},
\qquad
W_t=Y_{t,U_t}.
\label{eq:appendix-l-v8-1}
\end{equation}
All noise variables $E_{t,j}$ and selectors $U_t$ are mutually independent.

\paragraph{Why the usual batch CMI swap fails.}
Suppose $(Y_{s,0},Y_{s,1})=(0,1)$ for some $s<n$. Conditioning on a value $G_{s-1}=g$ that contains this pair fixes neither $U_s$ nor the noise in the next round. For this pair,
\[
U_s=0\Longrightarrow W_s=Y_{s,0}=0,
\qquad
U_s=1\Longrightarrow W_s=Y_{s,1}=1
\]
The first coordinate at the next round is $Y_{s+1,0}=W_s\oplus E_{s+1,0}$. Since $E_{s+1,0}$ is independent of $(G_{s-1},U_s)$ and has law $\operatorname{Bernoulli}(\eta)$,
\begin{align}
\Pr(Y_{s+1,0}=1\mid G_{s-1}=g,U_s=0)
&=\Pr(0\oplus E_{s+1,0}=1)
 =\Pr(E_{s+1,0}=1)=\eta,
\label{eq:appendix-l-v8-2}\\
\Pr(Y_{s+1,0}=1\mid G_{s-1}=g,U_s=1)
&=\Pr(1\oplus E_{s+1,0}=1)
 =\Pr(E_{s+1,0}=0)=1-\eta.
\label{eq:appendix-l-v8-3}
\end{align}
The same calculation holds for coordinate $1$. Flipping $U_s$ therefore changes both $W_s$ and the conditional law of the next pair: that law is $\operatorname{Bernoulli}(\eta)^{\otimes2}$ when $U_s=0$ and $\operatorname{Bernoulli}(1-\eta)^{\otimes2}$ when $U_s=1$. Hence flipping only the past selector while keeping the future pairs fixed does not preserve the sequential joint law.

The dependence also appears in the observations available to the learner. The observation at round $t$ is $Y_{t,U_t}=W_t$, and Eq.~\eqref{eq:appendix-l-v8-1} gives
\[
W_t=W_{t-1}\oplus E_{t,U_t}
\]
If $W_{t-1}$ is uniform on $\{0,1\}$, its independence from $E_{t,U_t}\sim\operatorname{Bernoulli}(\eta)$ gives
\[
\Pr(W_t=1)
=\tfrac12\eta+\tfrac12(1-\eta)
=\tfrac12,
\]
so by induction every $W_t$ is marginally uniform on $\{0,1\}$. However,
\begin{equation}
\Pr(W_t=W_{t-1})
=\Pr(E_{t,U_t}=0)
=1-\eta
\neq\tfrac12.
\label{eq:appendix-l-v8-4}
\end{equation}
Two independent uniform binary variables agree with probability $1/2$, so $W_t$ and $W_{t-1}$ are dependent. Thus the sequence observed by the learner, $(Y_{t,U_t})_{t\ge1}=(W_t)_{t\ge1}$, is not i.i.d., and a prefix of length at least two cannot be treated as an i.i.d.\ batch. Equations~\eqref{eq:appendix-l-v8-2}--\eqref{eq:appendix-l-v8-4} show why the usual batch CMI swap does not apply to this model.

\paragraph{Sequential comparison conditional on the history.}
Conditioning on $H_{t-1}$ fixes $W_{t-1}$. The two candidates in Eq.~\eqref{eq:appendix-l-v8-1} are then independent draws from the same conditional law. The comparison in Theorem~\ref{thm:general-rademacher} thus applies to this pair.

Take $\ell(h_w,y)=\mathbf 1\{w\neq y\}$. Since $W_t=Y_{t,U_t}$, the training loss is
\[
\ell(h_{W_t},Y_{t,U_t})
=\mathbf 1\{W_t\neq Y_{t,U_t}\}
=\mathbf 1\{Y_{t,U_t}\neq Y_{t,U_t}\}
=0,
\]
and the holdout loss is
\[
\ell(h_{W_t},Y_{t,1-U_t})
=\mathbf 1\{Y_{t,U_t}\neq Y_{t,1-U_t}\}
=\mathbf 1\{Y_{t,0}\neq Y_{t,1}\}.
\]

By Eq.~\eqref{eq:appendix-l-v8-1} and the associativity and commutativity of addition modulo two,
\begin{equation}
\begin{aligned}
Y_{t,0}\oplus Y_{t,1}
&=(W_{t-1}\oplus E_{t,0})\oplus(W_{t-1}\oplus E_{t,1})\\
&=(W_{t-1}\oplus W_{t-1})\oplus(E_{t,0}\oplus E_{t,1})\\
&=E_{t,0}\oplus E_{t,1}.
\end{aligned}
\label{eq:appendix-l-v8-5}
\end{equation}
where the last equality uses $W_{t-1}\oplus W_{t-1}=0$. For binary $a,b$, we have
\[
a\neq b
\quad\Longleftrightarrow\quad
a\oplus b=1
\]
so Eq.~\eqref{eq:appendix-l-v8-5} gives
\[
Y_{t,0}\neq Y_{t,1}
\quad\Longleftrightarrow\quad
E_{t,0}\neq E_{t,1}.
\]
It follows that
\begin{equation}
\begin{aligned}
\Pr(Y_{t,0}\neq Y_{t,1})
&=\Pr(E_{t,0}=0,E_{t,1}=1)
 +\Pr(E_{t,0}=1,E_{t,1}=0)\\
&=(1-\eta)\eta+\eta(1-\eta)\\
&=2\eta(1-\eta).
\end{aligned}
\label{eq:appendix-l-v8-6}
\end{equation}
Hence
\begin{equation}
\mathrm{Err}^{\mathrm{seq}}_n=2\eta(1-\eta).
\label{eq:appendix-l-v8-7}
\end{equation}

We next compute the SCMI of $L_t^+=\ell(h_{W_t},Y_{t,0})$. Fix a value $G_{t-1}=g$ with ordered pair $(Y_{t,0},Y_{t,1})=(y_0,y_1)$.

If $y_0=y_1$, then $W_t=y_0$ and $L_t^+=0$ for either value of $U_t$, so
\[
I(L_t^+;U_t\mid G_{t-1}=g)
=I(W_t;U_t\mid G_{t-1}=g)
=0.
\]

If $y_0\neq y_1$, the labels are binary and hence $y_1=y_0\oplus1$. Thus
\begin{equation}
Y_{t,U_t}
=\begin{cases}
y_0,&U_t=0,\\
y_0\oplus1,&U_t=1
\end{cases}
=y_0\oplus U_t.
\label{eq:appendix-l-v8-8}
\end{equation}
It follows that
\begin{equation}
W_t=y_0\oplus U_t,
\qquad
L_t^+=\mathbf 1\{W_t\neq y_0\}=U_t.
\label{eq:appendix-l-v8-9}
\end{equation}
Equation~\eqref{eq:appendix-l-v8-9} gives $U_t=L_t^+$, so
\begin{equation}
H(U_t\mid L_t^+,G_{t-1}=g)=0.
\label{eq:appendix-l-v8-10}
\end{equation}
Equation~\eqref{eq:appendix-l-v8-9} also gives
\[
U_t=W_t\oplus y_0,
\]
Thus the fixed label $y_0$ and the state $W_t$ uniquely determine $U_t$, and
\begin{equation}
H(U_t\mid W_t,G_{t-1}=g)=0.
\label{eq:appendix-l-v8-11}
\end{equation}
Since $U_t$ remains fair given $G_{t-1}$,
\[
H(U_t\mid G_{t-1}=g)=\ln2.
\]
For a pair with unequal labels, Eqs.~\eqref{eq:appendix-l-v8-10} and~\eqref{eq:appendix-l-v8-11} therefore give
\begin{align}
I(L_t^+;U_t\mid G_{t-1}=g)
&=H(U_t\mid G_{t-1}=g)-H(U_t\mid L_t^+,G_{t-1}=g)
=\ln2,
\label{eq:appendix-l-v8-12}\\
I(W_t;U_t\mid G_{t-1}=g)
&=H(U_t\mid G_{t-1}=g)-H(U_t\mid W_t,G_{t-1}=g)
=\ln2.
\label{eq:appendix-l-v8-13}
\end{align}
Both quantities are zero when the labels agree. Averaging over the pair using the disagreement probability in Eq.~\eqref{eq:appendix-l-v8-6} yields
\begin{equation}
I(L_t^+;U_t\mid G_{t-1})
=I(W_t;U_t\mid G_{t-1})
=2\eta(1-\eta)\ln2.
\label{eq:appendix-l-v8-14}
\end{equation}

By Eq.~\eqref{eq:appendix-l-v8-7}, the expected holdout loss minus the expected training loss is $2\eta(1-\eta)$ at each round. By Eq.~\eqref{eq:appendix-l-v8-14}, the SCMI of $L_t^+$ is $\ln 2$ times this gap. Averaging over $t=1,\ldots,n$ gives
\begin{equation}
\mathrm{Err}^{\mathrm{seq}}_n
=\frac1n\sum_{t=1}^n2\eta(1-\eta)
=\frac{1}{n\ln2}\sum_{t=1}^n I(L_t^+;U_t\mid G_{t-1}).
\label{eq:appendix-l-v8-15}
\end{equation}
Substitution into Theorem~\ref{thm:general-rademacher} gives
\[
\mathrm{Err}^{\mathrm{seq}}_n
\le \frac{2}{n}\sum_{t=1}^n\sqrt{2I(L_t^+;U_t\mid G_{t-1})}
=4\sqrt{\eta(1-\eta)\ln2}.
\]

When a past selector changes the law of future observations, the usual batch CMI argument cannot hold all observations fixed and swap the selector. The SCMI comparison instead fixes the past at each round and compares the two current candidates drawn from the same conditional law. This preserves the sequential law when relating the selector to the training and holdout losses.

\subsection{Square-root dependence in Theorem~\ref{thm:general-rademacher}}

Even when the candidate pair at each round is independent of the past, the generalization gap can be of order the square root of the average SCMI per round.

At each round draw $Z_{t,0},Z_{t,1}\stackrel{\mathrm{i.i.d.}}{\sim}\operatorname{Bernoulli}(1/2)$. For $0<\delta<1/2$, draw $N_t\sim\operatorname{Bernoulli}(1/2-\delta)$ independently of the pair, the past, and the selector, and set
\[
W_t=Z_{t,U_t}\oplus N_t.
\]
Use the loss $\ell(w,z)=\mathbf1\{w\neq z\}$. The training loss is
\[
\ell(W_t,Z_{t,U_t})=\mathbf1\{N_t=1\}
\]
so $L_n^{\mathrm{tr}}=1/2-\delta$. Given $U_t=u$, $Z_{t,u}=z$, and $N_t=k$, the prediction $W_t=z\oplus k$ is fixed, while the holdout observation $Z_{t,1-u}$ is independent of these variables and uniform on $\{0,1\}$. Hence
\[
\Pr(Z_{t,1-u}\neq z\oplus k\mid U_t=u,Z_{t,u}=z,N_t=k)=\tfrac12,
\]
Thus $L_n^{\mathrm{ho}}=1/2$, giving
\begin{equation}
\mathrm{Err}^{\mathrm{seq}}_n=\delta.
\label{eq:appendix-l-v8-16}
\end{equation}

To compute the SCMI, fix the current pair. Suppose first that its entries agree, with $Z_{t,0}=Z_{t,1}=z$. For either $u\in\{0,1\}$,
\[
Z_{t,U_t}=z,
\qquad
W_t=z\oplus N_t,
\qquad
L_t^+=\mathbf1\{z\oplus N_t\neq z\}=N_t.
\]
Since $N_t$ is independent of $U_t$, for either $u$ we have
\[
\operatorname{Law}(W_t\mid U_t=u,G_{t-1}=g)
=\operatorname{Law}(z\oplus N_t),
\]
\[
\operatorname{Law}(L_t^+\mid U_t=u,G_{t-1}=g)
=\operatorname{Law}(N_t)
\]
Neither conditional law depends on $u$. Thus, when the entries agree,
\[
I(W_t;U_t\mid G_{t-1}=g)
=I(L_t^+;U_t\mid G_{t-1}=g)
=0.
\]

If the entries differ, the binary calculation in Eq.~\eqref{eq:appendix-l-v8-8} gives
\[
Z_{t,U_t}=Z_{t,0}\oplus U_t.
\]
It follows that
\begin{equation}
W_t=Z_{t,0}\oplus U_t\oplus N_t,
\qquad
L_t^+=\mathbf1\{W_t\neq Z_{t,0}\}=U_t\oplus N_t.
\label{eq:appendix-l-v8-17}
\end{equation}
For fixed $Z_{t,0}$, the identity $W_t=Z_{t,0}\oplus L_t^+$ gives a one-to-one correspondence between $W_t$ and $L_t^+$, so their SCMI terms are equal.

Since $U_t$ is fair and independent of $N_t$,
\[
\Pr(L_t^+=1\mid G_{t-1}=g)
=\tfrac12\Pr(N_t=1)+\tfrac12\Pr(N_t=0)
=\tfrac12.
\]
Thus $H(L_t^+\mid G_{t-1}=g)=\ln2$. Given $U_t=u$, the map $N_t\mapsto L_t^+=u\oplus N_t$ is a bijection. Writing the binary entropy as
\[
h(q):=-q\ln q-(1-q)\ln(1-q)
\]
we obtain
\[
H(L_t^+\mid U_t,G_{t-1}=g)=H(N_t)=h(1/2-\delta).
\]
For unequal entries this gives
\begin{equation}
I(L_t^+;U_t\mid G_{t-1}=g)
=I(W_t;U_t\mid G_{t-1}=g)
=\ln2-h(1/2-\delta).
\label{eq:appendix-l-v8-18}
\end{equation}
The two independent uniform binary observations differ with probability $1/2$, so
\begin{equation}
I(L_t^+;U_t\mid G_{t-1})
=I(W_t;U_t\mid G_{t-1})
=\tfrac12\{\ln2-h(1/2-\delta)\}.
\label{eq:appendix-l-v8-19}
\end{equation}
By symmetry, $h(1/2-\delta)=h(1/2+\delta)$, and as $\delta\downarrow0$,
\[
\ln2-h(1/2-\delta)
=2\delta^2+\frac43\delta^4+O(\delta^6).
\]
The construction is the same at every round. Equation~\eqref{eq:appendix-l-v8-19} and the expansion above thus give
\[
\frac1n\sum_{t=1}^n I(W_t;U_t\mid G_{t-1})
=\delta^2+\frac23\delta^4+O(\delta^6)
=\delta^2\left(1+\frac23\delta^2+O(\delta^4)\right).
\]
Since $\delta>0$,
\[
\sqrt{\frac1n\sum_{t=1}^n I(W_t;U_t\mid G_{t-1})}
=\delta\sqrt{1+\frac23\delta^2+O(\delta^4)}.
\]
By Eq.~\eqref{eq:appendix-l-v8-16}, $|\mathrm{Err}^{\mathrm{seq}}_n|=\delta$, so
\[
\frac{|\mathrm{Err}^{\mathrm{seq}}_n|}
{\sqrt{\frac1n\sum_{t=1}^n I(W_t;U_t\mid G_{t-1})}}
=\left(1+\frac23\delta^2+O(\delta^4)\right)^{-1/2}
\longrightarrow 1
\qquad(\delta\downarrow0).
\]
Equivalently,
\[
|\mathrm{Err}^{\mathrm{seq}}_n|
=(1+o(1))
\sqrt{\frac1n\sum_{t=1}^n I(W_t;U_t\mid G_{t-1})}.
\]
Thus the square-root dependence in Theorem~\ref{thm:general-rademacher} cannot be improved uniformly, even for conditionally i.i.d.\ pairs.

\subsection{Linear dependence in Theorem~\ref{thm:general-bernstein}}

In this example, the training excess risk is zero and the Bernstein condition holds with fixed constants, yet the holdout excess risk is proportional to the cumulative SCMI divided by $n$.

At each round draw $X_{t,0},X_{t,1}\stackrel{\mathrm{i.i.d.}}{\sim}\operatorname{Bernoulli}(p)$, with $0<p<1$, and set $Y_{t,u}=X_{t,u}$. The hypothesis class contains $h_0(x)=0$, $h_1(x)=1$, and $h^\star(x)=x$. Use the zero--one loss. The learner outputs the constant classifier that predicts the selected label:
\[
W_t=h_{Y_{t,U_t}}
\]
Since $h^\star$ always has zero loss, the excess loss below equals the learner's classification loss.

The learner always classifies the training observation correctly, so $R_n^{\mathrm{tr}}=0$. It misclassifies the holdout observation exactly when the two labels differ, giving
\begin{equation}
R_n^{\mathrm{ho}}
=\Pr(Y_{t,0}\neq Y_{t,1})
=2p(1-p).
\label{eq:appendix-l-v8-20}
\end{equation}

Fix the two current labels. If they agree, neither the state nor the excess loss on coordinate $0$ depends on the selector, so both SCMI terms are zero. If they differ, the calculation for an ordered pair in Eqs.~\eqref{eq:appendix-l-v8-8}--\eqref{eq:appendix-l-v8-13} gives
\[
W_t=h_{Y_{t,0}}\Longleftrightarrow U_t=0,
\qquad
W_t=h_{Y_{t,1}}\Longleftrightarrow U_t=1,
\]
and
\[
e_{t,0}(h^\star)
=\mathbf1\{Y_{t,U_t}\neq Y_{t,0}\}
=U_t.
\]
Thus, conditional on a pair of distinct labels, the mutual information with the selector is $\ln2$ for both the state and the excess loss. Averaging over the pair using the disagreement probability in Eq.~\eqref{eq:appendix-l-v8-20} yields
\begin{equation}
I(W_t;U_t\mid G_{t-1})
=I(e_{t,0}(h^\star);U_t\mid G_{t-1})
=2p(1-p)\ln2.
\label{eq:appendix-l-v8-21}
\end{equation}
Combining Eqs.~\eqref{eq:appendix-l-v8-20} and~\eqref{eq:appendix-l-v8-21} gives
\begin{equation}
R_n^{\mathrm{ho}}
=\frac{1}{n\ln2}
\sum_{t=1}^n I(e_{t,0}(h^\star);U_t\mid G_{t-1}).
\label{eq:appendix-l-v8-22}
\end{equation}

Finally, we verify the Bernstein condition. The excess loss is binary, so $e_{t,0}(h^\star)^2=e_{t,0}(h^\star)$. It is zero when the labels agree and equals $U_t$ when they differ. Since $U_t$ is fair even conditional on the current pair,
\[
\mathbb E[e_{t,0}(h^\star)^2\mid G_{t-1}]
=\tfrac12\mathbf1\{Y_{t,0}\neq Y_{t,1}\}.
\]
It follows that
\begin{equation}
\frac1n\sum_{t=1}^n
\mathbb E\!\left[\mathbb E[e_{t,0}(h^\star)^2\mid G_{t-1}]\right]
=\frac12R_n^{\mathrm{ho}}.
\label{eq:appendix-l-v8-23}
\end{equation}
The Bernstein condition in Theorem~\ref{thm:general-bernstein} therefore holds with $b=1$ and $B=1/2$. Nevertheless, Eq.~\eqref{eq:appendix-l-v8-22} shows that the holdout excess risk is proportional to the sum of the SCMI terms for fixed $n$. As $p\downarrow0$, the SCMI per round tends to zero, whereas the ratio
\[
\frac{R_n^{\mathrm{ho}}}
{n^{-1}\sum_{t=1}^n I(e_{t,0}(h^\star);U_t\mid G_{t-1})}
=\frac1{\ln2}
\]
remains constant. Thus the linear dependence on $\mathrm{SCMI}/n$ in Theorem~\ref{thm:general-bernstein} cannot be uniformly improved to $o(\mathrm{SCMI}/n)$.

These two tightness examples establish how the general bounds depend on SCMI and the horizon; they do not establish minimax optimality in the parameters specific to each application.

\section{A motivating example based on threshold classification}
\label{app:threshold-motivation-en}

We consider a motivating example based on the deterministic threshold classifier of
Steinke and Zakynthinou~\citep{steinke20a}.
Section~\ref{app:threshold-batch-mi-en} shows that ordinary MI is infinite, whereas
Section~\ref{app:threshold-batch-cmi-en} shows that supersample CMI remains finite.
Section~\ref{app:threshold-sequential-en} extends the same update rule to
history-dependent observations. We calculate the sequential generalization gap
and independently bound the cumulative SCMI from its definition and the update rule.
Applying Theorem~\ref{thm:general-bernstein}\ gives a bound with the same order in $n$ as that gap.
For completeness, we include the existing MI and CMI calculations alongside the
SCMI analysis.

\subsection{Batch setting}
\label{app:threshold-batch-en}

\subsubsection{Setup, training error, and generalization error}
Fix an unknown threshold $\theta^\star\in(0,1)$ and let
$X\sim\operatorname{Unif}[0,1]$ and $Y=\mathbf1\{X\geq\theta^\star\}$.
Use the classifier $h_w(x)=\mathbf1\{x\geq w\}$ and zero--one loss.
Given an i.i.d. sample $S=((X_i,Y_i))_{i=1}^n$, the learner returns
\[
 W=\mathcal A(S):=\min\bigl(\{1\}\cup\{X_i:Y_i=1\}\bigr).
\]
The training error is zero, and $L_{\mathcal D}(W)=W-\theta^\star$.
For $w\in[\theta^\star,1)$, the event $W>w$ occurs exactly when none of the
sampled features lies in $[\theta^\star,w]$, so
$\Pr(W>w)=(1-w+\theta^\star)^n$. Hence
\begin{equation}
 r_n:=\mathbb E[L_{\mathcal D}(W)]
 =\frac{1-(\theta^\star)^{n+1}}{n+1}
 =\Theta(n^{-1}).
 \label{eq:threshold-batch-risk}
\end{equation}

\subsubsection{Divergence of MI}
\label{app:threshold-batch-mi-en}
For a fixed $w<1$, the event $W=w$ requires $w=X_i$ for some $i$, and therefore
$P_W(\{w\})=0$. In contrast, conditioning on $S=s$ gives
$P_{W\mid S=s}=\delta_{\mathcal A(s)}$.
Since $\Pr(W<1)=1-(\theta^\star)^n>0$, we have
$P_{W\mid S=s}\not\ll P_W$ on a set of samples with positive probability.
The corresponding conditional KL divergence is infinite, yielding
\[
 I(W;S)=+\infty.
\]

\subsubsection{Calculation of CMI}
\label{app:threshold-batch-cmi-en}
Use the batch supersample construction of Section~\ref{sec:preliminaries}. For each $i$, the fair selector
$U_i$ selects $Z_{i,U_i}$ from the independent candidate pair $(Z_{i,0},Z_{i,1})$.
Write $\widetilde Z=((Z_{i,0},Z_{i,1}))_{i=1}^n$ and
$U=(U_1,\ldots,U_n)$. Fix a realization $\widetilde Z=\widetilde z$ and sort its
$M$ positive features as $x_{(1)}<\cdots<x_{(M)}$.
Write $x_{(k)}=x_{i_k,u_k}$, where $i_k$ identifies the candidate pair and
$u_k\in\{0,1\}$ identifies its member, and set $x_{(M+1)}=1$.
The rank of the first positive feature selected for training is
\[
 K:=\min\bigl(\{k\leq M:U_{i_k}=u_k\}\cup\{M+1\}\bigr).
\]
Since $W=x_{(K)}$, and the candidate features are distinct and lie in $(0,1)$
almost surely, $W$ and $K$ determine each other given $\widetilde z$.
Moreover, $K$ is a deterministic function of $U$. Thus
\[
 I(W;U\mid\widetilde Z=\widetilde z)
 =H(K\mid\widetilde Z=\widetilde z).
\]
For $K\geq k$, none of the first $k-1$ positive features may be selected.
If they belong to distinct candidate pairs, this event has probability
$2^{-(k-1)}$. If they include both members of a candidate pair, the probability
is zero because one member must be selected. Consequently,
\[
 \Pr(K\geq k\mid\widetilde Z=\widetilde z)\leq2^{-(k-1)}.
\]
The pointwise identity $K=\sum_{k\geq1}\mathbf1\{K\geq k\}$ gives
$\mathbb E[K\mid\widetilde Z=\widetilde z]\leq2$.
Comparing the distribution of $K$ with $q_k=2^{-k}$ on the positive integers,
nonnegativity of KL divergence yields
\[
 H(K\mid\widetilde Z=\widetilde z)
 \leq\mathbb E[-\log q_K\mid\widetilde Z=\widetilde z]
 =(\log2)\mathbb E[K\mid\widetilde Z=\widetilde z]\leq2\log2.
\]
Averaging over $\widetilde Z$ therefore gives
\begin{equation}
 I(W;U\mid\widetilde Z)\leq2\log2.
 \label{eq:threshold-batch-cmi}
\end{equation}

Recall $L_i^+=\ell(W,Z_{i,0})$ from Section~\ref{sec:preliminaries}.
Independence of $U_i$ and $\widetilde Z$, together with data processing, gives
$I(L_i^+;U_i)\leq I(W;U_i\mid\widetilde Z)$.
The selectors $U_i$ are also conditionally independent given $\widetilde Z$, so
\[
 I(W;U\mid\widetilde Z)
 =\sum_{i=1}^n I(W;U_i\mid\widetilde Z,U_1,\ldots,U_{i-1})
 \geq\sum_{i=1}^n I(W;U_i\mid\widetilde Z).
\]
Here the inequality follows because conditioning on the preceding selectors does not
change $H(U_i\mid\widetilde Z)$ and cannot increase
$H(U_i\mid W,\widetilde Z)$.
Applying the zero-training-error bound of
Wang and Mao~\citep{wang23} and then
Eq.~\eqref{eq:threshold-batch-cmi}, we obtain
\[
 r_n\leq\frac{2}{n\log2}\sum_{i=1}^n I(L_i^+;U_i)
 \leq\frac{2I(W;U\mid\widetilde Z)}{n\log2}\leq\frac4n.
\]
This has the same $n^{-1}$ order as Eq.~\eqref{eq:threshold-batch-risk}.

\subsection{Sequential setting}
\label{app:threshold-sequential-en}

\subsubsection{Sequential construction}
Use the online setting of Section~\ref{sec:online-application}, with $Q_t\equiv1$.
For this example, write $H_t$ and $G_{t-1}$ for the history
$H_t^{\mathrm{on}}$ and the conditioning variable $G_{t-1}^{\mathrm{on}}$.
Set $W_0=1$, $A_t:=W_{t-1}$, and $H_0=(W_0)$.
Given the history $H_{t-1}$, generate
\[
 X_{t,0},X_{t,1}\overset{\mathrm{i.i.d.}}{\sim}\operatorname{Unif}[0,A_t],
 \qquad
 Y_{t,u}=\mathbf1\{X_{t,u}\geq\theta^\star\},
 \qquad
 Z_{t,u}=(X_{t,u},Y_{t,u}),
\]
and write $\widetilde Z_t=(Z_{t,0},Z_{t,1})$.
Draw $U_t\sim\operatorname{Unif}(\{0,1\})$ independently of $(H_{t-1},\widetilde Z_t)$, and set
$Z_t=Z_{t,U_t}$ and $Z_t'=Z_{t,1-U_t}$.
Using only the selected observation $Z_t=(X_t,Y_t)$, the learner updates by
\[
 W_t=
 \begin{cases}
  X_t,&Y_t=1,\\
  A_t,&Y_t=0.
 \end{cases}
\]
The history and the conditioning variable satisfy
\[
 H_t=(H_{t-1},Z_t,W_t),
 \qquad
 G_{t-1}=(H_{t-1},\widetilde Z_t).
\]
The construction gives
$\Pr(U_t=0\mid G_{t-1})=\Pr(U_t=1\mid G_{t-1})=1/2$.

This setting is sequential because the sampling range itself depends on previous
updates. Whenever a positive observation gives $W_t<W_{t-1}$, the next pair is
sampled from the smaller interval $[0,W_t]$. The learner therefore concentrates
its search on positive features below its current threshold estimate, while
changing the observation distribution from round to round.

Ordinary MI is still infinite in this setting. On $\{W_n<1\}$, the output equals
one of the selected continuous observations, so $\Pr(W_n=w)=0$ for every fixed
$w<1$. Moreover, $\Pr(W_n<1)=1-(\theta^\star)^n>0$.
The same failure of absolute continuity as in
Section~\ref{app:threshold-batch-mi-en} therefore gives
\[
 I(W_n;Z_1,\ldots,Z_n)=+\infty.
\]

In the batch construction, all candidates are generated independently of the
selectors. Here, the first selection changes the distribution of the
second pair. To see this, fix $\widetilde Z_1$ and write the threshold after
selecting $U_1=u$ as
\[
 a_u:=
 \begin{cases}
  X_{1,u},&Y_{1,u}=1,\\
  1,&Y_{1,u}=0.
 \end{cases}
\]
Then
\[
 \Law(X_{2,0},X_{2,1}\mid\widetilde Z_1,U_1=u)
 =\operatorname{Unif}[0,a_u]^{\otimes2}.
\]
Unless both candidates are negative, $a_0\neq a_1$ almost surely, and the
conditional distribution of the second pair depends on $U_1$.
Thus $\widetilde Z^n$ is not a batch supersample independent of $U^n$,
and the batch bound does not apply directly.
Instead, we condition on $G_{t-1}$, which excludes future pairs and preserves
the fairness of the current selector $U_t$.
The sequential generalization gap from Section~\ref{subsec:seq-setting} is
\[
 \mathrm{Err}_n^{\mathrm{seq}}
 =\frac1n\sum_{t=1}^n
 \mathbb E[\ell(W_t,Z_t')-\ell(W_t,Z_t)].
\]

\subsubsection{SCMI bounds}
We first calculate the order of the true sequential generalization gap.
We then evaluate SCMI and the resulting Bernstein expression directly from
the definition and the update rule, verify the Bernstein condition, and
compare the two rates. The SCMI upper bound does not use the preceding
calculation of the true gap.

\paragraph{The true sequential generalization gap.}
The selected observation is always classified correctly after the update, so
$\ell(W_t,Z_t)=0$.
Fix $H_{t-1}$ and write $A_t=a$.
If the selected feature is $x$, the updated threshold is
\[
 w_a(x):=
 \begin{cases}
  a,&0\leq x<\theta^\star,\\
  x,&\theta^\star\leq x\leq a.
 \end{cases}
\]
Given the selected feature $x$, the holdout feature remains uniform on
$[0,a]$. Its misclassification region is $[\theta^\star,w_a(x))$, so its
misclassification probability is $(w_a(x)-\theta^\star)/a$.
The selected feature itself has density $1/a$. Averaging over it gives
\begin{align*}
 \mathbb E[\ell(W_t,Z_t')\mid H_{t-1}]
 &=\frac1a\int_0^a\frac{w_a(x)-\theta^\star}{a}\,dx\\
 &=\frac{\theta^\star(a-\theta^\star)+\frac12(a-\theta^\star)^2}{a^2}
 =\frac12\left(1-\frac{(\theta^\star)^2}{a^2}\right).
\end{align*}
Consequently,
\begin{equation}
 g_t:=\mathbb E[\ell(W_t,Z_t')]
 =\frac12\mathbb E\!\left[1-\left(\frac{\theta^\star}{W_{t-1}}\right)^2\right],
 \qquad
 \mathrm{Err}_n^{\mathrm{seq}}=\frac1n\sum_{t=1}^n g_t.
 \label{eq:threshold-round-risk}
\end{equation}

We derive a recurrence for $g_t$ itself and use it to show $g_t=\Theta(t^{-1})$.
For $\vartheta\in(0,1]$, let $g_t(\vartheta)$ denote the expected holdout loss at
round $t$ in the same sequential process with true threshold $\vartheta$ and
initial state $W_0=1$.
Here $\vartheta$ parametrizes the true threshold; the original quantity is
$g_t=g_t(\theta^\star)$. Probabilities and expectations in the following
calculation of $g_t(\vartheta)$ refer to this process.
Eq.~\eqref{eq:threshold-round-risk} gives
\[
 g_1(\vartheta)=\frac{1-\vartheta^2}{2}.
\]

\paragraph{Rescaling after the first observation.}
We calculate the expected loss at round $t+1$ by conditioning on the first
selected feature $X_1=x$.
The subsequent observation distributions and updates depend on the past only
through the current threshold.
If $x<\vartheta$, the observation is negative and $W_1=1$.
Relabeling round~2 as the first round therefore gives the same process with
initial state~1 and true threshold $\vartheta$.
The conditional expected loss at the original round $t+1$ is consequently
$g_t(\vartheta)$.

If $x\geq\vartheta$, then $W_1=x$.
Identify the original round $s+1$ with round $s$ of a restarted process, and set,
for $s\geq1$,
\[
 \begin{aligned}
 \overline W_0&:=\frac{W_1}{x}=1,&
 \overline W_s&:=\frac{W_{s+1}}{x},&
 \overline X_{s,u}&:=\frac{X_{s+1,u}}{x},\\
 \overline Y_{s,u}&:=Y_{s+1,u},&
 \overline U_s&:=U_{s+1}.&&
 \end{aligned}
\]
We verify that the rescaled process has the same law as the original process
with true threshold $\vartheta/x$.

First, at the original round $s+1$, the two features are independent uniform
draws from $[0,W_s]$ given $H_s$.
Dividing by $x$ changes the upper endpoint to
$W_s/x=\overline W_{s-1}$, so
\[
 \overline X_{s,0},\overline X_{s,1}\mid H_s
 \overset{\mathrm{i.i.d.}}{\sim}\operatorname{Unif}[0,\overline W_{s-1}],
 \qquad
 \overline Y_{s,u}=\mathbf1\{\overline X_{s,u}\geq\vartheta/x\}.
\]
The selector $\overline U_s\sim\operatorname{Unif}(\{0,1\})$ remains independent of
the past and the current pair.
Second, dividing the original update by $x$ gives
\[
 \overline W_s=
 \begin{cases}
  \overline X_{s,\overline U_s},&\overline X_{s,\overline U_s}\geq\vartheta/x,\\
  \overline W_{s-1},&\overline X_{s,\overline U_s}<\vartheta/x.
 \end{cases}
\]
Thus both the initial state and the update rule agree.
Finally, since $x>0$,
\[
 \mathbf1\{X_{s+1,u}\geq W_{s+1}\}
 =\mathbf1\{\overline X_{s,u}\geq\overline W_s\}.
\]
The labels are unchanged. Writing
$\overline Z_{s,u}:=(\overline X_{s,u},\overline Y_{s,u})$, we therefore have
\[
 \ell(W_{s+1},Z_{s+1,u})=\ell(\overline W_s,\overline Z_{s,u}).
\]
This equality also holds for the holdout observation, with
$u=1-U_{s+1}=1-\overline U_s$.
The observation law, update rule, and evaluated loss all agree.
Hence the conditional expected loss at the original round $t+1$ equals the
round-$t$ expected loss $g_t(\vartheta/x)$ in the rescaled process.
Combining the two cases gives
\[
 \mathbb E[\ell(W_{t+1},Z_{t+1}')\mid X_1=x]
 =\begin{cases}
   g_t(\vartheta),&0\leq x<\vartheta,\\
   g_t(\vartheta/x),&\vartheta\leq x\leq1.
  \end{cases}
\]

\paragraph{Recurrence for $g_t$.}
The density of $X_1\sim\operatorname{Unif}[0,1]$ is~1.
Integrating the preceding conditional expectation over $x$ gives
\[
 g_{t+1}(\vartheta)
 =\int_0^1\mathbb E[\ell(W_{t+1},Z_{t+1}')\mid X_1=x]\,dx
 =\int_0^{\vartheta}g_t(\vartheta)\,dx
   +\int_{\vartheta}^1g_t(\vartheta/x)\,dx.
\]
The first integral equals $\vartheta g_t(\vartheta)$.
In the second, the substitution $v=\vartheta/x$ gives
$dx=-\vartheta v^{-2}\,dv$, with $x=\vartheta,1$ corresponding to
$v=1,\vartheta$, respectively. Therefore
\begin{equation}
 g_{t+1}(\vartheta)
 =\vartheta g_t(\vartheta)+\int_{\vartheta}^1g_t(\vartheta/x)\,dx
 =\vartheta g_t(\vartheta)+\vartheta\int_{\vartheta}^1\frac{g_t(v)}{v^2}\,dv.
 \label{eq:threshold-recurrence}
\end{equation}
Together with the initial function $g_1$, this recurrence determines $g_t$.

\paragraph{Bounds by induction.}
Define the comparison function
\[
 r_t(\vartheta):=\frac{1-\vartheta^{t+1}}{t+1}
 =\int_{\vartheta}^1 z^t\,dz.
\]
We show that, for every $t\geq1$ and $\vartheta\in(0,1]$,
\begin{equation}
 r_t(\vartheta)\leq g_t(\vartheta)\leq\frac{r_t(\vartheta)}{\vartheta}.
 \label{eq:threshold-sandwich}
\end{equation}
The claim holds at $t=1$ because $g_1=r_1$.
Suppose it holds at round $t$ for every argument $v\in(0,1]$.
Since the coefficient and integration weight in
Eq.~\eqref{eq:threshold-recurrence} are nonnegative,
\[
 \vartheta r_t(\vartheta)+\vartheta\int_{\vartheta}^1\frac{r_t(v)}{v^2}\,dv
 \leq g_{t+1}(\vartheta)\leq
 r_t(\vartheta)+\vartheta\int_{\vartheta}^1\frac{r_t(v)}{v^3}\,dv.
\]
To close the induction, it remains to show that the leftmost expression is at
least $r_{t+1}(\vartheta)$ and the rightmost expression is at most
$r_{t+1}(\vartheta)/\vartheta$.
We express each required difference as a nonnegative integral.

\paragraph{Preservation of the lower bound.}
Substitute $r_t(v)=\int_v^1z^t\,dz$ and reverse the order of integration over
$\vartheta\leq v\leq z\leq1$. This gives
\[
 \vartheta\int_{\vartheta}^1\frac{r_t(v)}{v^2}\,dv
 =\int_{\vartheta}^1z^t\left(\vartheta\int_{\vartheta}^z\frac{dv}{v^2}\right)dz
 =\int_{\vartheta}^1z^t\left(1-\frac{\vartheta}{z}\right)dz.
\]
The difference from the desired lower bound is therefore
\begin{align*}
 &\vartheta r_t(\vartheta)+\vartheta\int_{\vartheta}^1\frac{r_t(v)}{v^2}\,dv-r_{t+1}(\vartheta)\\
 &\quad=\int_{\vartheta}^1z^t\left(1+\vartheta-\frac{\vartheta}{z}-z\right)dz
 =\int_{\vartheta}^1z^{t-1}(1-z)(z-\vartheta)\,dz\geq0.
\end{align*}
Every factor in the final integrand is nonnegative on $[\vartheta,1]$.
Thus $g_{t+1}(\vartheta)\geq r_{t+1}(\vartheta)$.

\paragraph{Preservation of the upper bound.}
Reversing the integration order in the same way gives
\[
 \vartheta\int_{\vartheta}^1\frac{r_t(v)}{v^3}\,dv
 =\int_{\vartheta}^1z^t\left(\vartheta\int_{\vartheta}^z\frac{dv}{v^3}\right)dz
 =\int_{\vartheta}^1z^t\left(\frac{1}{2\vartheta}-\frac{\vartheta}{2z^2}\right)dz.
\]
Subtracting the upper bound obtained from the recurrence from the desired upper
bound yields
\[
 \frac{r_{t+1}(\vartheta)}{\vartheta}
 -\left(r_t(\vartheta)+\vartheta\int_{\vartheta}^1\frac{r_t(v)}{v^3}\,dv\right)
 =\int_{\vartheta}^1z^t\left(\frac{z}{\vartheta}-1-\frac{1}{2\vartheta}+\frac{\vartheta}{2z^2}\right)dz.
\]
To evaluate the preceding difference by integration by parts, define the following
antiderivative:
\begin{align*}
 F_{\vartheta}(z)
 &:=\int_{\vartheta}^z\left(\frac{y}{\vartheta}-1-\frac{1}{2\vartheta}+\frac{\vartheta}{2y^2}\right)dy\\
 &=\left[\frac{y^2}{2\vartheta}-\left(1+\frac{1}{2\vartheta}\right)y-\frac{\vartheta}{2y}\right]_{y=\vartheta}^{y=z}
 =\frac{(z-\vartheta)^2(z-1)}{2\vartheta z}\leq0.
\end{align*}
Since $F_{\vartheta}(\vartheta)=F_{\vartheta}(1)=0$, integration by parts writes
the preceding difference as
\[
 \int_{\vartheta}^1z^t F_{\vartheta}'(z)\,dz
 =-t\int_{\vartheta}^1z^{t-1}F_{\vartheta}(z)\,dz
 =\frac{t}{2\vartheta}\int_{\vartheta}^1z^{t-2}(z-\vartheta)^2(1-z)\,dz\geq0.
\]
Hence $g_{t+1}(\vartheta)\leq r_{t+1}(\vartheta)/\vartheta$, completing the
induction.

Setting $\vartheta=\theta^\star$ in Eq.~\eqref{eq:threshold-sandwich} gives
\[
 \frac{1-(\theta^\star)^{t+1}}{t+1}
 \leq g_t\leq
 \frac{1-(\theta^\star)^{t+1}}{\theta^\star(t+1)}.
\]
For fixed $\theta^\star\in(0,1)$, the numerator tends to~1.
Using $\sum_{t=1}^n(t+1)^{-1}=\Theta(\log n)$ in
Eq.~\eqref{eq:threshold-round-risk}, we obtain
\begin{equation}
 g_t=\Theta_{\theta^\star}(t^{-1}),
 \qquad
 \mathrm{Err}_n^{\mathrm{seq}}
 =\Theta_{\theta^\star}\!\left(\frac{\log n}{n}\right).
 \label{eq:threshold-rates}
\end{equation}

\paragraph{Direct evaluation from the definition of SCMI.}
We now evaluate SCMI without using the value or convergence rate of $g_t$.
Recall from Section~\ref{subsec:seq-setting} that
$L_t^+=\ell(W_t,Z_{t,0})$ when $Q_t\equiv1$, and write
$I_t^{\mathrm{loss}}=I(L_t^+;U_t\mid G_{t-1})$.
Let $B_t$ be the event that $Y_{t,0}=1$ and $X_{t,0}$ is the smallest feature
among the observations with label $1$ in the current pair:
\[
 B_t:=\{Y_{t,0}=1\}\cap
 \bigl(\{Y_{t,1}=0\}\cup\{X_{t,0}<X_{t,1}\}\bigr).
\]
We exclude the zero-probability events on which the two features coincide or
equal an endpoint of the sampling interval.
On $B_t$, if $U_t=0$, the learner updates using $Z_{t,0}$ itself, so $L_t^+=0$.
If $U_t=1$, the threshold becomes $X_{t,1}>X_{t,0}$ when $Y_{t,1}=1$, or
remains $W_{t-1}>X_{t,0}$ when $Y_{t,1}=0$; in either case, $L_t^+=1$.
Outside $B_t$, either $Y_{t,0}=0$, or both labels are $1$ and
$X_{t,1}<X_{t,0}$. In each case, the loss at $Z_{t,0}$ is zero under either
selection.
Consequently,
\[
 L_t^+=U_t\mathbf1_{B_t}.
\]
The event $B_t$ is determined by $G_{t-1}=(H_{t-1},\widetilde Z_t)$.
Thus, for fixed $G_{t-1}=g$, its indicator $\mathbf1_{B_t}(g)$ is a constant
in $\{0,1\}$.
Since $U_t\mid G_{t-1}=g\sim\operatorname{Unif}(\{0,1\})$, the conditional
probabilities are
\[
 \Pr(L_t^+=1\mid G_{t-1}=g)=\frac12\mathbf1_{B_t}(g),
 \qquad
 \Pr(L_t^+=0\mid G_{t-1}=g)=1-\frac12\mathbf1_{B_t}(g).
\]
Substituting these probabilities into the definition of entropy gives
\[
 H(L_t^+\mid G_{t-1}=g)=
 \begin{cases}
  -1\log1=0,&\mathbf1_{B_t}(g)=0,\\[2pt]
  -\frac12\log\frac12-\frac12\log\frac12=\log2,
       &\mathbf1_{B_t}(g)=1,
 \end{cases}
\]
where $0\log0=0$.
Hence $H(L_t^+\mid G_{t-1}=g)=\mathbf1_{B_t}(g)\log2$.
Once $U_t$ is also given, $L_t^+=U_t\mathbf1_{B_t}(g)$ is fixed, so
$H(L_t^+\mid U_t,G_{t-1}=g)=0$.
Averaging the difference of these two entropies over $G_{t-1}$ yields
\[
 I_t^{\mathrm{loss}}
 =H(L_t^+\mid G_{t-1})-H(L_t^+\mid U_t,G_{t-1})
 =(\log2)\Pr(B_t).
\]

Given $H_{t-1}$, the probability that $Y_{t,u}=1$ is
\[
 p_t:=\Pr(Y_{t,u}=1\mid H_{t-1})
 =1-\frac{\theta^\star}{W_{t-1}}.
\]
The probability that at least one of $Y_{t,0},Y_{t,1}$ is positive is
$1-(1-p_t)^2=2p_t-p_t^2$.
By symmetry, the smallest $X_{t,u}$ with $Y_{t,u}=1$ is equally likely to be
in either coordinate. Therefore
$\Pr(B_t\mid H_{t-1})=(2p_t-p_t^2)/2$, and
\begin{equation}
 I_t^{\mathrm{loss}}=(\log2)\mathbb E\!\left[p_t-\frac12p_t^2\right].
 \label{eq:threshold-local-information}
\end{equation}
Given $G_{t-1}$, the state $W_t$ determines $U_t$ whenever
$Y_{t,0}=1$ or $Y_{t,1}=1$.
Indeed, the two possible updated thresholds coincide when both labels are
negative and are distinct otherwise, outside the zero-probability events
already excluded. The same conditional-entropy calculation therefore gives
\begin{equation}
 I(W_t;U_t\mid G_{t-1})
 =(\log2)\mathbb E[2p_t-p_t^2]
 =2I_t^{\mathrm{loss}}.
 \label{eq:threshold-state-information}
\end{equation}

\paragraph{A direct cumulative SCMI bound from the update rule.}
Write $J_n^{\mathrm{loss}}:=\sum_{t=1}^n I_t^{\mathrm{loss}}$.
Equation~\eqref{eq:threshold-local-information} and $0\leq p_t\leq1$ imply
\[
 J_n^{\mathrm{loss}}\leq(\log2)\sum_{t=1}^n\mathbb E[p_t].
\]
We bound this sum directly.
Fix a horizon $n\geq1$ and define
\[
 F_n(w):=\log\!\left(1+\frac{n(w-\theta^\star)}{\theta^\star}\right),
 \qquad \theta^\star\leq w\leq1.
\]
This function satisfies
\[
 0\leq F_n(w)\leq F_n(1)
 =\log\!\left(1+\frac{n(1-\theta^\star)}{\theta^\star}\right).
\]
We will bound $p_t$ by the conditional expected decrease
$F_n(W_{t-1})-F_n(W_t)$.
Summing these decreases over time gives $F_n(1)-F_n(W_n)$, which will yield
the logarithmic bound.

\paragraph{Decrease in one update.}
Condition only on the history $H_{t-1}$ before the update, and write
$W_{t-1}=a$. The current pair $\widetilde Z_t$ and selector $U_t$ are still
random under this conditioning.
We have
$X_{t,0},X_{t,1}\mid H_{t-1}\overset{\mathrm{i.i.d.}}{\sim}
\operatorname{Unif}[0,a]$, and
$U_t\sim\operatorname{Unif}(\{0,1\})$ is independent of
$(H_{t-1},\widetilde Z_t)$. Hence
\[
 \Pr(X_t\leq z\mid H_{t-1})
 =\frac12\sum_{u=0}^1\Pr(X_{t,u}\leq z\mid H_{t-1})
 =\frac za,
 \qquad 0\leq z\leq a.
\]
Thus the selected feature $X_t=X_{t,U_t}$ has conditional density $1/a$
on $[0,a]$.

For $X_t=x<\theta^\star$, the observation is negative and the threshold
remains $a$. For $x\geq\theta^\star$, the update gives $W_t=x$.
Therefore
\[
 F_n(W_{t-1})-F_n(W_t)
 =\begin{cases}
  F_n(a)-F_n(a)=0,&0\leq x<\theta^\star,\\
  F_n(a)-F_n(x),&\theta^\star\leq x\leq a.
 \end{cases}
\]
Multiplying this decrease by the density $1/a$ of $X_t$ and integrating
over $[0,a]$ gives
\begin{align*}
 \mathbb E[F_n(W_{t-1})-F_n(W_t)\mid H_{t-1}]
 &=\int_0^{\theta^\star}0\,\frac{dx}{a}
   +\int_{\theta^\star}^{a}[F_n(a)-F_n(x)]\frac{dx}{a}\\
 &=\frac1a\int_{\theta^\star}^a[F_n(a)-F_n(x)]\,dx.
\end{align*}
The first integral vanishes because a negative observation gives zero
decrease. The density remains $1/a$ because the expectation conditions only
on $H_{t-1}$, not on observing a positive label.

For $x\in[\theta^\star,a]$,
\[
 F_n(a)-F_n(x)
 =\int_x^a\frac{dv}{v-\theta^\star+\theta^\star/n}
 \geq\frac{a-x}{a-\theta^\star+\theta^\star/n}.
\]
The inequality follows because the denominator in the integral is at most
$a-\theta^\star+\theta^\star/n$.
Substituting this lower bound into the conditional expectation yields
\begin{align*}
 \mathbb E[F_n(W_{t-1})-F_n(W_t)\mid H_{t-1}]
 &\geq\frac1a\int_{\theta^\star}^a
       \frac{a-x}{a-\theta^\star+\theta^\star/n}\,dx\\
 &=\frac{(a-\theta^\star)^2}
         {2a(a-\theta^\star+\theta^\star/n)}.
\end{align*}
Using $p_t=(a-\theta^\star)/a$, the last expression can be written as
\[
 \frac{(a-\theta^\star)^2}{2a(a-\theta^\star+\theta^\star/n)}
 =\frac{p_t}{2}-\frac{\theta^\star}{2na}
       \frac{a-\theta^\star}{a-\theta^\star+\theta^\star/n}
 \geq\frac{p_t}{2}-\frac1{2n}.
\]
Here we used $\theta^\star/a\leq1$ and
$0\leq(a-\theta^\star)/(a-\theta^\star+\theta^\star/n)\leq1$.
Consequently,
\[
 p_t\leq2\mathbb E[F_n(W_{t-1})-F_n(W_t)\mid H_{t-1}]+\frac1n.
\]

\paragraph{Summing over time.}
Take expectations and sum over $t=1,\ldots,n$.
The function $F_n$ uses the same fixed horizon $n$ at every round, so the
intermediate terms cancel:
\begin{align*}
 \sum_{t=1}^n\mathbb E[p_t]
 &\leq2\sum_{t=1}^n\mathbb E[F_n(W_{t-1})-F_n(W_t)]+1\\
 &=2\bigl(F_n(1)-\mathbb E[F_n(W_n)]\bigr)+1
 \leq2F_n(1)+1.
\end{align*}
Thus the definition of SCMI and the original update rule give
\begin{equation}
 J_n^{\mathrm{loss}}
 \leq(\log2)\left[1+2\log\!\left(1+\frac{n(1-\theta^\star)}{\theta^\star}\right)\right]
 =O_{\theta^\star}(\log n).
 \label{eq:threshold-cumulative-information}
\end{equation}
This estimate uses neither the value nor the convergence rate of $g_t$.
By Eq.~\eqref{eq:threshold-state-information}, twice the displayed upper
bound also bounds $\sum_{t=1}^n I(W_t;U_t\mid G_{t-1})$.

For the parameter choice $b=1$, $B=1/2$, and $\lambda=3/4$, whose applicability
will be verified next, write
$\mathcal B_n^{\mathrm{Ber}}:=16J_n^{\mathrm{loss}}/(3n)$ for the resulting
Bernstein expression.
Equation~\eqref{eq:threshold-cumulative-information} bounds its value by
\begin{equation}
 \mathcal B_n^{\mathrm{Ber}}
 \leq\frac{16\log2}{3n}
 \left[1+2\log\!\left(1+\frac{n(1-\theta^\star)}{\theta^\star}\right)\right]
 =O_{\theta^\star}\!\left(\frac{\log n}{n}\right).
 \label{eq:threshold-bound-value}
\end{equation}

\paragraph{Verification of the Bernstein condition.}
The weights remain $Q_t\equiv1$, and we take the comparator
$w^\star=\theta^\star$.
The true threshold has zero loss, so
\[
 e_{t,u}(w^\star)
 =\ell(W_t,Z_{t,u})-\ell(\theta^\star,Z_{t,u})
 =\ell(W_t,Z_{t,u})\in\{0,1\}.
\]
Thus $b=1$, $R_n^{\mathrm{tr}}=0$, and
$R_n^{\mathrm{ho}}=\mathrm{Err}_n^{\mathrm{seq}}\geq0$.
Since the selected loss is zero, the sum of the losses at the two coordinates
equals the holdout loss.
Moreover, replacing $(Z_{t,0},Z_{t,1},U_t)$ by
$(Z_{t,1},Z_{t,0},1-U_t)$ preserves the law and the selected observation
$Z_t$, hence also the updated state $W_t$.
The two coordinate losses therefore have equal expectations, and
\[
 \mathbb E[(L_t^+)^2]=\mathbb E[L_t^+]
 =\frac12\mathbb E[\ell(W_t,Z_{t,0})+\ell(W_t,Z_{t,1})]
 =\frac12g_t.
\]
Consequently, the second-moment condition in Theorem~\ref{thm:general-bernstein}\ is
\begin{equation}
 \frac1n\sum_{t=1}^n
 \mathbb E\!\left[\mathbb E[e_{t,0}(w^\star)^2\mid G_{t-1}]\right]
 =\frac1{2n}\sum_{t=1}^n g_t
 =\frac12R_n^{\mathrm{ho}},
 \label{eq:threshold-bernstein-condition}
\end{equation}
so it holds with $B=1/2$.
The observations $Z_{t,0},Z_{t,1}$ are conditionally i.i.d. given $H_{t-1}$,
$U_t\mid G_{t-1}\sim\operatorname{Unif}(\{0,1\})$, and the update uses only
the previous history and $Z_t$.
Thus the remaining hypotheses of Theorem~\ref{thm:general-bernstein}\ also hold.
For $\lambda=3/4$,
\[
 0<\lambda<3/b,
 \qquad
 1-\frac{B\lambda}{1-\lambda b/3}
 =1-\frac{(1/2)(3/4)}{1-(3/4)/3}
 =\frac12>0.
\]
Applying Theorem~\ref{thm:general-bernstein}\ under
Eq.~\eqref{eq:threshold-bernstein-condition} gives
$\frac12R_n^{\mathrm{ho}}\leq\frac{8}{3n}J_n^{\mathrm{loss}}$.
Together with Eq.~\eqref{eq:threshold-bound-value}, this yields
\begin{equation}
 \mathrm{Err}_n^{\mathrm{seq}}
 \leq\mathcal B_n^{\mathrm{Ber}}
 \leq\frac{16\log2}{3n}
 \left[1+2\log\!\left(1+\frac{n(1-\theta^\star)}{\theta^\star}\right)\right].
 \label{eq:threshold-bernstein-bound}
\end{equation}

\paragraph{Comparison with the true rate.}
The separate analysis of $g_t$ in Eq.~\eqref{eq:threshold-rates} gives
$\mathrm{Err}_n^{\mathrm{seq}}=
\Theta_{\theta^\star}((\log n)/n)$.
The direct SCMI estimate in Eq.~\eqref{eq:threshold-bound-value} and the
verified application in Eq.~\eqref{eq:threshold-bernstein-bound} give
\[
 \mathrm{Err}_n^{\mathrm{seq}}
 \leq\mathcal B_n^{\mathrm{Ber}}
 =O_{\theta^\star}\!\left(\frac{\log n}{n}\right).
\]
Since the true error has a lower bound of the same order,
$\mathcal B_n^{\mathrm{Ber}}=
\Theta_{\theta^\star}((\log n)/n)$ as well.
The Bernstein-type SCMI bound therefore matches the order in $n$ of the true
sequential generalization gap for this update rule.

Finally, $L_t^+=U_t\mathbf1_{B_t}$ and the conditional uniform law of $U_t$
give $\mathbb E[L_t^+]=\Pr(B_t)/2$.
Combining this with the identity $\mathbb E[L_t^+]=g_t/2$ established above
yields
\begin{equation}
 I_t^{\mathrm{loss}}=(\log2)g_t.
 \label{eq:threshold-loss-scmi}
\end{equation}
Using this identity and the already established rate
$g_t=\Theta_{\theta^\star}(t^{-1})$, the right-hand side of the first
inequality in Theorem~\ref{thm:general-rademacher}\ is
\[
 \mathcal B_n^{\mathrm{sqrt}}
 :=\frac2n\sum_{t=1}^n\sqrt{2I_t^{\mathrm{loss}}}
 =\Theta_{\theta^\star}(n^{-1/2}),
\]
which is slower than the Bernstein rate.
This final comparison uses the rate of $g_t$; the direct cumulative SCMI
bound in Eq.~\eqref{eq:threshold-cumulative-information} did not use that
rate or Eq.~\eqref{eq:threshold-loss-scmi}.

\end{document}